\documentclass[]{svjour3}
\usepackage{amsmath,amsfonts,amssymb}
\usepackage{tikz} \usetikzlibrary{arrows,automata,calc}
\usepackage{booktabs}
\usepackage[plainpages=false,pdfpagelabels,colorlinks,linkcolor=blue,citecolor=blue,urlcolor=blue,linktoc=all]{hyperref}
\usepackage{geometry}
\usepackage{subcaption}
\usepackage[T1]{fontenc}
\usepackage{IEEEtrantools}
\usepackage{enumerate}
\usepackage{cases}
\usepackage{pifont}
\usepackage{pgfplots,wrapfig}
\usepackage{multicol}
\usepackage{floatrow}

\newcommand{\defref}[1]{Definition~\ref{#1}}
\newcommand{\propref}[1]{Property~\ref{#1}}
\newcommand{\exref}[1]{Example~\ref{#1}}
\newcommand{\figref}[1]{Figure~\ref{#1}}
\newcommand{\secref}[1]{Section~\ref{#1}}
\newcommand{\thref}[1]{Theorem~\ref{#1}}
\newcommand{\lemref}[1]{Lemma~\ref{#1}}
\newcommand{\condref}[1]{Condition~\ref{#1}}
\newcommand{\caseref}[1]{Case~\ref{#1}}

\newcommand{\Integers}{\mathbb{Z}}
\newcommand{\IntegersStar}{\mathbb{Z}^*}
\newcommand{\Naturals}{\mathbb{N}}
\newcommand{\NValues}{A}
\newcommand{\Curly}[1]{\{#1\}}
\newcommand{\Tuple}[1]{\left \langle#1 \right \rangle}

\newcommand{\constraint}[1]{\textsc{#1}}
\newcommand{\Result}{R}
\newcommand{\CtrGamma}{\gamma}

\newcommand{\SumAggr}{\mathtt{sum}}

\newcommand{\One}{\mathtt{one}}
\newcommand{\Width}{\mathtt{width}}

\newcommand{\DecreasingSequencePattern}{\reg{(>(>|=)^*)^*>}}
\newcommand{\DecreasingTerracePatternName}{\textsc{decreasing\_terrace}}
\newcommand{\DecreasingTerracePattern}{\reg{>=^+>}}
\newcommand{\IncreasingTerracePatternName}{\textsc{increasing\_terrace}}
\newcommand{\IncreasingTerracePattern}{\reg{<=^+<}}
\newcommand{\PeakPatternName}{\textsc{peak}}
\newcommand{\PeakPattern}{\reg{<(<|=)^*(>|=)^*>}}
\newcommand{\ProperPlainPattern}{\reg{>=^+<}}
\newcommand{\ValleyPatternName}{\textsc{valley}}
\newcommand{\ValleyPattern}{\reg{>(>|=)^*(<|=)^*<}}
\newcommand{\ZigzagPattern}{\reg{(<>)^+<(>|\varepsilon)~|~(><)^+>(<|\varepsilon)}}

\newcommand{\After}{a}
\newcommand{\Before}{b}
\newcommand{\pattern}{\sigma}
\newcommand{\reg}[1]{\text{\mbox{`$#1$'}}}
\newcommand{\SigmaPattern}{$\pattern$-pattern}

\newcommand{\Char}[2]{#1_{#2}}
\newcommand{\CharArg}[3]{#1_{#2}(#3)}

\newcommand{\Definition}[6]{The~\emph{#1} of~#2,
  denoted by~$#3$, is a function that maps an element of~$#4$
  to~$#5$#6.}

\newcommand{\width}{\omega}
\newcommand{\Shift}{\delta}

\newcommand{\Language}[1]{\ensuremath{\mathcal{L}_{#1}}}
\newcommand{\Alphabet}{\Sigma}

\newcommand{\seqlength}{n}
\newcommand{\Upp}[1]{\textnormal{up}_{#1}}
\newcommand{\XSeq}{\Tuple{X_1, X_2, \dots, X_\seqlength}}
\newcommand{\OutputSeq}{T}
\newcommand{\X}{X}

\newcommand{\Frac}[2]{\left \lfloor \frac{#1}{#2} \right \rfloor}
\newcommand\xqed[1]{%
  \leavevmode\unskip\penalty9999 \hbox{}\nobreak\hfill
  \quad\hbox{#1}}
\newcommand\qedexample{\xqed{\scriptsize$\triangle$}}

\newcommand\ListStyle{%
   \renewcommand{\labelitemi}{$\circ$}}

\newcommand{\NAut}{k}
\newcommand{\FinalV}{\Result}
\newcommand{\acc}{r}
\newcommand{\Acc}{A}
\newcommand{\Coeff}[1]{e_{#1}}
\newcommand{\Function}{f}
\newcommand{\Aut}{\mathcal{M}}
\newcommand{\IncrementalAut}{incremental\nobreakdash-automaton}
\newcommand{\TrCoeff}[3]{\alpha^{#1}_{#2, #3}}
\newcommand{\DTrCoeff}[3]{\gamma^{#1}_{#2, #3}}
\newcommand{\InitialV}[1]{\alpha^{0}_{#1}}
\newcommand{\DInitialV}[1]{\gamma^{0}_{#1}}
\newcommand{\Delay}[1]{d_{#1}}
\newcommand{\Tr}{t}
\newcommand{\LinFun}{\Coeff{}+\Coeff{0}\cdot\seqlength+\sum\limits_{i=1}^{\NAut}\Coeff{i}\cdot\Result_i}
\newcommand{\LinFunVar}{\Coeff{0}\cdot\seqlength+\sum\limits_{i=1}^{\NAut}\Coeff{i}\cdot\Result_i }
\newcommand{\final}[1]{{\color{black} #1}}
\newcommand{\Peak}{P}
\newcommand{\Valley}{V}
\newcommand{\Product}{\mathcal{I}}
\newcommand{\State}{q}
\newcommand{\DelayedProduct}{\Product^*}
\newcommand{\Intersect}{\Product=\Aut_1\cap\Aut_2\cap\dots\cap\Aut_{\NAut}}
\newcommand{\Digraph}[2]{G_{#1}^{#2}}
\newcommand{\Weight}{w}
\newcommand{\Cycles}{\mathcal{C}}
\newcommand{\Automaton}{\textsc{automaton}}
\newcommand{\x}{X}
\newcommand{\SigSeq}{S}
\newcommand{\Sig}{S}
\newcommand{\Found}{\mathtt{found}}
\newcommand{\NotFound}{\mathtt{not\_found}}

\newcommand{\Const}{d}
\newcommand{\Hyp}{H}
\newcommand{\BoolFun}{f}
\newcommand{\Predicate}{C}
\newcommand{\NbPredicates}{p}
\newcommand{\NbFunctions}{m}
\newcommand{\InfeasibleSet}{\mathcal{I}}
\newcommand{\HypTimeSeries}[1]{\mathcal{T}_{#1}}
\newcommand{\DataPos}{\mathcal{D}^+}
\newcommand{\DataNeg}{\mathcal{D}^-}
\newcommand{\DataSet}{\mathcal{D}}
\newcommand{\HypSet}{\mathcal{H}}

\newcommand{\xvar}{X}
\newcommand{\LastFoundIndex}{t^*}
\newcommand{\Regret}[1]{\rho(#1)}
\newcommand{\MaxElement}{\phi}
\newcommand{\LossInt}{[\LossMin,\LossMax]}
\newcommand{\LossI}[1]{\mathcal{L}_{#1}}
\newcommand{\QreClass}{\mathbb{S}}
\newcommand{\GapToLossFun}{h_{\CtrGamma}}
\newcommand{\xseq}{\langle X_1, X_2, \dots, X_{\seqlength} \rangle}
\newcommand{\Sign}{\Sgn{\Result}}
\newcommand{\HomogeneityProp}[1]{\textsc{homogeneity} property#1}
\newcommand{\ValueIndependent}[1]{value independent#1}
\newcommand{\Max}[1]{\max \left  (#1 \right)}
\newcommand{\C}[1]{c_{#1}}
\newcommand{\D}[1]{d_{#1}}
\newcommand{\Sgn}[1]{\Signum(#1)}
\newcommand{\LossMin}{\ell_{\mathit{min}}}
\newcommand{\LossMax}{\ell_{\mathit{max}}}
\newcommand{\FirstAcc}{C}
\newcommand{\SecondAcc}{D}
\newcommand{\ThirdAcc}{R}
\newcommand{\Transducer}[1]{\mathcal{T}_{#1}}
\newcommand{\FixedAut}{\mathcal{A}_{\Aut}}
\newcommand{\nacc}{p}
\newcommand{\PeakConstraint}{\constraint{nb\_peak}}
\newcommand{\trans}[3]{#1\xrightarrow{#3} #2}

\DeclareMathOperator{\Gap}{gap}
\DeclareMathOperator{\Loss}{loss}
\DeclareMathOperator{\Signum}{sgn}

\begin{document}

\author{
   Ekaterina~Arafailova 
  \and Nicolas~Beldiceanu 
  \and Helmut~Simonis
  }

\title{Synthesising a Database of Parameterised Linear and \\
Non-Linear Invariants for Time-Series Constraints
\thanks{This is an extended version of the CP~2017 paper~\cite{CP17LinearInvariants}.
Ekaterina Arafailova is supported by the EU H2020 programme under grant 640954 for project GRACeFUL.
Nicolas~Beldiceanu is partially supported by the GRACeFUL project and
by the Gaspard Monge Program for Optimisation and Operations Research~(PGMO).
Helmut Simonis is supported by Science Foundation Ireland (SFI) under
grant SFI/10/IN.1/I3032; the Insight Centre for Data Analytics is supported by~SFI
under grant SFI/12/RC/2289. \\
}}
   \institute{Ekaterina~Arafailova and Nicolas~Beldiceanu \at TASC (LS2N) IMT
     Atlantique, FR -- 44307 Nantes, France \\
     \email{\{Ekaterina.Arafailova,Nicolas.Beldiceanu\}@imt-atlantique.fr
     }
\and Helmut~Simonis \at Insight Centre for Data Analytics,
     University College Cork, Ireland  \\
     \email{Helmut.Simonis@insight-centre.org}
   }

\maketitle

\begin{abstract}
Many constraints restricting the result of some computations
over an integer sequence can be compactly represented by register
automata.
We improve the propagation of the conjunction of such constraints
on the same sequence by synthesising a database of linear
and non-linear invariants using their register-automaton
representation.
The obtained invariants are formulae parameterised by a function of the
sequence length and proven to be true for any long enough sequence.
To assess the quality of such linear invariants,
we developed a method to verify whether a generated linear invariant
is a facet of the convex hull of the feasible points.
This method, as well as the proof of non-linear invariants, are based on the 
systematic generation of constant-size deterministic finite automata that accept
all integer sequences whose result verifies some simple condition.
We apply such methodology to a set of~$44$ time\nobreakdash-series constraints
and obtain~$1400$ linear invariants from which~70\% are facet defining,
and~$600$ non-linear invariants, which were tested on short-term electricity
production problems.
\end{abstract}

\section{Introduction}\label{sec:introduction}

We present a framework for synthesising necessary conditions 
for a conjunction of sequence constraints that are each represented by
a register automaton~\cite{BeldiceanuCarlssonPetit04}, and are imposed
on the same integer sequence of length $\seqlength$.
Our necessary conditions are in the form of linear inequalities,
implications whose right-hand side is a linear inequality, and
disjunctions of inequalities.
In addition, they are parameterised by a function of $\seqlength$ and
instance-independent, i.e.~they are true for any integer sequence
of length $\seqlength$ greater than some small constant.

In order to synthesise linear inequalities and implications with linear
inequalities we draw full benefit from register automata representing the
constraints since they do not encode explicitly all potential
values of registers as states, and allow a constant-size representation of
many counting constraints imposed on a sequence of integer variables.
Moreover their compositional nature permits representing a conjunction
of sequence constraints as the intersection of the corresponding
register automata~\cite{Menana,menana_thesis},
i.e.~the intersection of the languages accepted by all register automata,
without representing explicitly the Cartesian product of all register values.
As a consequence, the size of such an intersection register automaton is often quite compact,
even if maintaining domain consistency for such constraints is in general
NP-hard~\cite{ASTRA:AAAI14:regCount};
for instance, the intersection of the $22$ register automata for all
$\constraint{nb}\_\pattern$ time-series constraints described in~\cite{Catalog18}
has only $16$ states.

To formally analyse the quality of the generated invariants we developed a method
allowing us to verify whether a linear invariant is a facet of the convex hull or not.
The method identifies two distinct points located on the line corresponding to the linear invariant,
and shows that these points are always feasible provided the precondition associated
with the invariant holds.

For synthesising disjunctions of inequalities,
we use a slightly different approach, comprising three steps:
data generation, mining of invariants, and proof of invariants.
The proof part is based on the idea that, in order to prove that there
is no sequence satisfying a conjunction of conditions,
we can represent a set of sequences satisfying each condition
by a constant-size automaton without registers.
Then, a sequence satisfying all the conditions must be accepted
by the intersection of such automata.
If the intersection is empty, then such a sequence does not exist. 

\noindent The contributions of this paper are:
\begin{itemize}
\item
First,
Section~\ref{sec:linear-invariants} provides the basis of a simple,
systematic method to precompute linear inequalities and conditional
linear inequalities for a conjunction of $\Automaton$ constraints on
the same sequence.
We call such inequalities and implications \emph{linear invariants}
and \emph{conditional linear invariants}, respectively.
Each linear invariant and each conditional linear invariant involves
the result variables of the different $\Automaton$ constraints in a
considered conjunction representing the fact that the result variables
cannot vary independently. 
Such invariants may be parametrised by a function of the
sequence length and are independent of the domains of the
sequence variables.
Finally, we describe a systematic method for verifying whether
a linear invariant is a facet of the convex hull or not.
\item
Second,
Section~\ref{sec:non-linear-invariants} shows how to obtain disjunctions of
inequalities, possibly parameterised by a function of the sequence length.
We call such disjunctions \emph{non-linear invariants}.
\item
Third, to mechanise all proofs required in Section~\ref{sec:linear-invariants}
for proving that a linear invariant is facet defining,
and in Section~\ref{sec:non-linear-invariants} for proving non-linear invariants,
Section~\ref{sec:conditional-automata}
defines a special kind of constant\nobreakdash-size
automaton without registers, named \emph{conditional automata}
that recognises all (and only all) sequences satisfying some condition,
e.g.~all sequences maximising the number of peaks.
It shows how to construct such conditional automata in a systematic way.
\item
Fourth, within the context of time-series constraints, \secref{sec:evaluation} shows
the impact of the database of $2000$ synthesised invariants on the propagation of
time-series constraints on short-term electricity production problems.
\end{itemize}

Note that all obtained parameterised invariants are formulae that are \emph{always true}.
Hence they are computed once and for all, put into a \emph{database of parameterised invariants},
and consulted every time when required:
there is no need to rerun our methods for synthesising invariants for every instance.

Adding redundant constraints to a constraint model has been recognised from the very beginning
of Constraint Programming as a major source of improvement~\cite{DincbasSimonisVanHentenryck88}.
Attempts to generate such implied constraints in a systematic way were limited
(1)~by the difficulty to manually prove a large number of conjectures~\cite{HansenCaporossi00,BeldiceanuCarlssonRamponTruchet05},
(2)~by the limitations of automatic proof systems~\cite{FrischMiguelWalsh01,CharnleyColtonMiguel06}, or
(3)~to special cases for very few constraints like
$\constraint{alldifferent}$,
$\constraint{cardinality}$,
$\constraint{element}$~\cite{Lee02,AppaMagosMourtos04,Hooker:2011:IMO:2090089}.
Within the context of register automata, linear invariants relating
consecutive register values of the same constraint were obtained~\cite{ASTRA:GCAI15:ICs}
using Farkas's lemma~\cite{Boyd:convexOpti} in a resource-intensive procedure.


\section{Background \label{sec:background}}

This section presents the necessary background and notation on regular
expressions, register automata, and time-series constraints.
Two complementary facets of time-series constraints will be presented:
first, their declarative definition, second the transducers used to synthesise
an implementation of time-series constraints.
These transducers will be used in Section~\ref{sec:conditional-automata}
to generate a constant\nobreakdash-size automaton associated with an upper bound
minus a constant shift of a time-series constraint.

\subsection{Background on Regular Expressions and Register Automata}

For a regular expression $\pattern$, its language~\cite{Crochemore}
is denoted by $\Language{\pattern}$.
The \emph{size}~\cite{BoundsConstraints} of a regular expression $\pattern$,
denoted by $\Char{\width}{\pattern}$, is the number of letters in the shortest word
of $\Language{\pattern}$.

A \emph{register automaton}~\cite{Beldiceanu:automata:journal} $\Aut$
with $p>0$ registers is a tuple $\Tuple{Q, \Sigma, \delta, q_0, I, A, \alpha}$, where
$Q$ is the set of \emph{states},
$\Sigma$ is the \emph{input alphabet},
$\delta \colon (Q\times\Integers^p) \times \Sigma \rightarrow Q \times \Integers^p$
is the \emph{transition function},
$q_0 \in Q$ is the \emph{initial~state},
$I$ is a sequence of length $p$ of the initial values of the $p$ registers,
$A\subseteq Q$ is the \emph{set of accepting~states}, and 
$\alpha \colon  \Integers^p \rightarrow \Integers$ is a function,
called \emph{acceptance function},
which maps the registers of an accepting~state into an integer.
If, by consuming the symbols of a word $w$ in $\Sigma^*$, the automaton
$\Aut$ triggers a sequence of transitions from $q_0$, its initial state, to some accepting
state where $\Tuple{d_1,d_2,\dots,d_p}$ are the values of the
registers at this stage, then $\Aut$ returns $\alpha(d_1,d_2,\dots,d_p)$, otherwise it \emph{fails}.
In this paper, the input alphabet of the register automata is $\Curly{\reg{<}, \reg{=}, \reg{>}}$.

Within all figures, the acceptance function is depicted by a box
connected by dotted lines to each state.
If a register is left unchanged while triggering a given transition,
then we do not mention this register update on the corresponding transition.

\subsection{Defining Time-Series Constraints}

Given an integer sequence $\x=\xseq$,
a \emph{time-series constraint} $g\_f\_\pattern(\x,\Result)$,
introduced in~\cite{Beldiceanu:synthesis},
restricts $\Result$ to be the result of some computations over an
integer sequence $X = \XSeq$, where:
\begin{itemize}
\item
$\pattern$ is a regular expression~\cite{Crochemore} over the alphabet
$\Alphabet = \{\reg{<},\reg{=},\reg{>}\}$ with which we associate
two integer constants $b_\sigma$ and $a_\sigma$ whole role is explained below;
the sequence $\SigSeq=\langle \Sig_1,\Sig_2,\dots,\Sig_{n-1}\rangle$,
called the \emph{signature} and containing \emph{signature symbols}, is linked to the sequence
$\x$ via the \emph{signature conditions} $(\xvar_i<\xvar_{i+1}\Leftrightarrow \Sig_i=\reg{<})$
$\land~(\xvar_i=\xvar_{i+1}\Leftrightarrow \Sig_i=\reg{=})$
$\land~(\xvar_i>\xvar_{i+1}\Leftrightarrow \Sig_i=\reg{>})$ for all $i\in[1,n-1]$~\cite{Beldiceanu:automata:journal,VeanesHooimeijerLivshitsMolnarBjorner12}.
When~$\Tuple{\Sig_i,\Sig_{i+1},\dots,\Sig_j}$ (with $1\leq i\leq j\leq\seqlength$)
is a maximal word matching~$\pattern$, the
sequence~$\Tuple{\xvar_{i+\Char{\Before}{\pattern}},\xvar_{i+\Char{\Before}{\pattern}+1},\dots,\xvar_{j+1-\Char{\After}{\pattern}}}$ is called a $\pattern$-\emph{pattern};
\item
$f$ is a function over sequences, called \emph{feature},
and is used for computing a value for each $\pattern$-pattern;
the role of the two constants $b_\sigma$ and $a_\sigma$ is to trim the left and right borders
of an occurrence of the regular expression $\pattern$ when computing the feature values;
\item
$g$ is a function over sequences, called \emph{aggregator}, and
is used for aggregating the feature values of the different
$\pattern$\nobreakdash-patterns.
\end{itemize}

The result value $\Result$ of a time-series constraints is restricted
to be the result of aggregation, computed using $g$, of the list of values
of feature $f$ for all $\pattern$-patterns in $\x$.
In this paper, we consider the following class of time-series constraints.
\begin{definition}[value-independent time-series constraints]
\label{def:qre-class}
A time-series constraints $g\_f\_\pattern(\x, \Result)$  is
\emph{\ValueIndependent{}} if any two integer sequences with
the same signature yield the same value of $\Result$.
\end{definition}

We denote by $\QreClass$ the class of all \ValueIndependent{}
time-series constraints.
In the rest of the paper, we only consider  time-series constraints in $\QreClass$,
namely the $\constraint{sum\_one}\_\pattern(\x, \Result)$ and the
$\constraint{sum\_width}\_\pattern(\x, \Result)$ families:
\begin{itemize}
\item
For $\constraint{sum\_one}\_\pattern$, the feature $\One$ denotes the constant function $1$,
and the aggregator $\SumAggr$ is a sum.
Consequently $\Result$ is the number of $\pattern$-patterns of $\x$.
In the following we use $\constraint{nb}\_\pattern$ as a shorthand for
$\constraint{sum\_one}\_\pattern$.
\item
For $\constraint{sum\_width}\_\pattern$, the feature $\Width$ denotes the number of elements
in a $\pattern$-pattern.
Then $\Result$ is the sum of the number of elements of all $\pattern$-patterns of $\x$.
\end{itemize}
If there is no~\SigmaPattern~in~$\x$, then~$\Result$ is the default value of $g$,
which is $0$ in the case of the $\SumAggr$ aggregator.
The \emph{length} of an integer sequence  is the number of its elements.
In the following, we assume non-empty integer sequences.

\begin{example}
\label{ex:peak}
Consider the $\PeakPatternName = \PeakPattern$
and the $\ValleyPatternName = \ValleyPattern$ regular expressions
with the values
$\Char{\Before}{\PeakPatternName}$, $\Char{\After}{\PeakPatternName}$,
$\Char{\Before}{\ValleyPatternName}$ and $\Char{\After}{\ValleyPatternName}$ all
being $1$.
The signature of $\x = \Tuple{0,1,2,2,0,0, 4,1}$ is $\SigSeq = \Tuple{<,<,=,>,=,<,>}$.
There is one maximal occurrence of the $\ValleyPatternName$ regular expression in $\SigSeq$,
namely $\reg{>=<}$.
There are two maximal occurrences of the $\PeakPatternName$ regular expression in $\SigSeq$,
namely $\reg{<<=>}$ and $\reg{<>}$.
Hence, $\constraint{nb}\_\PeakPatternName(\x, 2)$ holds.
The $\PeakPatternName$-pattern $\Tuple{1,2,2}$ (resp.\ $\Tuple{4}$) corresponds
to the first (resp.\ second) maximal occurrence of $\PeakPatternName$ in $\SigSeq$.
The width of the first and the second $\PeakPatternName$-patterns
of $\x$, is, respectively, $3$ and $1$.
The sum of the widths of all $\PeakPatternName$-patterns of $\x$ is $3+1=4$.
Hence, $\constraint{sum\_width}\_\PeakPatternName(\x, 4)$ holds.
\qedexample
\end{example}

\subsection{Operational View of Time-Series Constraints}

Both, to identify all $\pattern$-patterns of an integer sequence $\x$
and to synthesise a register automaton computing the result $\Result$ of a time-series
constraint $g\_f\_\pattern(\x, \Result)$,
the notion of \emph{seed transducer} was introduced in~\cite{Beldiceanu:synthesis}.
It was shown in~\cite{ASTRA:ICTAI17:generation} how to generate such seed transducer
from a regular expression.
For the purpose of this paper, we consider a simplified version of seed transducers
of~\cite{Beldiceanu:synthesis,ASTRA:ICTAI17:generation} that we now present.

A \emph{seed transducer} of $\pattern$ is a deterministic transducer where each
transition is labelled with two letters:
a letter in the input alphabet $\Alphabet = \Curly{\reg{<}, \reg{=}, \reg{>}}$,
called the \emph{input symbols}, and a letter in the output alphabet
$\Omega = \Curly{\Found, \NotFound}$, called the  \emph{output symbols}.
Hence, a transducer consumes the signature $\SigSeq$ of an integer sequence $\x$
and produces an output sequence $T$ where each element is in $\Omega$.
Every element of $\Omega$ is called a \emph{phase letter} and corresponds
to a recognition phase of a new occurrence of $\pattern$ in $\SigSeq$.
Consider different possibilities of the produced symbol $T_i$
when consuming a symbol $\Sig_i$ of $\SigSeq$:
\begin{itemize}
\item
$T_i$ is $\Found$.
A transition labelled by this output symbol corresponds to the discovery
of a new occurrence of $\pattern$ in $\SigSeq$.
\item
$T_i$ is $\NotFound$.
Such transitions do not correspond to the discovery of a new occurrence of $\pattern$ in $\SigSeq$,
but rather to some intermediate phases that do not need to be detailed for the purpose of this paper.
\end{itemize}

A transition labelled with $\Found$ is called a \emph{$\Found$-transition}.
A \emph{$\Found$-path} is any sequence of consecutive transitions of
the transducer containing at least one $\Found$-transition.

\begin{example}
\label{ex:peak-transducer}
Consider the $\PeakPatternName$ regular expression introduced in \exref{ex:peak},
and its seed transducer given in Part~(A) of \figref{fig:peak-transducer-separated}:
\begin{itemize}
\item
the transition from $r$ to $t$ is a single $\Found$-transition,
\item
the sequence of transitions from $s$ to $r$, from $r$ to $t$ and
from $t$ to $r$ is a $\Found$-path.
\end{itemize}
While consuming the signature $\SigSeq = \Tuple{<,<,=,>,=,<,>}$ of 
the integer sequence $\Tuple{0,1,2,2,0,0,4,1}$,
the seed transducer produces the output sequence
$\langle\NotFound,\NotFound,\NotFound,\Found,\NotFound,$ $\Found\rangle$.
As shown in \exref{ex:peak}, $\SigSeq$ contains two maximal occurrences of
$\PeakPatternName$, complying with the two $\Found$ letters in $t$.
\qedexample
\end{example}


\section{Types of Synthesised Invariants}
\label{sec:invariant-types}

Consider a conjunction of two time-series constraints
$\CtrGamma_1(X,\Result_1)$ and $\CtrGamma_2(X,\Result_2)$
imposed on the same sequence of integer variables $X=\XSeq$.
In this section, we present a classification of different types of
invariants that involves $\Result_1$, $\Result_2$ and $\seqlength$.

\paragraph{Farkas Linear Invariants for a Single Constraint}

The method for generating linear invariants based on the Farkas's lemma
was described in~\cite{ASTRA:GCAI15:ICs}, and is used for generating
linear invariants linking the registers of a register automaton
representing a single constraint $\CtrGamma_i$ with $i$ in $\Curly{1,2}$.
Although, this method is fairly general,
the generation of invariants can be time consuming and the
set of generated invariants is too large.
This requires an extra step for selecting the tightest generated invariants.

\paragraph{Linear Invariants for a Conjunction of Constraints}
A contribution of this paper is a systematic method for generating parameterised
linear invariants linking the result variables $\Result_1$ and $\Result_2$
of two time-series constraints.
This method applies for any conjunction of constraints, where each
constraint can be represented by a register automaton, satisfying a
certain property, named  the~\IncrementalAut~property, which will be
introduced in \propref{prop:incremental-automaton} of Section~\ref{sec:linear-invariants}.
The class of automata satisfying the \IncrementalAut~property is
smaller compared to the ones satisfying the conditions of the method
of~\cite{ASTRA:GCAI15:ICs}.
However, it still covers $35$ constraints of the volume~II of the
Global Constraint Catalogue~\cite{Catalog18}. 
We further show in a systematic way that many of the generated
invariants are facets of the convex hull of feasible combinations of
$\Result_1$ and $\Result_2$.

\paragraph{Conditional Linear Invariants for a Conjunction of Constraints}
We also generate conditional parameterised linear invariants, where the
condition may be a requirement on $\seqlength$, $\Result_1$ or $\Result_2$,
e.g.~$\Result_1 > 0 \land \Result_2 > 0$, $\seqlength > 3$.
Such invariants are useful when, for example, a linear invariant is
a facet of the convex hull and holds only for long enough sequences. 
The method for generating such invariants is based on the method for
synthesising linear invariants, and the same conditions on register
automata apply.

\paragraph{Non-Linear Invariants}
The non-linear invariants we synthetise are of the form
$P_1 \lor P_2 \lor \dots \lor P_k$, where every $P_k$ is
a negation of an \emph{atomic relation}. 
We define in \secref{sec:non-linear-invariants} a set of $8$ atomic relations,
some of which are $\Result_i=c$, $\Result_i=\Upp{\Result_i}(\seqlength)-c$,
where $c$ is a natural number, and $\Upp{\Result_i}(\seqlength)$ is the maximum
value of $\Result_i$ among all time series of length $\seqlength$~\cite{BoundsConstraints}.
Such invariants are required when the set of feasible combinations of
$\Result_1$ and $\Result_2$ is non-convex and therefore
linear invariants are not enough for fully describing it.


\section{Synthesising Parameterised Linear Invariants}
\label{sec:linear-invariants}

Consider $\NAut$ register automata $\Aut_1,\Aut_2,\dots,\Aut_{\NAut}$ over the same alphabet~$\Sigma$.
Let $\acc_i$ denote the number of registers of $\Aut_i$, and let~$\FinalV_i$ designate its returned value.
In this section we show how to systematically generate linear invariants of the form

\begin{equation}
\label{inv:main}
\LinFun \geq 0~~\textnormal{with}~~e,e_0,e_1,\dots,e_k\in\mathbb{Z},
\end{equation}

which hold after the signature of the same input sequence~$\XSeq$
is completely consumed by the $\NAut$ register automata $\Aut_1,\Aut_2,\dots,\Aut_{\NAut}$.
We call such linear invariant \emph{general} since it holds regardless of any
conditions on the result variables $\FinalV_1,\FinalV_2,\dots,\FinalV_{\NAut}$.
Stronger, but less general, invariants may be obtained when the
initial values of the registers cannot be assigned to the result variables.

Our method for generating invariants is applicable to a restricted
class of register automata that we now introduce.
\begin{property}[\IncrementalAut~property]
\label{prop:incremental-automaton}
A register automaton~$\Aut$ with~$\acc$ registers has the \emph{\IncrementalAut} property
if the following four conditions are all satisfied:
\begin{enumerate}
\item\label{enum:inc_auto_0}
For every register~$A_j$ of~$\Aut$, its initial value~$\InitialV{j}$ is a natural number.
\item\label{enum:inc_auto_1}
For every register~$A_j$ of~$\Aut$ and for every transition $\Tr$ of~$\Aut$,
the update of~$A_j$ upon triggering transition $\Tr$ is of the form
$A_j \gets \TrCoeff{\Tr}{j}{0} + \sum \limits_{i = 1}^{\acc} \TrCoeff{\Tr}{j}{i}\cdot A_i$, with
$\TrCoeff{\Tr}{j}{0}\in\mathbb{N}$ and $\TrCoeff{\Tr}{j}{1},\TrCoeff{\Tr}{j}{2},\dots,\TrCoeff{\Tr}{j}{\acc}\in\{0,1\}$.
\item\label{enum:inc_auto_2}
The register~$A_{\acc}$ is called the \emph{main register} and
verifies all the following three conditions:
  \begin{enumerate}
  \item\label{enum:inc_auto_2a}
  the value returned by $\Aut$ is the last value of its main register $A_{\acc}$,
  \item\label{enum:inc_auto_2b}
  for every transition~$\Tr$ of~$\Aut$, $\TrCoeff{\Tr}{\acc}{\acc} =1$,
  \item\label{enum:inc_auto_2c}
  for a non\nobreakdash-empty subset~$T$ of transitions of~$\Aut$,
  $\sum \limits_{i = 1}^{\acc-1} \TrCoeff{\Tr}{\acc}{i} > 0, ~ \forall \Tr \in T$.
\end{enumerate}
\item\label{enum:inc_auto_3}
For all other registers~$A_j$ with $j<\acc$, on every transition~$\Tr$
of~$\Aut$, we have $\sum \limits_{i=1,i\neq j}^{\acc}\TrCoeff{\Tr}{j}{i}=0$ and,
if~$\TrCoeff{\Tr}{r}{j}>0$, then $\TrCoeff{\Tr}{j}{j}$ is $0$.
\end{enumerate}
\end{property}

The intuition behind the~\IncrementalAut~property is that there is one register
that we name the \emph{main register}, whose last value is the final value,
returned by the register automaton, (see~\ref{enum:inc_auto_2a}).
At some transitions, the update of the main register is a linear
combination of the other registers, while on the other transitions its
value either does not change or is incremented by a non-negative constant,
(see~\ref{enum:inc_auto_2b} and \ref{enum:inc_auto_2c}).
All other registers may only be incremented by a non-negative constant or
assigned to some non-negative integer value, and they \emph{may} contribute to
the final value, (see~\ref{enum:inc_auto_3}).
These registers are called \emph{potential registers}.
Both register automata in Parts~(A) and~(B) of~\figref{fig:peak-valley-separated}
have the~\IncrementalAut~property, and their single registers are the main registers.
Volumes~I and~II of the global constraint catalogue contain more than~$50$ such register automata.
In particular, in Volume~II, the register automata for all the constraints of the
$\constraint{nb}\_\pattern$ and the $\constraint{sum\_width}\_\pattern$ families
have the \IncrementalAut~property.
In the rest of this paper we assume that all register automata~$\Aut_1,\Aut_2,\dots,\Aut_{\NAut}$
have the~\IncrementalAut~property.

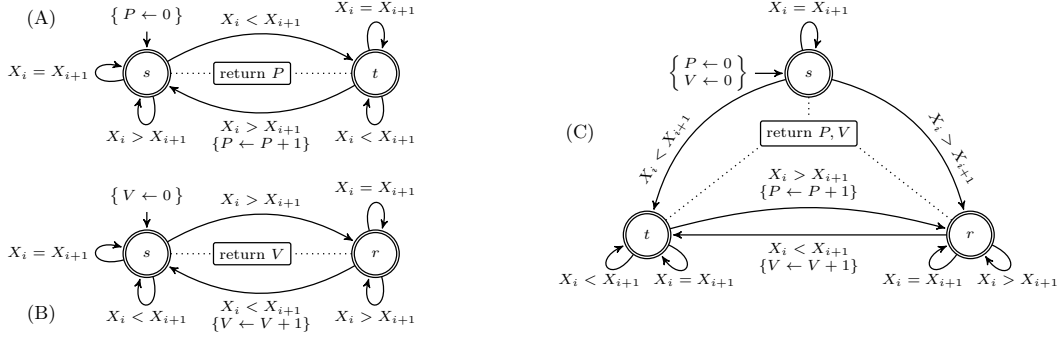
\begin{figure}
{\centering
\scalebox{0.8}{
\newcommand{\InitialText}{\scriptsize $\left \{\begin{array}{l}\Peak \leftarrow 0\\ \Valley \leftarrow 0 \end{array}\right\}$}
\begin{tikzpicture}[->,>=stealth',shorten >=1pt,auto,node distance=38mm,semithick,
                    information text/.style={rounded corners=1pt,inner sep=1ex},font=\scriptsize]
\begin{scope}[xshift=1cm]
\node[initial, accepting, initial where=above,
      initial text={\scriptsize $\left \{\begin{array}{l} P\leftarrow 0\end{array}\right\}$},
      initial distance=3mm,state,draw=black](s) {$s$};
\node[state,draw=black, accepting] (t) [right of=s]{$t$};
\node[rectangle,draw,rounded corners=1pt] at (1.75,0) (ctr1) {return $P$};
\path
 (s) edge [draw=black, loop left]  node[left]  {$X_i=X_{i+1}$} (s)
 (s) edge [draw=black, loop below] node[below] {$X_i>X_{i+1}$} (s)
 (s) edge [draw=black, bend left]  node[above] {$X_i<X_{i+1}$} (t)
 (t) edge [draw=black, loop above] node[above] {$X_i=X_{i+1}$} (t)
 (t) edge [draw=black, loop below] node[below] {$X_i<X_{i+1}$} (t)
 (t) edge [draw=black, bend left]  node[below]{$\begin{array}{c} X_i > X_{i+1} \\
                                                                 \Curly{P\leftarrow P+1}
                                                \end{array}$ } (s);
\draw [dotted,-] (s) -- (ctr1);
\draw [dotted,-] (t) -- (ctr1);
\node at (-1.75,0.9) {\normalsize (A)};
\end{scope}
\begin{scope}[xshift=1cm, yshift=-3cm]
      \node[initial, accepting, initial where=above,
            initial text={\scriptsize $\left \{\begin{array}{l} V\leftarrow 0\end{array}\right\}$},
            initial distance=3mm,state,draw=black](s) {$s$};
\node[state,draw=black, accepting] (r) [right of=s]{$r$};
\node[rectangle,draw,rounded corners=1pt] at (1.75,0) (ctr2) {return $V$};
\path
 (s) edge [draw=black, loop left] node[left]{$X_i = X_{i+1}$} (s)
 (s) edge [draw=black, loop below] node[below]{$X_i < X_{i+1}$} (s)
 (s) edge [draw=black, bend left]  node[above]{$X_i > X_{i+1}$ } (r)
 (r) edge [draw=black, loop above] node[above]{$X_i = X_{i+1}$ } (r)
 (r) edge [draw=black, loop below] node[below]{$X_i > X_{i+1}$ } (r)
 (r) edge [draw=black, bend left]  node[below]{$\begin{array}{c} X_i < X_{i+1} \\
                                                                 \Curly{V\leftarrow V+1}
                                                \end{array}$ } (s);
\draw [dotted,-] (s) -- (ctr2);
\draw [dotted,-] (r) -- (ctr2);
\node at (-1.75,-1) {\normalsize (B)};
\end{scope}
\begin{scope}[xshift=12cm,yshift=0cm]
\node[initial,initial where=left,initial text=\InitialText,initial distance=5mm,state,draw=black, accepting](s) {$s$};
\node[state, draw=black, accepting] (t) [below left of=s]{$t$};
\node[state, draw=black, accepting] (r) [below right of=s]{$r$};
\node[rectangle,draw,rounded corners=1pt] at (0,-1) (ctr3) {return $P,V$};
\path
 (s) edge [draw=black, loop above] node[above] {$X_i = X_{i+1}$}   (s)
 (s) edge [draw=black, bend left]  node[pos=0.7,above=0.05,sloped] {$X_i > X_{i+1}$}  (r)
 (s) edge [draw=black, bend right] node[pos=0.7,above=0.05,sloped] {$X_i < X_{i+1}$}  (t)
 (r) edge [draw=black, in=205,out=235,loop] node[below] {$X_i = X_{i+1}~~~$} (r)
 (r) edge [draw=black, in=305,out=335,loop] node[below] {$~~~X_i > X_{i+1}$} (r)
 (t) edge [draw=black, in=205,out=235,loop] node[below] {$X_i < X_{i+1}~~~$} (t)
 (t) edge [draw=black, in=305,out=335,loop] node[below] {$~~~X_i = X_{i+1}$} (t)
 (r) edge [draw=black]             node[below] {$\begin{array}{c} X_i < X_{i+1} \\
                                                                  \{\Valley\leftarrow\Valley+1\}
                                                 \end{array}$} (t)
 (t) edge [draw=black, bend angle=15, bend left] node[above]{$\begin{array}{c} X_i > X_{i+1} \\
                                                                               \{\Peak\leftarrow\Peak+1\}
                                                              \end{array}$} (r);
\draw [dotted,-] (s) -- (ctr3);
\draw [dotted,-] (t) -- (ctr3.south west);
\draw [dotted,-] (r) -- (ctr3.south east);
\node at (-3.8,-1) {\normalsize (C)};
\end{scope}
\end{tikzpicture}
}}

\caption{\label{fig:peak-valley-separated}
(A)~Register automaton for $\constraint{nb\_peak}$;
(B)~Register automaton for $\constraint{nb\_valley}$;
(C)~Intersection of~(A) and~(B).
}
\end{figure}

Our approach for systematically generating linear invariants of type
$\LinFun \geq 0$ considers each combination of signs of the coefficients
$e_i$ (with $i\in[0,\NAut]$).
It consists of three steps:
\begin{enumerate}
\item
Construct a non-negative function $v = \LinFun$, which represents the left-hand side
of the sought linear invariant (see~\secref{sec:constructing-automaton}).
\item
Select the coefficients $e_0,e_1,\dots,\Coeff{\NAut}$, called the \emph{relative coefficients}
of the linear invariant, so that there exists a constant $C$ such that $\LinFunVar\geq C$
(see~\secref{sec:relative-terms}).
\item
Compute $C$ and set the coefficient $e$, called the \emph{constant term}
of the linear invariant, to $-C$ (see~\secref{sec:absolute-term}).
\end{enumerate}

\noindent The three previous steps are performed as follows:
\begin{enumerate}
\item
First, we assume a sign for each coefficient $e_i$ (with $i\in[0,\NAut]$),
which tells whether we have to consider or not the contribution of the potential registers;
note that each combination of signs of the coefficients $e_i$ (with $i\in[0,\NAut]$)
will lead to a different linear invariant.
Then, from the intersection $\Product$ of $\Aut_1,\Aut_2,\dots,\Aut_{\NAut}$,
we construct a digraph  called the \emph{invariant digraph},
where each transition $\Tr$ of $\Product$ is replaced by an arc whose weight represents
the lower bound of the variation of the term $\LinFunVar$ while triggering~$\Tr$.
\item
Second, we find the coefficients $e_i$ (with $i\in[0,\NAut]$) so that
the invariant digraph does not contain any negative cycles.
When the invariant digraph has no negative cycles, the value of $\LinFunVar$
is bounded from below for any integer sequence.
\item
Third, to obtain $C$ we compute the shortest path in the invariant digraph from the node 
of the invariant digraph corresponding to the initial state of $\Product$, to all nodes
corresponding to accepting states of $\Product$.
\end{enumerate}

\subsection[Constructing the Invariant Digraph for a Conjunction of 
            $\Automaton$ Constraints wrt a Linear Function]{Constructing
            the Invariant Digraph for a Conjunction of 
            $\Automaton$ Constraints \\ wrt a Linear Function}
\label{sec:constructing-automaton}

First, \defref{def:invariant-graph} introduces the notion of \emph{invariant digraph}~$\Digraph{\Product}{v}$
of the register automaton $\Product = \Aut_1 \cap \Aut_2 \cap \dots \cap \Aut_{\NAut}$ wrt
a linear function~$v$ involving the values returned by these register automata.
Second, \defref{def:sequence-walk} introduces the notion of \emph{weight of an accepting sequence~$X$}
wrt $\Product$ in $\Digraph{\Product}{v}$, which makes the link between a path in~$\Digraph{\Product}{v}$
and the vector of values returned by~$\Product$ after consuming the signature of~$X$.
Finally, \thref{th:linear-invariants-main} shows that the weight of~$X$ in~$\Digraph{\Product}{v}$
is a lower bound on the linear function~$v$.

\begin{definition}[invariant digraph]
\label{def:invariant-graph}
Consider an accepting sequence~$X=\XSeq$ wrt the register automaton $\Intersect$,
and a linear function~$v=\LinFun$, where $(\FinalV_1,\FinalV_2,\dots,\FinalV_{\NAut})$
is the vector of values returned by $\Product$ after consuming the signature of~$X$.
The \emph{invariant digraph} of~$\Product$ wrt~$v$, denoted by $\Digraph{\Product}{v}$,
is a weighted digraph defined in the following way:

\begin{itemize}
\item
The set of nodes of~$\Digraph{\Product}{v}$ is the set of states of~$\Product$.
\item
The set of arcs of~$\Digraph{\Product}{v}$ is the set of transitions of~$\Product$,
where for every transition~$\Tr$, the corresponding symbol of the alphabet is
replaced by an integer weight, which is
$\Coeff{0}+\sum\limits_{i=1}^{\NAut}\Coeff{i}\cdot\beta^{\Tr}_i$,
where~$\beta^{\Tr}_i$ is defined as follows:
\begin{numcases}{\beta^{\Tr}_i = }
\TrCoeff{\Tr}{i,\acc_i}{0}                           & \text{if~} $\Coeff{i}\geq 0$, \label{cond:definition-graph-pos} \\
\sum \limits_{j = 1}^{\acc_i} \TrCoeff{\Tr}{i, j}{0} & \text{if~} $\Coeff{i}<0$,     \label{cond:definition-graph-neg} 
\end{numcases}

where~$\acc_i$ denotes the number of registers of $\Aut_i$, and
$\TrCoeff{\Tr}{i, p}{0}$ (with $p\in[1,\acc_i]$) is the constant
in the update of the register of~$\Product$ corresponding to the register~$p$ of~$\Aut_i$.
\end{itemize}
\end{definition}

\begin{definition}[walk and weight of an accepting sequence]
\label{def:sequence-walk}
Consider an accepting sequence~$X$ of length $\seqlength$ wrt the register automaton $\Intersect$,
and a linear function $v = \LinFun$, where $(\FinalV_1,\FinalV_2,\dots,\FinalV_{\NAut})$ is
the vector of values returned by $\Product$ after consuming the signature of~$X$.
\begin{itemize}
\item
The \emph{walk of~$X$ in~$\Digraph{\Product}{v}$} is a path in~$\Digraph{\Product}{v}$
whose sequence of arcs is the sequence of the corresponding transitions of~$\Product$
triggered upon consuming the signature of~$X$.
\item
The \emph{weight of~$X$ in~$\Digraph{\Product}{v}$} is the weight of its path
in~$\Digraph{\Product}{v}$ plus a constant value, which is a lower bound
on~$v$ corresponding to the initial values of the registers and is called
the \emph{initialisation weight} in~$\Digraph{\Product}{v}$.
It equals~$\Coeff{}+\Coeff{0}\cdot(p-1)+\sum _{i=1}^{\NAut}\Coeff{i}\cdot\beta_i^0$,
where~$p$ is the arity of the signature, and where~$\beta_i^0$ is defined as follows:

\begin{numcases}{\beta_i^0 = }
\InitialV{i,\acc_i}                       & \text{if~} $\Coeff{i}\geq0$, \label{cond:definition-sequence-weight-pos}\\
\sum \limits_{j=1}^{\acc_i}\InitialV{i,j} & \text{if~} $\Coeff{i}<0$,    \label{cond:definition-sequence-weight-neg}
\end{numcases}

where~$\acc_i$ denotes the number of registers of $\Aut_i$,
and $\InitialV{i,p}$ (with $p\in[1,\acc_i]$) is the initial value
of the register of~$\Product$ corresponding to the register $p$ of~$\Aut_i$.
\end{itemize}
\end{definition}

\begin{example}
\label{ex:peak-valley-automaton}
Consider the $\constraint{peak}(X,\Peak)$ and the $\constraint{valley}(X,\Valley)$ constraints
introduced in~\exref{ex:peak} on the same sequence $X = \Tuple{X_1,X_2,\dots,X_\seqlength}$.
\figref{fig:peak-valley-separated} gives the automata for $\constraint{peak}$, $\constraint{valley}$,
and their intersection $\Product$.
We aim to find inequalities of the
form~$\Coeff{}+\Coeff{0}\cdot\seqlength+\Coeff{1}\cdot\Peak+\Coeff{2}\cdot\Valley\geq 0$~that hold for every
\vspace{-0.4cm}
\begin{multicols}{2}
\noindent integer sequence~$X$.
After consuming the signature of~$X$, $\Product$ returns a pair of values $(\Peak,\Valley)$,
which are the number of peaks (resp.\ valleys) in~$X$.
The invariant digraph of
$\Product$ wrt~$v =  \Coeff{} + \Coeff{0} \cdot \seqlength +$ $\Coeff{1} \cdot \Peak + \Coeff{2} \cdot \Valley$
is given in the figure on the right.
As neither of the two automata has any potential registers,
the weights of the arcs of~$\Digraph{\Product}{v}$
do not depend on the signs of~$\Coeff{1}$ and $\Coeff{2}$.
Hence, for every integer sequence~$X$, its weight in~$\Digraph{\Product}{v}$
equals~$\Coeff{}+\Coeff{0}\cdot\seqlength+\Coeff{1}\cdot\Peak+\Coeff{2}\cdot\Valley$. \qedexample

\scalebox{0.8}{
\newcommand{\InitialText}{\scriptsize $\left \{\begin{array}{l}\Peak \leftarrow 0\\ \Valley \leftarrow 0 \end{array}\right\}$}
\begin{tikzpicture}[->,>=stealth',shorten >=1pt,auto,node distance=35mm,semithick,
                    information text/.style={rounded corners=1pt,inner sep=1ex},font=\normalsize]
\begin{scope}
\node[state] (s)                    {$s$};
\node[state] (t) [below left of=s]  {$t$};
\node[state] (r) [below right of=s] {$r$};
\path
 (s)     edge [loop left]                node[left]{$e_0$}      (s)
 (s.330) edge [bend left]                node[right]{$~e_0$}    (r)
 (s.210) edge [bend right]               node[left] {$e_0~$}    (t)
 (r)     edge [loop below]               node[below]{$e_0$}     (r)
 (t)     edge [loop below]               node[below] {$e_0$}    (t)
 (r)     edge                            node[below]{$e_0+e_2$} (t)
 (t)     edge [bend angle=15, bend left] node[above]{$e_0+e_1$} (r);
\node at (0,0.8) {~};
\end{scope}
\end{tikzpicture}
}

\end{multicols}
\end{example}

\begin{theorem}[lower bound on the weight of an accepting sequence]
\label{th:linear-invariants-main}
Consider an accepting sequence~$X=\XSeq$ wrt the register automaton~$\Intersect$,
and a linear function $v = \LinFun$, where~$(\FinalV_1,\FinalV_2,\dots,\FinalV_{\NAut})$
is the vector of values returned by~$\Product$.
Then, the weight of~$X$ in~$\Digraph{\Product}{v}$ is less than or equal to~$\LinFun$.
\end{theorem}

\begin{proof}
Since, when doing the intersection of register automata we do not merge registers, 
the registers of $\Product$ that come from different register automata do not interact,
i.e.~their updates are independent, hence their returned values are also independent.
By definition of the invariant digraph, the weight of any of its arc is
$\Coeff{0}+\sum\limits_{i=1}^{\NAut}\Coeff{i}\cdot\beta^{\Tr}_i$,
where~$\beta^{\Tr}_i$ depends on the sign of~$\Coeff{i}$, and
where~$\Tr$ is the corresponding transition in~$\Product$.
Then, the weight of~$X$ in~$\Digraph{\Product}{v}$ is the constant
$\Coeff{}+\Coeff{0}\cdot(p-1)+\sum \limits_{i=1}^{\NAut}\Coeff{i}\cdot\beta_i^0$
(see~\defref{def:sequence-walk}) plus the weight of the walk of~$X$, which is in total
$\Coeff{}+\Coeff{0}\cdot(p-1)+\sum \limits_{i=1}^{\NAut}\Coeff{i}\cdot\beta_i^0+\Coeff{0}\cdot(\seqlength-p+1)+\sum\limits_{j=1}^{\seqlength-p+1}\sum\limits_{i=1}^{\NAut}\Coeff{i}\cdot\beta^{\Tr_j}_i=\Coeff{}+\Coeff{0}\cdot\seqlength+\sum\limits_{i=1}^{\NAut}
\Coeff{i}\cdot\left(\beta_i^0+\sum\limits_{j=1}^{\seqlength-p+1}\beta^{\Tr_j}_i\right)$,
where $p$ is the arity of the considered signature,
and~$\Tr_1,\Tr_2,\dots\Tr_{\seqlength-p+1}$ is the sequence of transitions of~$\Product$
triggered upon consuming the signature of~$X$.
We now show that the value~$\Coeff{i}\cdot\left(\beta_i^0+\sum\limits_{j=1}^{\seqlength-p+1}\beta^{\Tr_j}_i\right)$
is not greater than~$\Coeff{i}\cdot\FinalV_i$.
This will imply that the weight of the walk of~$X$ in~$\Digraph{\Product}{v}$ is less than or equal to~$v=\LinFun$.

Consider the~$v_i=\Coeff{i}\cdot\FinalV_i$ linear function.
We show that the weight of~$X$ in~$\Digraph{\Product}{v_i}$,
which equals~$\Coeff{i}\cdot\left(\beta_i^0+\sum\limits_{j=1}^{\seqlength-p+1}\beta^{\Tr_j}_i\right)$,
is less than or equal to~$\Coeff{i}\cdot\FinalV_i$.
Depending on the sign of~$\Coeff{i}$ we consider two cases.

\noindent{\bf Case~1: $\mathbf{\Coeff{i}\geq 0}$}.
In this case, the weight of every arc of $\Digraph{\Product}{v_i}$ is~$\Coeff{i}$
multiplied by $\TrCoeff{\Tr}{\acc_i}{0}$, where~$\Tr$ is the corresponding transition
in~$\Product$, and~$\acc_i$ is the main register of~$\Aut_i$
(see~\caseref{cond:definition-graph-pos} of \defref{def:invariant-graph}).
If, on transition $\Tr$, some potential registers of~$\Aut_i$
are incremented by a positive constant, the real contribution
of the register updates on this transition to $\FinalV_i$ is at
least~$\TrCoeff{\Tr}{\acc_i}{0}$ since $\Coeff{i}\geq 0$.
The same reasoning applies to the contribution of the initial values
of the potential registers to the final value $\FinalV_i$.
Since this contribution is non-negative, it is ignored, and $\beta_i^0 = \InitialV{j}$
(see~\caseref{cond:definition-graph-pos} of \defref{def:sequence-walk}).
Hence $\Coeff{i}\cdot\left(\beta_i^0+\sum\limits_{j=1}^{\seqlength-p+1}\beta^{\Tr_j}_i\right)=\Coeff{i}\cdot\left(\InitialV{\acc_i}+\sum\limits_{j=1}^{\seqlength-p+1}\TrCoeff{\Tr}{\acc_i}{0}\right)\leq\Coeff{i}\cdot\FinalV_i$.

\noindent{\bf Case~2: $\mathbf{\Coeff{i}<0}$}.
In this case, the weight of every arc of~$\Digraph{\Product}{v_i}$ is~$\Coeff{i}$
multiplied by the sum of the non-negative constants, which come from the updates
of \emph{every} register of~$\Aut_i$
(see~\caseref{cond:definition-sequence-weight-neg} of \defref{def:invariant-graph}).
The contribution of the potential registers is always taken into account, and
since~$\Coeff{i}<0$, it is always negative.
The same reasoning applies to the contribution of the initial values of the potential registers
to the returned value $\FinalV_i$.
To obtain a lower bound on~$v$, observe that the initial values of the potential registers are
non-negative and that $\Coeff{i}<0$; therefore we assume that the initial values of the potential
registers always contribute to~$\FinalV_i$ (see~\caseref{cond:definition-graph-neg} of \defref{def:sequence-walk}).
Hence $\Coeff{i}\cdot(\beta_i^0+\sum\limits_{j=1}^{\seqlength-p+1}\beta^{\Tr_j}_i)\leq\Coeff{i}\cdot\FinalV_i$.
\hspace{\fill} \qed
\end{proof}

Note that, if all the considered register automata $\Aut_1,\Aut_2,\dots,\Aut_{\NAut}$
do not have potential registers, then for every accepting sequence~$X=\XSeq$
wrt~$\Intersect$ and for any linear function~$v=\LinFun$, the weight of~$X$
in~$\Digraph{\Product}{v}$ is \emph{equal} to~$v$. 
If there is at least one potential register for at least one register automaton~$\Aut_i$,
then there may exist an accepting sequence~$X=\XSeq$ wrt~$\Intersect$ whose weight
in~$\Digraph{\Product}{v}$ is strictly less than~$v$.

\subsection{Finding the Relative Coefficients of the Linear Invariant}
\label{sec:relative-terms}
We now focus on finding the relative coefficients~$\Coeff{0},\Coeff{1},\dots,\Coeff{\NAut}$ of
the linear invariant~$v=\LinFun\geq0$ such that, after consuming the signature of any accepting sequence
by the register automaton~$\Intersect$, the value of~$v$ is non-negative. 

For any accepting sequence~$X$ wrt~$\Product$, by~\thref{th:linear-invariants-main},
we have that the weight~$\Weight$ of~$X$ in~$\Digraph{\Product}{v}$ is less than or equal to~$v$.
Recall that $\Weight$ consists of a constant part, and of a part that depends on~$X$,
which involves the coefficients $\Coeff{0},\Coeff{1},\dots,\Coeff{\NAut}$; thus,
these coefficients must be chosen in a way that there exists a constant~$C$
such that~$\Weight \geq C$, and $C$ does not depend on~$X$.
This is only possible when~$\Digraph{\Product}{v}$ does not contain \emph{any} negative cycles.
Let~$\Cycles$ denote the set of all simple circuits of~$\Digraph{\Product}{v}$,
and let~$\Weight_e$ denote the weight of an arc~$e$ of~$\Digraph{\Product}{v}$.
In order to prevent negative cycles in $\Digraph{\Product}{v}$, we solve the following
minimisation problem, parameterised by $(s_0,s_1,\dots s_{\NAut})$,
the signs of~$\Coeff{0},\Coeff{1},\dots,\Coeff{\NAut}$:
\begin{IEEEeqnarray}{l.L.l}
\textnormal{minimise }  & \sum\limits_{c\in\Cycles} W_c+\sum\limits_{i=1}^{\NAut}|\Coeff{i}| \label{fun:obj} \\
\textnormal{subject to} & W_c=\sum\limits_{e\in c}\Weight_e &~~\forall c\in\Cycles \label{cond:1} \\
                        & W_c\geq 0 &~~\forall c\in\Cycles \label{cond:2}  \\
                        & s_i=\textnormal{`$-$'}\Rightarrow\Coeff{i}\leq 0,~~
                          s_i=\textnormal{`$+$'}\Rightarrow\Coeff{i}\geq 0  &~~\forall i \in [0, \NAut] \label{cond:3} \\
                        & \Coeff{i}\neq 0 &~~\forall i\in[1,\NAut] \label{cond:5}
\end{IEEEeqnarray}

In order to obtain the coefficients~$\Coeff{0},\Coeff{1},\dots,\Coeff{\NAut}$
so that~$\Digraph{\Product}{v}$ does not contain any negative cycles, it is
enough to find a solution to the satisfaction problem (\ref{cond:1})-(\ref{cond:5}).
Minimisation is required to obtain linear invariants that eliminate as many infeasible values
of~$(\FinalV_1,\FinalV_2,\dots,\FinalV_{\NAut})$ as possible.
Within the objective function (\ref{fun:obj}), the term $\sum\limits_{c\in\Cycles} W_c$
is for minimising the weight of every simple circuit,
while the term $\sum\limits_{i=1}^{\NAut} |\Coeff{i}|$ is for obtaining the coefficients with the smallest absolute value.
By changing the sign vector~$(s_0,s_1,\dots s_{\NAut})$ we obtain different linear invariants.

\begin{example}[finding the relative coefficients]
\label{ex:peak-valley-relative-terms}
Consider $\constraint{nb\_peak}(X, \Peak)$ and $\constraint{nb\_valley}(X,$ $\Valley)$
with~$X$ being a time series of length $\seqlength$.
The invariant digraph of the intersection of the register automata for the~$\constraint{nb\_peak}$ and
$\constraint{nb\_valley}$ constraints wrt~$v=\Coeff{}+\Coeff{0}\cdot\seqlength+\Coeff{1}\cdot\Peak+\Coeff{2}\cdot\Valley$
was given in~\exref{ex:peak-valley-automaton}.
This digraph has four simple circuits, namely $s-s$, $t-t$, $r-r$, and $r-t-r$,
which are labelled by~$1$, $2$, $3$ and~$4$, respectively.
Then, the minimisation problem for finding the relative coefficients of the linear invariant $v\geq 0$,
parameterised by $(s_0,s_1,s_2)$, the signs of $\Coeff{0}$, $\Coeff{1}$ and $\Coeff{2}$,
is the following:

\begin{IEEEeqnarray}{l.L.l}
  \textnormal{minimise } &  \sum\limits_{j=1}^{4} W_j+\sum\limits_{i=0}^2 |\Coeff{i}| \nonumber \\
  \textnormal{subject to } & W_j=\Coeff{0}, &~~\forall j\in [1,3] \nonumber \\
                           & W_{4}=\Coeff{0}+\Coeff{1}+\Coeff{2} \nonumber \\
                           & W_j\geq 0 &~~\forall j\in[1,4] \label{cond:ex} \\
                           & s_i=\textnormal{`$-$'}\Rightarrow\Coeff{i}\leq 0,~~ s_i=\textnormal{`$+$'}\Rightarrow\Coeff{i}\geq 0 &~~\forall i\in [0,2] \nonumber \\
                           & \Coeff{i}\neq 0 &~~\forall i\in [1,2] \nonumber
\end{IEEEeqnarray}

Note that the value of~$\Coeff{0}$ must be non-negative otherwise (\ref{cond:ex})~cannot be satisfied for $j\in\Curly{1,2,3}$.
Hence we consider only the combinations of signs of the form~$(\textnormal{`$+$'},s_1,s_2)$
with $s_1$ and $s_2$ being either `$-$' or `$+$'.
The following table gives the optimal solution of the minimisation problem for the considered combinations of signs:

\vspace{0.2cm}
\setlength{\tabcolsep}{1em}
\hspace*{1.5cm}\begin{tabular}{ccccc}
\toprule
$(s_0, s_1, s_2)$                 & $(+,-,-)$   & $(+,-,+)$  & $(+,+,-)$  & $(+,+,+)$ \\

$(\Coeff{0},\Coeff{1},\Coeff{2})$ & $(1,-1,-1)$ & $(0,-1,1)$ & $(0,1,-1)$ & $(0,1,1)$ \\ 

\bottomrule
\end{tabular}
\qedexample
\end{example}

\subsection{Finding the Constant Term of the Linear Invariant}
\label{sec:absolute-term}
Finally, we focus on finding the constant term~$\Coeff{}$ of the linear invariant $v=\LinFun\geq0$,
when the coefficients $\Coeff{0},\Coeff{1},\dots,\Coeff{\NAut}$ are known, and
when the digraph of the register automaton $\Intersect$ wrt $v$ does not contain any negative cycles.
By~\thref{th:linear-invariants-main}, the weight of any accepting sequence~$X$ wrt~$\Product$
in~$\Digraph{\Product}{v}$ is less than or equal to~$v$, then if the weight of~$X$ is non-negative,
it implies that~$v$ is also non-negative.
Since the invariant digraph~$\Digraph{\Product}{v}$ does not contain any negative cycles,
then the weight of~$X$ cannot be smaller than some constant~$C$.
Hence it suffices to find this constant and set the constant term~$\Coeff{}$ to $-C$.
The value of~$C$ is computed as the constant~$\Coeff{0}\cdot(p-1)-\sum\limits_{i=1}^{\NAut} \beta_i^0$
(see~\defref{def:sequence-walk}) plus the shortest path length from the node of~$\Digraph{\Product}{v}$
corresponding to the initial state of~$\Product$ to all the nodes of~$\Digraph{\Product}{v}$
corresponding to the accepting states of~$\Product$.

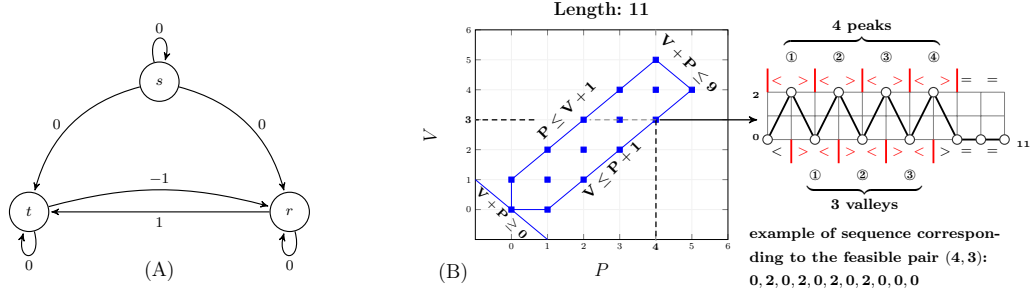
\begin{figure}[!t]
\scalebox{0.7}{
\begin{tikzpicture}[font=\scriptsize]
\begin{scope}[yshift=9cm,
              ->,>=stealth',shorten >=1pt,auto,node distance=35mm,semithick,
              information text/.style={rounded corners=1pt,inner sep=1ex}]
\node[state] (s) {\small$s$};
\node[state] (t) [below left of=s]  {\small$t$};
\node[state] (r) [below right of=s] {\small$r$};
\path
 (s) edge [loop above]               node[above]{\small$0$}  (s)
 (s) edge [bend left]                node[right]{\small$~0$} (r)
 (s) edge [bend right]               node[left] {\small$0~$} (t)
 (r) edge [loop below]               node[below]{\small$0$}  (r)
 (t) edge [loop below]               node[below]{\small$0$}  (t)
 (r) edge                            node[below]{\small$1$}  (t)
 (t) edge [bend angle=15, bend left] node[above]{\small$-1$} (r);
\node at (0,-3.6) {\large (A)};
\end{scope}
\begin{scope}[xshift=6cm, yshift=6cm, scale = 0.7]
\begin{axis}[xlabel={\Large \bf $P$},
    ylabel={\Large \bf $V$},
    title={\Large \bf Length: 11},
    minor tick num=1,
    grid=both,grid style={line width=.1pt, draw=gray!10},
    major grid style={line width=.1pt,draw=gray!10},
    xmin=-1,xmax=6,
    ymin=-1,ymax=6,
    xtick={0,1,2,3,5,6},
    ytick={0,1,2,4,5,6},
    extra x ticks={4},
    extra x tick style={xticklabel style={font=\boldmath}},
    extra y ticks={3},
    extra y tick style={yticklabel style={font=\boldmath}},
    scatter/classes={
     c1={mark=square*,blue!100},
     c2={mark=triangle*,blue!90},
     c3={mark=triangle*,blue!80},
     c4={mark=diamond*,blue!70},
     c5={mark=pentagon*,blue!60},
     c6={mark=*,blue!40},
     c7={mark=x,blue!30},
     c8={mark=+,blue!20},
     c9={mark=o,blue!10},
     c10={mark=o,green!30},
     c11={mark=o,green!20},
     c12={mark=o,green!10}}]
    \addplot[scatter,only marks,scatter src=explicit symbolic]
    coordinates {(0,0) [c1]
                 (1,0) [c1]
                 (1,1) [c1]
                 (2,1) [c1]
                 (2,2) [c1]
                 (3,2) [c1]
                 (3,3) [c1]
                 (4,3) [c1]
                 (4,4) [c1]
                 (5,4) [c1]
                 (0,1) [c1]
                 (1,2) [c1]
                 (2,3) [c1]
                 (3,4) [c1]
                 (4,5) [c1]};
    \addplot[blue]
    coordinates {(0,0)
                 (1,0)
                 (5,4)
                 (4,5)
                 (0,1)
                 (0,0)};
 \addplot[blue]
    coordinates {(1,-1)
                 (-1,1)};
\end{axis}
\draw[line width=0.7pt,densely dashed] (4.9,3.25) -- (4.9,-0.02);
\draw[line width=0.7pt,densely dashed] (1.6,3.25) -- (0,3.25);
\draw[line width=0.5pt,densely dashed,color=gray] (4.9,3.25) -- (3.1,3.25);
\node at (2.5,3.5) [rotate=45]        (Cut1) {\small $\bf\Peak\leq\Valley+1$};
\node at (3.5,2.1) [below,rotate=45]  (Cut2) {\small $\bf\Valley\leq\Peak+1$};
\node at (5.8+0.2,5.1-0.1) [below,rotate=-45] (Cut3) {\small $\bf\Valley+\Peak\leq9$};
\node at (0.86,0.96) [below,rotate=-45] (Cut4) {\scalebox{0.96}{ $\bf\Valley+\Peak\geq0$}};
\node at (-0.6,-0.9) {\large(B)};
\draw[line width=0.7pt,->,>=stealth',shorten >=1pt] (5,3.25) -- (7.7,3.25);
\end{scope}
\begin{scope}[xshift=11.1cm, yshift=7.9cm, scale = 0.45]
\draw[step=1cm,gray,very thin] (1,0) grid (11,2);
\coordinate [label=left:{\tiny$\mathbf{11}$}] (Xmax) at (11+1.33,0-0.2);
\coordinate [label=left:{\tiny{\color{black}$\mathbf{0}$}}] (Ymin) at (0.9,0+0.1);
\coordinate [label=left:{\tiny{\color{black}$\mathbf{2}$}}] (Ymax) at (0.9,2-0.1);
\node at (1.4,2.5) {$\color{red}<$};
\node at (2.5,2.5) {$\color{red}>$};
\node at (3.4,2.5) {$\color{red}<$};
\node at (4.5,2.5) {$\color{red}>$};
\node at (5.4,2.5) {$\color{red}<$};
\node at (6.5,2.5) {$\color{red}>$};
\node at (7.4,2.5) {$\color{red}<$};
\node at (8.5,2.5) {$\color{red}>$};
\node at (9.4,2.5) {$=$};
\node at (10.5,2.5) {$=$};
\draw[draw=red,line width=1pt] (1,2) -- (1,3);
\draw[draw=red,line width=1pt] (3,2) -- (3,3);
\draw[draw=red,line width=1pt] (5,2) -- (5,3);
\draw[draw=red,line width=1pt] (7,2) -- (7,3);
\draw[draw=red,line width=1pt] (9,2) -- (9,3);
\coordinate [label=left:{\footnotesize\bf{\color{black}\ding{172}}}] (peak1) at (2.0+0.5,3+0.5);
\coordinate [label=left:{\footnotesize\bf{\color{black}\ding{173}}}] (peak2) at (4.0+0.5,3+0.5);
\coordinate [label=left:{\footnotesize\bf{\color{black}\ding{174}}}] (peak3) at (6.0+0.5,3+0.5);
\coordinate [label=left:{\footnotesize\bf{\color{black}\ding{175}}}] (peak4) at (8.0+0.5,3+0.5);
\coordinate [label=left:$\footnotesize\color{black}\overbrace{\hspace*{84pt}}^{\displaystyle\text{\bf 4 peaks}}$] (P) at (8.5,4.5);
\node at (1.4,-0.5) {$<$};
\node at (2.5,-0.5) {$\color{red}>$};
\node at (3.4,-0.5) {$\color{red}<$};
\node at (4.5,-0.5) {$\color{red}>$};
\node at (5.4,-0.5) {$\color{red}<$};
\node at (6.5,-0.5) {$\color{red}>$};
\node at (7.4,-0.5) {$\color{red}<$};
\node at (8.5,-0.5) {$>$};
\node at (9.4,-0.5) {$=$};
\node at (10.5,-0.5) {$=$};
\draw[draw=red,line width=1pt] (2,0) -- (2,-1);
\draw[draw=red,line width=1pt] (4,0) -- (4,-1);
\draw[draw=red,line width=1pt] (6,0) -- (6,-1);
\draw[draw=red,line width=1pt] (8,0) -- (8,-1);
\coordinate [label=left:{\footnotesize\bf{\color{black}\ding{172}}}] (valley1) at (3.0+0.5,0-1.5);
\coordinate [label=left:{\footnotesize\bf{\color{black}\ding{173}}}] (valley2) at (5.0+0.5,0-1.5);
\coordinate [label=left:{\footnotesize\bf{\color{black}\ding{174}}}] (valley3) at (7.0+0.5,0-1.5);
\coordinate [label=left:$\footnotesize\color{black}\underbrace{\hspace*{62pt}}_{\displaystyle\text{\bf 3 valleys}}$] (V) at (7.75,-2.5);
\draw[draw=black,line width=1pt,rounded corners=1pt] (1,0) -- (2,2) -- (3,0) -- (4,2) -- (5,0) -- (6,2) -- (7,0) -- (8,2) -- (9,0) -- (10,0) -- (11,0);
\coordinate (c1)  at ( 1,0); \filldraw[fill=white,draw=black!80,line width=0.6pt] (c1)  circle (0.2);
\coordinate (c2)  at ( 2,2); \filldraw[fill=white,draw=black!80,line width=0.6pt] (c2)  circle (0.2);
\coordinate (c3)  at ( 3,0); \filldraw[fill=white,draw=black!80,line width=0.6pt] (c3)  circle (0.2);
\coordinate (c4)  at ( 4,2); \filldraw[fill=white,draw=black!80,line width=0.6pt] (c4)  circle (0.2);
\coordinate (c5)  at ( 5,0); \filldraw[fill=white,draw=black!80,line width=0.6pt] (c5)  circle (0.2);
\coordinate (c6)  at ( 6,2); \filldraw[fill=white,draw=black!80,line width=0.6pt] (c6)  circle (0.2);
\coordinate (c7)  at ( 7,0); \filldraw[fill=white,draw=black!80,line width=0.6pt] (c7)  circle (0.2);
\coordinate (c8)  at ( 8,2); \filldraw[fill=white,draw=black!80,line width=0.6pt] (c8)  circle (0.2);
\coordinate (c9)  at ( 9,0); \filldraw[fill=white,draw=black!80,line width=0.6pt] (c9)  circle (0.2);
\coordinate (c10) at (10,0); \filldraw[fill=white,draw=black!80,line width=0.6pt] (c10) circle (0.2);
\coordinate (c11) at (11,0); \filldraw[fill=white,draw=black!80,line width=0.6pt] (c11) circle (0.2);
\node[anchor=west] at (0,-4) {\bf \footnotesize example of sequence correspon-};
\node[anchor=west] at (0,-5) {\bf \footnotesize ding to the feasible pair $\mathbf{(4,3)}$:};
\node[anchor=west] at (0,-6) {\footnotesize$\mathbf{0,2,0,2,0,2,0,2,0,0,0}$};
\end{scope}
\end{tikzpicture}
}

\caption[(A)~The invariant digraph of the register automata for two time-series constraints;
         (B)~The set of feasible values of the result variables of two constraints]
{\label{fig:concrete-graph-polytope}
(A)~The invariant digraph of the register automata for the $\constraint{nb\_peak}$ and
the $\constraint{nb\_valley}$ time-series constraints;
(B)~The set of feasible values of the result variables $\Peak$ and $\Valley$ of the $\constraint{nb\_peak}$ and
the $\constraint{nb\_valley}$ time-series constraints, respectively, for sequences of length~$11$.}
\end{figure}

\begin{example}[obtaining invariants]
\label{ex:invariants}
Consider~$\constraint{nb\_peak}(X, \Peak)$ and $\constraint{nb\_valley}(X,\Valley)$
with~$X$ being a time series of length $\seqlength$ such that \final{$\seqlength \geq 2$}.
In~\exref{ex:peak-valley-relative-terms}, we found four vectors for
the relative coefficients~$\Coeff{0}$, $\Coeff{1}$, $\Coeff{2}$ of the
linear invariant~$\Coeff{}+\Coeff{0}\cdot\seqlength+\Coeff{1}\cdot\Peak+\Coeff{2}\cdot\Valley\geq 0$.
For every found vector for the relative coefficients~$(\Coeff{0},\Coeff{1},\Coeff{2})$,
we obtain a weighted digraph, whose weights now are integer numbers.
For example, for the vector $(\Coeff{0},\Coeff{1},\Coeff{2})=(0,-1,1)$, the obtained digraph
is given in Part~(A) of~\figref{fig:concrete-graph-polytope}.
We compute the length of the shortest path from the node~$s$, which corresponds to the initial state
of the register automaton in Part~(C) of~\figref{fig:peak-valley-separated} to every node corresponding
to an accepting state of the register automaton in Part~(C) of~\figref{fig:peak-valley-separated}.
The length of the shortest path from~$s$ to~$s$ is~$0$, from~$s$ to~$t$ is $0$, and from~$s$ to~$r$ is $-1$.
The minimum of these values is~$-1$, hence the constant term~$\Coeff{}$ equals~$-(0+(-1))=1$.
The obtained linear invariant is~$\Peak\leq\Valley+1$.

In a similar way, we find the constant terms for the other found vectors of the relative coefficients
$(\Coeff{0},\Coeff{1},\Coeff{2})$, and obtain three other linear invariants:
$\Valley\leq\Peak+1$,
$\Valley+\Peak\leq\seqlength-2$,
$\Valley+\Peak\geq 0$.

Part~(B) of \figref{fig:concrete-graph-polytope} shows the polytope of feasible points
$(\Peak,\Valley)$ when $\seqlength$ is $11$.
Observe that three of the four linear invariants found are facets of the convex hull of this polytope.
\qedexample
\end{example}

The next example illustrates how the method presented in this section can also be used
for generating linear invariants for non-time-series constraints.

\begin{example}[generating invariants for non-time-series constraints]
Consider a sequence of integer variables~$X=\XSeq$ with every~$X_i$ ranging over~$[0,3]$,
four $\constraint{among}$~\cite{BeldiceanuContejean94} constraints that restrict
the variables~$\FinalV_0$, $\FinalV_1$, $\FinalV_2$, $\FinalV_3$ to be
the number of occurrences of values $0,1,2,3$, respectively,
in~$X$, as well as the four corresponding $\constraint{stretch}$~\cite{Pesant:stretch}
constraints restricting the stretch length in~$X$ to be respectively
in~$[1,4]$, $[2,5]$, $[3,5]$, and~$[1,2]$.
In addition assume that value~$2$ (resp.\ $1$) cannot immediately
follow a~$3$ (resp.\ $2$). 
The intersection of the corresponding register automata has~$17$ states
and allows one to generate~$16$ linear invariants, one of them
being~$2+\seqlength+\FinalV_0+\FinalV_1-\FinalV_2-2\cdot\FinalV_3\geq 0$.
Since the sum of all~$\FinalV_i$ is $\seqlength$, this linear invariant can
be simplified to $2+2\cdot\seqlength-2\cdot\FinalV_2-3\cdot\FinalV_3\geq 0$,
which is equivalent to $2\cdot(\FinalV_2+\FinalV_3-\seqlength)\leq 2-\FinalV_3$. 
This inequality means that if~$X$ consists only of the values~$2$ and~$3$,
i.e.~$\FinalV_2+\FinalV_3-\seqlength=0$, then $\FinalV_3\leq 2$,
which represents the conjunction of the conditions that the stretch length
of~$\FinalV_3\in [1,2]$ and $(X_i=3)\Rightarrow (X_{i+1}\neq 2)$.
\qedexample
\end{example}


\subsection{Improving the Generated Linear Invariants}
\label{sec:linear-invariants-improving}
When at least one of the register automata $\Aut_1,\Aut_2,\dots,\Aut_{\NAut}$
has at least one potential register, then there may exist an accepting sequence~$X=\XSeq$
wrt~$\Intersect$ such that the weight of~$X$ in the invariant digraph~$\Digraph{\Product}{v}$
is strictly less than~$v=\LinFun$.
This may lead to weaker invariants and \exref{ex:weak-invariants} illustrates such a situation.

\begin{example}[weak invariant]
\label{ex:weak-invariants}
Given the proper plateau regular expression~$\ProperPlainPattern$,
consider a conjunction of $\constraint{nb\_proper\_plateau}(X,\Result_1)$
and $\constraint{sum\_width}$ $\constraint{\_proper\_plateau}(X,\Result_2)$
imposed on the same time series~$X$ of length~$\seqlength$, and a linear function
$v=\Coeff{}+\Coeff{0}\cdot\seqlength+\Coeff{1}\cdot\FinalV_1+\Coeff{2}\cdot\FinalV_2$.
The intersection of the register automata for these two constraints is given in
Part~(A) of \figref{fig:bad-intersections1}.
By inspection we can derive the invariant $\FinalV_2\geq 2\cdot\FinalV_1$,
which cannot be generated by the method described in Sections~\ref{sec:constructing-automaton},
\ref{sec:relative-terms} and~\ref{sec:absolute-term}, because of the following reason:
when $\Coeff{0}=0$, $\Coeff{1}=-2$, and $\Coeff{2}=1$,
the weights of the arcs from~$a$ to~$b$ and from~$b$ to~$c$ are both~$\Coeff{0}$,
and the weight of the arcs from~$c$ to~$a$ is $\Coeff{0}+\Coeff{1}+\Coeff{2}$, and
thus the weight of the cycle~$a-b-c-a$ is~$3\cdot\Coeff{0}+\Coeff{1}+\Coeff{2}=-1$.

Just before triggering the transition from $c$ to $a$, the value of the register~$D_2$
is at least~$1$ since the register automaton had triggered the transition from~$b$ to~$c$
before, which incremented~$D_2$.
Let us modify the intersection~$\Product$ so that the register~$D_2$ is not updated on
the transition from~$b$ to~$c$, and the register~$R_2$ is updated as~$R_2+D_2+2$ on
the transition from~$c$ to~$a$.
The modified register automaton~$\DelayedProduct$ recognises the same set of signatures
as~$\Product$, and after consuming any accepting sequence wrt~$\Product$,
the register automaton~$\DelayedProduct$ returns the same tuple of final
values as~$\Product$.
In addition, the weight of the cycle~$a-b-c-a$ in~$\DelayedProduct$ is equal
to~$3\cdot\Coeff{0}+\Coeff{1}+2\cdot\Coeff{2}$, which is~$0$ when~$\Coeff{0}=0$,
$\Coeff{1}=-2$, and~$\Coeff{2}=1$.
Hence, the invariant $\FinalV_2\geq 2\cdot\FinalV_1$ can be generated after
some modifications of the intersection~$\Product$.
\qedexample
\end{example}

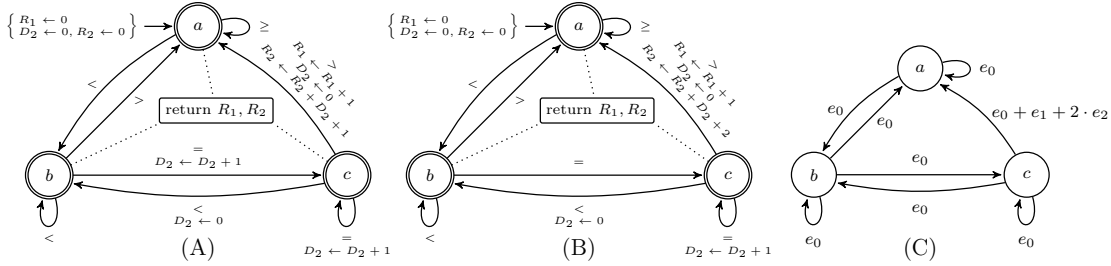
\begin{figure}[!h]
\scalebox{0.8}{
\begin{tikzpicture}[->,>=stealth',shorten >=1pt,auto,node distance=35mm,semithick,
                    information text/.style={rounded corners=1pt,inner sep=1ex}]
\begin{scope}
\node[initial,initial where=left,initial text={\tiny $\left
  \{\begin{array}{l} R_1 \leftarrow 0\\ D_2 \leftarrow 0, R_2 \leftarrow 0 \end{array}\right\}$},initial distance=5mm,state,accepting] (a) {$a$};
\node[state,accepting] (c) [below right of=a]{$c$};
\node[state,accepting] (b) [below left  of=a]{$b$};
\node[rectangle,draw,rounded corners=1pt] at (0.3,-1.4) (ctr) {\scriptsize return $\Result_1,\Result_2$};
\path
 (a)     edge [loop right]    node[right]                {\tiny$\begin{array}{c}\geq                        \end{array}$} (a)
 (a.200) edge [bend right=15] node[left]                 {\tiny$\begin{array}{c}<                           \end{array}$} (b.70)
 (b)     edge [loop below]    node                       {\tiny$\begin{array}{c}<                           \end{array}$} (b)
 (b)     edge                 node[right]                {\tiny$\begin{array}{c}>                           \end{array}$} (a)
 (b)     edge                 node[above]                {\tiny$\begin{array}{c}= \\ D_2 \leftarrow D_2 + 1 \end{array}$} (c)
 (c)     edge [loop below]    node                       {\tiny$\begin{array}{c}= \\ D_2 \leftarrow D_2 + 1 \end{array}$} (c)
 (c.200) edge [bend left=10]  node[below]                {\tiny$\begin{array}{c}< \\ D_2 \leftarrow 0 \\    \end{array}$} (b.340)
 (c)     edge [bend right=15] node[above,pos=0.4,sloped] {\tiny$\begin{array}{c}> \\ R_1 \leftarrow R_1 + 1 \\
                                                                                     D_2 \leftarrow 0 \\
                                                                                     R_2 \leftarrow R_2 + D_2 + 1 \\
                                                                                                            \end{array}$}(a.320);
\draw [dotted,-] (a) -- (ctr);
\draw [dotted,-] (b) -- (ctr.south west);
\draw [dotted,-] (c) -- (ctr.south east);
\node at (0,-3.7cm) {\large (A)};
\end{scope}
\begin{scope}[xshift=6.33cm]
\node[initial,initial where=left,initial text={\tiny $\left
  \{\begin{array}{l} R_1 \leftarrow 0\\ D_2 \leftarrow 0, R_2 \leftarrow 0 \end{array}\right\}$},initial distance=5mm,state,accepting] (a) {$a$};
\node[state,accepting] (c) [below right of=a]{$c$};
\node[state,accepting] (b) [below left  of=a]{$b$};
\node[rectangle,draw,rounded corners=1pt] at (0.3,-1.4) (ctr) {\scriptsize return $\Result_1,\Result_2$};
\path
 (a)     edge [loop right]    node[right]                {\tiny$\begin{array}{c}\geq                        \end{array}$} (a)
 (a.200) edge [bend right=15] node[left]                 {\tiny$\begin{array}{c}<                           \end{array}$} (b.70)
 (b)     edge [loop below]    node                       {\tiny$\begin{array}{c}<                           \end{array}$} (b)
 (b)     edge                 node[right]                {\tiny$\begin{array}{c}>                           \end{array}$} (a)
 (b)     edge                 node[above]                {\tiny$\begin{array}{c}=                           \end{array}$} (c)
 (c)     edge [loop below]    node                       {\tiny$\begin{array}{c}= \\ D_2 \leftarrow D_2 + 1 \end{array}$} (c)
 (c.200) edge [bend left=10]  node[below]                {\tiny$\begin{array}{c}< \\ D_2 \leftarrow 0 \\    \end{array}$} (b.340)
 (c)     edge [bend right=15] node[above,pos=0.4,sloped] {\tiny$\begin{array}{c}> \\ R_1 \leftarrow R_1 + 1 \\
                                                                                     D_2 \leftarrow 0 \\
                                                                                     R_2 \leftarrow R_2 + D_2 + 2 \\
                                                                                                            \end{array}$} (a.320);
\draw [dotted,-] (a) -- (ctr);
\draw [dotted,-] (b) -- (ctr.south west);
\draw [dotted,-] (c) -- (ctr.south east);
\node at (0,-3.7cm) {\large (B)};
\end{scope}
\begin{scope}[xshift=12cm,node distance=25mm,yshift=-0.7cm]
\node[state] (a) {$a$};
\node[state] (c) [below right of=a]{$c$};
\node[state] (b) [below left  of=a]{$b$};
\path
 (a)     edge [loop right]    node[right] {\footnotesize$\begin{array}{c} \Coeff{0}                           \end{array}$} (a)
 (a.200) edge [bend right=15] node[left]  {\footnotesize$\begin{array}{c} \Coeff{0}                           \end{array}$} (b.70)
 (b)     edge [loop below]    node        {\footnotesize$\begin{array}{c} \Coeff{0}                           \end{array}$} (b)
 (b)     edge                 node[right] {\footnotesize$\begin{array}{c} \Coeff{0}                           \end{array}$} (a)
 (b)     edge                 node[above] {\footnotesize$\begin{array}{c} \Coeff{0}                           \end{array}$} (c)
 (c)     edge [loop below]    node        {\footnotesize$\begin{array}{c} \Coeff{0}                           \end{array}$} (c)
 (c.200) edge [bend left=10]  node[below] {\footnotesize$\begin{array}{c} \Coeff{0}                           \end{array}$} (b.340)
 (c)     edge [bend right=15] node[right] {\footnotesize$\begin{array}{c} \Coeff{0}+\Coeff{1}+2\cdot\Coeff{2} \end{array}$} (a.320);
\node at (0,-3cm) {\large (C)};
\end{scope}
\end{tikzpicture}}
\caption[Intersection of the register automata for two time-series constraints, 
for which our method does not generate sharp linear invariants]{(A)~Intersection\index{intersection} of register automata for $\constraint{nb\_proper\_plateau}$
and $\constraint{sum\_width\_}$ $\constraint{proper\_plateau}$, for which the method described in Sections~\ref{sec:constructing-automaton},
\ref{sec:relative-terms} and~\ref{sec:absolute-term} does not generate facet\nobreakdash-defining invariants;
(B)~Delayed intersection obtained from the intersection in~(A);
(C)~Invariant digraph obtained from the delayed intersection in~(B).\label{fig:bad-intersections1}}
\end{figure}

To handle the issue presented in \exref{ex:weak-invariants} we introduce
a \emph{preprocessing technique} of the intersection of register automata.
The technique relies on the notion of \emph{delay} of a potential register~$\Acc$ at
a state~$q$ of the intersection~$\Product$, which is a lower bound on the value of~$\Acc$
when a sequence of triggered transitions of the register automaton ends up in state~$q$.
Intuitively, we can change the updates of some registers in a way that for any accepting
sequence wrt~$\Product$, the returned tuple of values does not change, but the arcs of
the invariant digraph obtained from the modified intersection~$\DelayedProduct$ will have
larger weights.
The modified intersection that we obtain satisfies the three following conditions:

\begin{enumerate}

\item
\label{cond:delayed-product-1}
The set of accepting sequences wrt~$\Product$ coincides with the set of
accepting sequences wrt~$\DelayedProduct$.

\item
\label{cond:delayed-product-2}
For every accepting sequence~$X$ wrt~$\Product$, the register automata~$\Product$
and~$\DelayedProduct$ return the same tuple of values.

\item
\label{cond:delayed-product-3}
For any accepting sequence~$X$, the weight of~$X$ in~$\Digraph{\DelayedProduct}{v}$
is greater than or equal to the weight of~$X$ in~$\Digraph{\Product}{v}$, where~$v$
is~$\LinFun$.
\end{enumerate}

By \condref{cond:delayed-product-3}, since for every~$X$, the weight of~$X$
in~$\Digraph{\DelayedProduct}{v}$ is greater than or equal to the weight
of~$X$ in $\Digraph{\Product}{v}$, the weight of every simple circuit
in~$X$ may also increase, which may lead to stronger invariants.
To obtain such register automaton $\DelayedProduct$, we first introduce in
\defref{def:delay} the notion of \emph{list of delays} of a state~$\State$ of
the intersection~$\Product$, denoted by~$\Delay{\State}$.
An element~$i$ of~$\Delay{\State}$ is an array whose values correspond to
the potential registers of~$\Aut_i$.
The value~$j$ of this array represents a lower bound on the value of
the register of $\Product$ corresponding the potential register~$j$
of~$\Aut_i$ when the register automaton~$\Product$ arrives to the
state~$\State$.
Further, based on this notion, in \defref{def:delayed-product},
we introduce the notion of \emph{delayed intersection}.
Finally, in \thref{th:delayed-product} we show that the delayed
intersection satisfies Conditions~\ref{cond:delayed-product-1},
\ref{cond:delayed-product-2}, and~\ref{cond:delayed-product-3}.  

\begin{definition}[list of delays of a state]
\label{def:delay}
Consider a register automaton~$\Intersect$.
The \emph{list of delays}~$\Delay{\State}$ of a state~$\State$ is a
list of arrays, where the size of the~$i$-{\rm th} array in~$\Delay{\State}$
is the number of potential registers in the register automaton~$\Aut_i$.
Let~$j$ be the index of a register of~$\Aut_i$, let~$\mathcal{T}_{\State}$
denote the set of transitions entering~$\State$, and~$\mathcal{T'}_{\State}$
denote a subset of transitions of~$\mathcal{T}_{\State}$ starting from a state
different from~$\State$, then the value~$\Delay{\State}[i][j]$ is defined as

$
\Delay{\State}[i][j] = 
\begin{cases}

0 & \exists \Tr \in \mathcal{T}_{\State},~\TrCoeff{\Tr}{i, j}{j} = 0, \\ 
 \min(\InitialV{i,j}, \min \limits_{\Tr \in \mathcal{T'}_{\State}}  \TrCoeff{\Tr}{i, j}{0})  & \textnormal{$\State$ is the initial state
  of $\Product$, and $\forall \Tr \in \mathcal{T'}_{\State},~\TrCoeff{\Tr}{i, j}{j} > 0$, } \\ 
\min \limits_{\Tr \in \mathcal{T'}_{\State}} 
\TrCoeff{\Tr}{i, j}{0} & \textnormal{~otherwise,} \\ 
\end{cases}$
\end{definition}

where~$\TrCoeff{\Tr}{i,j}{j}$ (resp.\ $\TrCoeff{\Tr}{i,j}{0}$) denotes
the coefficient of the register~$\Acc_j$ (resp.\ the free term) in the
update of~$\Acc_j$ in the automaton~$\mathcal{M}_i$.

\begin{example}[list of delays of a state]
\label{ex:delay-list}
Consider two register automata~$\Aut_1$ and~$\Aut_2$ such that their
intersection~$\Product$ is given in Part~(A) of \figref{fig:bad-intersections1}.
The register automaton~$\Aut_1$ has one register~$R_1$, and~$\Aut_2$ has two
registers~$D_2$ and~$R_2$.
Let us compute the list of delays of every state of~$\Product$.
Since only~$\Aut_1$ does not have any potential registers then for any
state~$q$ of~$\Product$, the array~$\Delay{\State}[1]$ is empty.
The following table gives the list of delays of every potential register of~$\Product$.
\begin{center}
\begin{tabular}{lccc} \toprule
state & $a$ & $b$ & $c$\\ \midrule
$\Delay{\State}$ & $[[], [0]]$ & $[[], [0]]$ & $[[], [1]]$ \\ \bottomrule
\end{tabular}
\end{center}
It implies that, when the register automaton~$\Product$ is either in state~$a$ or state~$b$,
we only know that its potential register~$D_2$ is non-negative.
However, when~$\Product$ is in the state~$c$, the value of its potential register is at least~$1$.
\qedexample 
\end{example}

\begin{definition}[delayed intersection]
\ListStyle
\label{def:delayed-product}
Consider the register automaton~$\Intersect$.
The delayed intersection~$\DelayedProduct$ of $\Aut_1,\Aut_2,\dots,\Aut_{\NAut}$
is obtained from~$\Product$ using the following rules:

\begin{itemize}
\item
The set of states and accepting states of~$\DelayedProduct$ coincide
with those of~$\Product$.
\item
The set of transitions of~$\DelayedProduct$ coincide
with the one of~$\Product$.
\item
The number of registers of~$\DelayedProduct$ is the same as for~$\DelayedProduct$,
and is denoted by~$\acc$.
\item
The initial values of main registers of~$\DelayedProduct$ are the same as for~$\DelayedProduct$.
For every potential register~$\Acc^*_{i,j}$ of~$\DelayedProduct$, its initial value equals
$\InitialV{i,j}-\Delay{\State}[i][j]$, where~$\State$ is the initial state of~$\DelayedProduct$
and~$\InitialV{i,j}$ is the initial value of~$\Acc_{i,j}$ of~$\Product$.
\item
For every transition~$\Tr$ from a state~$\State_1$ to a state~$\State_2$ and for any register~$\Aut_{i,j}$
of~$\Product$, the update of~$A_{i,j}$ on~$\Tr$ is equal to
$\TrCoeff{\Tr}{i,j}{0}+\sum\limits_{k=1}^{\acc}\TrCoeff{\Tr}{i,j}{k}\cdot A_{i,k}$,
while the update of the corresponding register~$\Aut^*_{i,j}$ on the corresponding transition
of~$\DelayedProduct$ is equal to $\DTrCoeff{\Tr}{i,j}{0}+\sum\limits_{k=1}^{\acc}\TrCoeff{\Tr}{i,j}{k}\cdot A^*_{i,k}$,
where~$\DTrCoeff{\Tr}{i,j}{0}$ is defined as follows:
  \begin{itemize}
  \item
  If~$\Acc_{i,j}$ is a main register of~$\Product$, then
  $\DTrCoeff{\Tr}{i,j}{0}=\TrCoeff{\Tr}{i,j}{0}+\sum \limits_{k = 1}^{\acc_i - 1}\TrCoeff{\Tr}{i,j}{k}\cdot\Delay{\State_1}[i][k]$,
  where~$\acc_i$ is the number of registers of the register automaton~$\Aut_i$.
  \item
  If~$\Acc_{i,j}$ is a potential register of~$\Product$, then
  $\DTrCoeff{\Tr}{i,j}{0}=\TrCoeff{\Tr}{i,j}{0}+\Delay{\State_1}[i][j]-\Delay{\State_2}[i][j]$.
  \end{itemize}

\item
The acceptance function of~$\DelayedProduct$ is the same as for~$\Product$.
\end{itemize}
\end{definition}

\begin{example}[delayed intersection]
Consider two register automata~$\Aut_1$ and~$\Aut_2$
such that their intersection~$\Product$ is
given in Part~(A) of \figref{fig:bad-intersections1}.
The delayed intersection~$\DelayedProduct$ constructed according to
\defref{def:delayed-product} is given in Part~(B) of \figref{fig:bad-intersections1}.
The main difference between~$\DelayedProduct$ and~$\Product$ is that the register~$D_2$
is no longer updated on the transition from~$b$ to~$c$, but its contribution is integrated directly
to~$R_2$ on the transition from state~$c$ to state~$a$.
\qedexample
\end{example}

\begin{theorem}[properties of delayed intersection]
\label{th:delayed-product}
Consider the register automaton $\Intersect$ and the corresponding delayed intersection $\DelayedProduct$.
The three following conditions are satisfied:

\begin{enumerate}
\item
\label{cond:theorem-delayed-product-1}
The set of accepting sequence wrt~$\Product$ coincides with the set of
accepting sequence wrt~$\DelayedProduct$.
\item
\label{cond:theorem-delayed-product-2}
For every accepting sequence~$X$ wrt~$\Product$, the register automata~$\Product$
and~$\DelayedProduct$ return the same tuple of values.
\item
\label{cond:theorem-delayed-product-3}
For any accepting sequence~$X$, the weight of~$X$ in~$\Digraph{\DelayedProduct}{v}$
is greater than or equal to the weight of~$X$ in~$\Digraph{\Product}{v}$, where~$v$ is~$\LinFun$.
\end{enumerate}
\end{theorem}

\begin{proof}
We prove each of the three statements separately.

\vspace{0.2cm}
\noindent{\bf [Proof of (\ref{cond:theorem-delayed-product-1})].}
Since~$\Product$ have the same sets of states, transitions and
accepting states, and every~$\Aut_i$ has the \IncrementalAut~property,
then the sets of accepting sequences of~$\Product$ and~$\DelayedProduct$
are the same.

\vspace{0.2cm}
\noindent{\bf [Proof of (\ref{cond:theorem-delayed-product-2})].}  
Since the acceptance function of both~$\Product$ and~$\DelayedProduct$
returns a tuple of main registers, we will show that after
consuming the signature~$S$ of any accepting sequence,
the main registers of~$\Product$ and~$\DelayedProduct$ contain the same values.
Let us prove this statement by induction on the length of~$S$.

\vspace{0.2cm}
\noindent{\bf Base case.}
Let us consider a sequence~$S=\Tuple{S_1}$ consumed by~$\DelayedProduct$.
The register automaton~$\DelayedProduct$ triggered one transition~$\Tr$ from
its initial state~$\State$ to some other state~$\State'$.
Then, let us consider a main register~$\Acc^*_{i,\acc_i}$. 
By definition, its value equals
$\TrCoeff{\Tr}{i,j}{0}+\Acc^*_{i,\acc_i,\acc_i}+\sum\limits_{k=1}^{\acc_i-1}\TrCoeff{\Tr}{i,j}{k}\cdot(\Acc^*_{i,k}+\Delay{\State}[i][k])$.
Since any potential register~$\Acc^*_{i,k}$ has not been updated,
its contains the initial value, which equals $\InitialV{i,j}-\Delay{\State}[i][k]$.
Furthermore, the value of~$\Acc^*_{i,\acc_i}$ after one transition is
equal to $\TrCoeff{\Tr}{i,j}{0}+ 
 \InitialV{i,\acc_i}+\sum \limits_{k=1}^{\acc_i-1}\TrCoeff{\Tr}{i,j}{k}\cdot(\InitialV{i,j}-\Delay{\State}[i][k]+\Delay{\State}[i][k]) =
 \TrCoeff{\Tr}{i,j}{0}+\InitialV{i,\acc_i}+\sum\limits_{k=1}^{\acc_i-1}\TrCoeff{\Tr}{i,j}{k}\cdot\InitialV{i,j}$,
which coincides with the value of the corresponding register~$\Acc_{i,j}$ of~$\Product$.

\vspace{0.2cm}
\noindent{\bf Induction step.}
Assume that after having consumed a sequence $S=\Tuple{S_1,S_2,\dots,S_{m-1}}$,
the main registers of~$\DelayedProduct$ contain the same values as the main register
of~$\Product$ after having consumed the same sequence.
Let us show that after consuming one another symbol~$S_m$, which triggers a transition~$\Tr$,
the main registers of~$\DelayedProduct$ and~$\Product$ will have the same value.
The update of~$\Acc^*_{i,\acc_i}$ on~$\Tr$ is equal to
$\TrCoeff{\Tr}{i,j}{0}+\Acc^*_{i,\acc_i}+\sum\limits_{k=1}^{\acc_i-1}\TrCoeff{\Tr}{i,j}{k}\cdot(\Acc^*_{i,k}+\Delay{\State}[i][k])$.
By assumption of induction the value of~$\Acc^*_{i,\acc_i}$ in~$\Product$
and~$\Acc_{i,\acc_i}$ in~$\DelayedProduct$ are the same after consuming~$S$.
Hence, we only need to show after having consumed~$S$, that the value of
the potential register~$\Acc_{i,k}$ of~$\Product$ equals~$\Acc^*_{i,k}+\Delay{\State}[i][k]$. 
This can also be shown by induction, starting from a state that is a destination of
a triggered transition~$\Tr'$ such that~$\TrCoeff{\Tr'}{i,k}{k}=0$.

\vspace{0.2cm}
\noindent{\bf [Proof of (\ref{cond:theorem-delayed-product-3})].}  
We now prove the last statement.  
Let us consider the invariant digraphs~$\Digraph{\DelayedProduct}{v}$
and~$\Digraph{\Product}{v}$, where~$v=\LinFun$.
We now show that for every accepting sequence~$X=\XSeq$ wrt~$\Product$,
its weight in~$\Digraph{\DelayedProduct}{v}$ is greater than or equal to
its weight in~$\Digraph{\Product}{v}$.
The weight of~$X$ in~$\Digraph{\Product}{v}$ is the constant
$\Coeff{}+\Coeff{0}\cdot(p-1)+\sum\limits_{i=1}^{\NAut}\Coeff{i}\cdot\beta_i^0$
(see~\defref{def:sequence-walk}) plus the weight of the walk of~$X$, which is in total
$\Coeff{}+\Coeff{0}\cdot(p-1)+\sum\limits_{i=1}^{\NAut}\Coeff{i}\cdot\beta_i^0+\Coeff{0}\cdot(\seqlength-p+1)+\sum\limits_{j=1}^{\seqlength-p+1}\sum\limits_{i=1}^{\NAut}\Coeff{i}\cdot\beta^{\Tr_j}_i=\Coeff{}+\Coeff{0}\cdot\seqlength+\sum\limits_{i=1}^{\NAut}\Coeff{i}\cdot\left(\beta_i^0+\sum\limits_{j=1}^{\seqlength-p+1}\beta^{\Tr_j}_i\right)$,
where~$p$ is the arity of the considered signature,
and~$\Tr_1,\Tr_2,\dots\Tr_{\seqlength-p+1}$ is the sequence of transitions
of~$\Product$ triggered upon consuming the signature of~$X$.
Similarly, the weight of~$X$ in~$\Digraph{\DelayedProduct}{v}$ is equal to
$\Coeff{}+\Coeff{0}\cdot\seqlength+\sum\limits_{i=1}^{\NAut}\Coeff{i}\cdot\left(\delta_i^0+\sum\limits_{j=1}^{\seqlength-p+1}\delta^{\Tr_j}_i\right)$,
where~$\delta_i^0$ is the initialisation weight in~$\DelayedProduct$,
and every~$\delta^{\Tr_j}_i$ is the weight of an arc~$\Tr_j$ in~$\Digraph{\DelayedProduct}{v}$.

We now show that the value~$\Coeff{i}\cdot\left(\beta_i^0+\sum\limits_{j=1}^{\seqlength-p+1}\beta^{\Tr_j}_i\right)$
is not greater than~$\Coeff{i}\cdot\left(\delta_i^0+\sum\limits_{j=1}^{\seqlength-p+1}\delta^{\Tr_j}_i\right)$.
This will imply that the weight of the walk of~$X$ in~$\Digraph{\Product}{v}$
is less than or equal to the weight of the walk of~$X$ in~$\Digraph{\DelayedProduct}{v}$.

By \defref{def:invariant-graph}, the weight of every arc of~$\Digraph{\Product}{v}$
(resp.\ $\Digraph{\DelayedProduct}{v}$), corresponding to a transition~$\Tr$ of~$\Product$,
(resp.\ $\DelayedProduct$) is equal to
$\sum\limits_{i=1}^{\NAut}\Coeff{i}\cdot\beta_i^{\Tr}$
(resp.\ $\sum\limits_{i=1}^{\NAut}\Coeff{i}\cdot\delta_i^{\Tr}$).

As in \thref{th:linear-invariants-main}, we consider the
function~$v_i=\Coeff{i}\cdot\FinalV_i$.
Depending on the sign of~$\Coeff{i}$ we have two cases:

\vspace{0.25cm}
\noindent{\bf Case~(1): $\Coeff{i} \geq 0$.}~~
Then, the weight of~$X$ in~$\Digraph{\Product}{v_i}$
(resp.\ $\Digraph{\DelayedProduct}{v_i}$) is equal
to~$\Coeff{i}\cdot\alpha$ (resp.\ $\Coeff{i}\cdot\gamma$),
where~$\alpha$ denotes
$\beta_i^0+\sum\limits_{j=1}^{\seqlength-p+1}\beta^{\Tr_j}_i=
\sum\limits_{k=1}^{\acc_i}\InitialV{i,k}+\sum\limits_{\ell=1}^{\seqlength-p+1}\TrCoeff{\Tr_{\ell}}{i,\acc_i}{0}$
(resp.\ $\gamma$ denotes
$\delta_i^0+\sum\limits_{j=1}^{\seqlength-p+1}\delta^{\Tr_j}_i=
\sum\limits_{k=1}^{\acc_i}\DInitialV{i,k}+\sum\limits_{\ell=1}^{\seqlength-p+1}\DTrCoeff{\Tr_{\ell}}{i,\acc_i}{0}$).
Since every
$\DTrCoeff{\Tr_{\ell}}{i,\acc_i}{0}=
\TrCoeff{\Tr_{\ell}}{i,\acc_i}{0}+\sum\limits_{k=1}^{\acc_i-1}\Delay{\State}[i][k]$,
it implies that
$\DTrCoeff{\Tr_{\ell}}{i,\acc_i}{0}\geq\TrCoeff{\Tr_{\ell}}{i,\acc_i}{0}$.
Then,~$\alpha\leq\gamma$, and when~$\Coeff{i}>0$,
we have $\Coeff{i}\cdot\gamma\geq\Coeff{i}\cdot\alpha$.

\vspace{0.25cm}
\noindent{\bf Case~(2): $\Coeff{i} < 0$.}~~
Then, the weight of~$X$ in~$\Digraph{\Product}{v_i}$
(resp.\ $\Digraph{\DelayedProduct}{v_i}$) is equal
to~$\Coeff{i}\cdot\alpha$ (resp.\ $\Coeff{i}\cdot\gamma$),
where~$\alpha$ denotes
$\beta_i^0+\sum\limits_{j=1}^{\seqlength-p+1}\beta^{\Tr_j}_i=
\sum\limits_{k=1}^{\acc_i}\InitialV{i,k}+\sum\limits_{\ell=1}^{\seqlength-p+1}\sum\limits_{k=1}^{\acc_i}\TrCoeff{\Tr_{\ell}}{i,k}{0}$
(resp.\ $\gamma$ denotes
$\delta_i^0+\sum\limits_{j=1}^{\seqlength-p+1}\delta^{\Tr_j}_i=
\sum\limits_{k=1}^{\acc_i}\DInitialV{i,k}+\sum\limits_{\ell=1}^{\seqlength-p+1}\sum\limits_{k=1}^{\acc_i}\DTrCoeff{\Tr_{\ell}}{i,k}{0}$).
Further, by construction of~$\DelayedProduct$, every~$\DTrCoeff{\Tr_{\ell}}{i,k}{0}$ (with $i\in[1,\acc_i]$)
is equal to $\TrCoeff{\Tr_{\ell}}{i,k}{0}+\Delay{\State_1}[i][k]-\Delay{\State_2}[i][k]$,
where~$\State_1$ and~$\State_2$ are the source and the destination of the transition~$\Tr_{\ell}$, respectively.
In addition,~$\DTrCoeff{\Tr_{\ell}}{i,\acc_i}{0}=\TrCoeff{\Tr_{\ell}}{i,\acc_i}{0}$. 
By replacing every $\DTrCoeff{\Tr_{\ell}}{i,k}{0}$ with its expression, and simplifying the sum, we obtain 
$\sum\limits_{k=1}^{\acc_i}\InitialV{i,k}+\sum\limits_{{\ell} = 1}^{\seqlength-p+1}\sum\limits_{k=1}^{\acc_i}(\TrCoeff{\Tr_{\ell}}{i,k}{0}-\Delay{\State'}[i][k])$, where~$\State'$ is the last state visited by~$\Product$ upon consuming~$X$.
Since every~$\Delay{\State'}[i][k]$ is non-negative,
$\TrCoeff{\Tr_{\ell}}{i,k}{0}-\Delay{\State'}[i][k]\leq\TrCoeff{\Tr_{\ell}}{i,k}{0}$.
This implies that~$\gamma\leq\alpha$, and when~$\Coeff{i}<0$,
$\Coeff{i}\cdot\gamma\geq\Coeff{i}\cdot\alpha$. \hspace{\fill} \qed
\end{proof}

Note that in the register automaton~$\DelayedProduct$, all the constants~$\DTrCoeff{\Tr}{i,j}{0}$
introduced in \defref{def:delayed-product} are non\nobreakdash-negative by definition of the delay (see~\defref{def:delay}).
It means that the reasoning used in the proof of \thref{th:linear-invariants-main} requiring the non-negativity
of these constants remains valid for the invariant digraph~$\Digraph{\DelayedProduct}{v}$. 

\begin{example}[generating stronger invariants]
\label{ex:delayed-product-better}
Consider two register automata~$\Aut_1$ and~$\Aut_2$ such that their intersection~$\Product$,
and their delayed intersection~$\DelayedProduct$ are respectively given in Parts~(A) and~(B)
of \figref{fig:bad-intersections1}.
The invariant digraph~$\Digraph{\DelayedProduct}{v}$ is given in
Part~(C) of \figref{fig:bad-intersections1} when~$\Coeff{0}>0$,
$\Coeff{1}>0$, and~$\Coeff{2}<0$.
By stating the minimisation problem from \secref{sec:relative-terms},
we obtain the following coefficients: $\Coeff{0}=0$, $\Coeff{1}=-2$,
and~$\Coeff{2}=1$.
The constant~$\Coeff{}$ is found to be~$0$, and we obtain the
invariant $2\cdot\FinalV_1\geq\FinalV_2$, which could not be found
with the invariant digraph~$\Digraph{\Product}{v}$. 
\qedexample
\end{example}


\subsection[Generating Conditional Linear Invariants with  the Non-Default Value Condition]{Generating Conditional Linear Invariants with the Non-Default Value Condition}

Quite often a register automaton~$\Aut_i$ (with $i\in[1,\NAut]$)
returns the initial value of one of its registers only when the signature
of~$X$ does not contain any occurrence of some regular expression~$\pattern_i$.
This may lead to a convex hull of points of coordinates
$(\FinalV_1,\FinalV_2,\dots,\FinalV_{\NAut})$ returned by~$\Product$ containing
infeasible points, e.g.~see Part~(A) of \figref{fig:polytope-linear-invariants}.
Some of these infeasible points can be eliminated by stronger invariants subject
to a condition, called the \emph{non-default value} condition, that no variable
of the returned vector is assigned to the initial value of the corresponding
register.
We first illustrate the motivation for such conditional linear invariants.

\begin{example}[motivation for conditional invariants]
Consider the $\constraint{nb\_decreasing\_terrace}(X,\FinalV_1)$ and the
$\constraint{sum\_width\_}$ $\constraint{increasing\_terrace}(X,\FinalV_2)$ constraints,
where~$X$ is a time series of length $\seqlength$, $\FinalV_1$ is restricted to be the
number of maximal occurrences of $\DecreasingTerracePatternName=\DecreasingTerracePattern$
in the signature of~$X$, and~$\FinalV_2$ is restricted to be the sum of the number of
elements in subseries of~$X$ whose signatures correspond to words of the language
of~$\IncreasingTerracePatternName=\IncreasingTerracePattern$.
In \figref{fig:polytope-linear-invariants}, for $\seqlength=12$, the squared points
represent feasible pairs $(\FinalV_1,\FinalV_2)$, while the circled points stand for
infeasible pairs $(\FinalV_1,\FinalV_2)$ inside the convex hull.
The linear invariant $2\cdot\FinalV_1+\FinalV_2\leq\seqlength-2$ is a facet of the polytope,
which does not eliminate the points~$(1,8)$, $(2,6)$, $(3,4)$, $(4,2)$.
However, if we assume that both~$\FinalV_1>0$ and~$\FinalV_2>0$,
then we can add a linear invariant eliminating these four infeasible points, 
namely $2\cdot\FinalV_1+\FinalV_2\leq\seqlength-3$, shown in Part~(B)
of \figref{fig:polytope-linear-invariants}.
In addition, the infeasible points on the straight line~$\FinalV_2=1$ will also be eliminated
by the restriction $\FinalV_2=0\lor\FinalV_2\geq 2$ given in \cite[p.~2962]{Catalog18}.
\qedexample

\end{example}

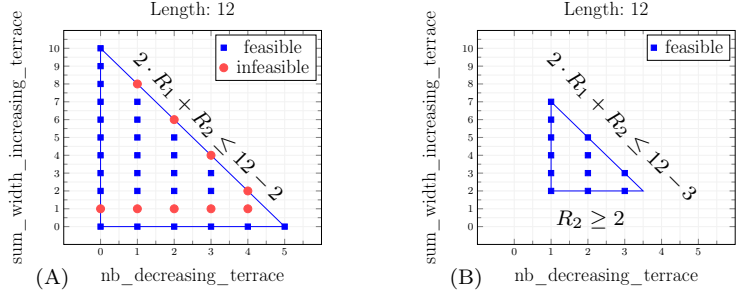
\begin{figure}[!h]
\begin{center}
\newcommand{\InitialText}{\scriptsize $\scriptstyle\left \{\begin{array}{l}\scriptstyle\Peak \leftarrow 0\\ \scriptstyle\Valley \leftarrow 0 \end{array}\right\}$}
\begin{tikzpicture}
\begin{scope}[scale=0.5,font=\footnotesize]
\begin{axis}
[xlabel={\Large nb\_decreasing\_terrace},
    ylabel={\Large sum\_width\_increasing\_terrace},
    title={\Large Length: 12},
    minor tick num=1,
    grid=both,grid style={line width=.1pt, draw=gray!10},
    major grid style={line width=.1pt,draw=gray!10},
    xmin=-1,xmax=6,
    ymin=-1,ymax=11,
    xtick={0,1,2,3,4,5},
    ytick={0,1,2,3,4,5,6,7,8,9,10},
    scatter/classes={
      f={mark=square*,blue},%
      i={mark=*,red!70,mark size=3pt}}] %
\addplot[
    scatter,
    only marks,
    point meta=explicit symbolic
]
table[meta=label] {
x y label
0 0 f
1 0 f
0 2 f
0 3 f
0 4 f
1 2 f
0 5 f
1 3 f
2 0 f
0 6 f
1 4 f
0 7 f
2 2 f
1 5 f
0 8 f
1 6 f
2 3 f
1 7 f
0 9 f
0 10 f
2 4 f
3 0 f
2 5 f
3 2 f
4 0 f
3 3 f
5 0 f
0 1 i
1 1 i
1 8 i
2 1 i
2 6 i
3 1 i
3 4 i
4 1 i
4 2 i
};
\addplot[blue] coordinates {
(0,0)
(5,0)
(0,10)
(0,0)};
\legend{\Large~feasible,\Large~infeasible}
\end{axis}
\node at (3.5,2.8) [rotate=-45,above] (Cut1) {$2\cdot\FinalV_1+\FinalV_2\leq12-2$};
\node at (-0.3,-0.9) {\small (A)};
\end{scope}
\begin{scope}[xshift=5.5cm,scale=0.5,font=\footnotesize]
\begin{axis}
[xlabel={\Large nb\_decreasing\_terrace},
    ylabel={\Large sum\_width\_increasing\_terrace},
    title={\Large Length: 12},
    minor tick num=1,
    grid=both,grid style={line width=.1pt, draw=gray!10},
    major grid style={line width=.1pt,draw=gray!10},
    xmin=-1,xmax=6,
    ymin=-1,ymax=11,
    xtick={0,1,2,3,4,5},
    ytick={0,1,2,3,4,5,6,7,8,9,10},
    scatter/classes={
      f={mark=square*,blue},%
      i={mark=*,red!70}}] %
\addplot[
    scatter,
    only marks,
    point meta=explicit symbolic
]
table[meta=label] {
x y label
1 2 f
1 3 f
1 4 f
2 2 f
1 5 f
1 6 f
2 3 f
1 7 f
2 4 f
2 5 f
3 2 f
3 3 f
};
\addplot[blue] coordinates {
(1,2)
(3.5,2)
(1,7)
(1,2)};
\legend{\Large~feasible}
\end{axis}
\node at (3.5,2.8) [rotate=-45,above] (Cut1) {$2\cdot\FinalV_1+\FinalV_2\leq12-3$};
\node at (3,0.2) [above] (Cut2) {$\FinalV_2\geq2$};
\node at (-0.3,-0.9) {\small (B)};
\end{scope}
\end{tikzpicture}
\end{center}

\caption{\label{fig:polytope-linear-invariants} Invariants on the result values~$\FinalV_1$
and~$\FinalV_2$ of $\constraint{nb\_decreasing\_terrace}$ and $\constraint{sum\_width\_increasing\_terrace}$
for a sequence length of~$12$ (A)~with the general linear invariants, and (B)~with the Non-Default Value condition.}
\end{figure}

Consider that each register automaton~$\Aut_i$ (with~$i\in[1,\NAut]$) returns its initial value
after consuming the signature of an accepting sequence~$X$ wrt~$\Aut_i$ iff the signature of~$X$
does not contain any occurrence of some regular expression~$\pattern_i$ over the alphabet~$\Sigma$.
Let~$\Aut'_i$ denote the register automaton which accepts the words of the language $\Sigma^*\sigma_i\Sigma^*$,
where~$\Sigma^*$ denotes any word over~$\Sigma$.
Then, using the method described in Sections~\ref{sec:constructing-automaton}, \ref{sec:relative-terms}
and~\ref{sec:absolute-term} we generate the linear invariants for $\Aut'_1\cap\Aut'_2\cap\dots\cap\Aut'_{\NAut}$.
These linear invariants hold when the non-default value condition is satisfied.


\subsection{Facet Analysis of Linear Invariants}
\label{sec:sharpness}
Consider two time-series constraints~$\CtrGamma_1(X,\Result_1)$ and~$\CtrGamma_2(X,\Result_2)$
imposed on the same sequence~$X$ of length~$\seqlength$.
After having generated linear and conditional linear invariants linking~$\Result_1$, $\Result_2$ and~$\seqlength$,
an essential question is whether these invariants are facets of the convex hull of feasible combinations~$\Result_1$
and~$\Result_2$, or not.
Given a linear invariant
$\Function=\Coeff{}+\Coeff{0}\cdot\seqlength+\Coeff{1}\cdot\Result_1+\Coeff{2}\cdot\Result_2\geq 0$,
this section presents a three-step method for answering this question:

\begin{enumerate}
\item
Assume an infinite set~$\NValues$ of values of~$\seqlength$ such that the set of sequences whose length
is in~$\NValues$ can be represented by a constant-size automaton, e.g.~$\seqlength\geq 5$,
$\seqlength\bmod 2=1$, $\seqlength\in\Naturals$.
\item
Find two distinct points~$P_1$ and~$P_2$, possibly parameterised by~$\seqlength\in\NValues$,
laying on the straight line $\Function=0$.
\item
Prove that~$P_1$ and~$P_2$ are feasible for any~$\seqlength\in\NValues$.
\end{enumerate}

The challenge here is the third step, which requires to prove the feasibility of~$P_1$ and~$P_2$
for an infinite set of values of~$\seqlength$.
Let~$\Upp{\Result_i}(\seqlength)$ denote the maximum value of~$\Result_i$ among all time series
of length~$\seqlength$, let~$a_x,a_y$ be in~$\Curly{0,1}$ and let~$b_x$ and~$b_y$ be natural numbers.
It turns out that for points of the form
$\left(\begin{array}{c}h_x,\\h_y\hspace*{2.7pt}
 \end{array}\right)=
 \left(\begin{array}{c}a_x\cdot\Upp{\Result_1}(\seqlength)+(1-2\cdot a_x)\cdot b_x,\\
                       a_y\cdot\Upp{\Result_2}(\seqlength)+(1-2\cdot a_y)\cdot b_y\hspace*{2.7pt}
 \end{array}\right)$
we can represent the set of time series corresponding to such a point as the intersection of three
constant-size automata, namely
(i)~the automaton representing the assumed condition on~$\seqlength$,
(ii)~the automaton that accepts only and only all time series yielding~$h_x$ as the value of~$\Result_1$, and
(iii)~the automaton that accepts only and only all time series yielding~$h_y$ as the value of~$\Result_2$.
The constant-size automata representing a condition on~$\Result_1$ and~$\Result_2$ can be synthesised
from the seed transducers for the regular expressions associated with~$\CtrGamma_1$ and~$\CtrGamma_2$,
as shown in \secref{sec:conditional-automata}.
We now give in Sections~\ref{sec:assume-condition}, \ref{sec:find-two-points} and~\ref{sec:prove-feasibility}
more details for each of the three steps.

\subsubsection{Step One: Assuming a Condition on the Sequence Length} 
\label{sec:assume-condition}

Some of the invariants we generate are facets of the convex hull only for a subset of values of~$\seqlength$,
e.g.~only even-length sequences.
This requires to assume a condition on $\seqlength$ that can be represented by a constant\nobreakdash-size automaton.
We start with the less restrictive condition and try to prove that an invariant is a facet,
and then gradually restrict the condition if we cannot prove it in full generality.

\subsubsection{Step Two: Finding Two Integer Points on a Straight Line} 
\label{sec:find-two-points}

To find two distinct points on the straight line~$\Function=0$, we assume a value of~$\Result_1$
as $a_x\cdot\Upp{\Result_1}(\seqlength)+(1-2\cdot a_x)\cdot b_x$, which by~\cite{BoundsConstraints} is equal
to $a_x\cdot\frac{\seqlength-c_1-(\seqlength-c_1)\bmod d_1}{d_1}+(1-2\cdot a_x)\cdot b_x$,
with~$c_1$ and~$d_1$ being integer constants depending on the regular expression associated with~$\CtrGamma_1$.
If the coefficient of~$\Result_2$ in~$\Function$ is~$0$, then the value of~$\Result_2$ is not relevant and we can take,
for example, $0$ or~$1$ as the value of~$\Result_2$.
Otherwise, by isolating~$\Result_2$ from the equation~$\Function=0$ we obtain:

\begin{equation}
\label{eq:y-expression}
\small
\Result_2 = \frac{(-\Coeff{0} \cdot d_1 - \Coeff{1} \cdot a_x) \cdot \seqlength
+ (-\Coeff{} \cdot d_1 + \Coeff{1} \cdot a_x \cdot c_1 - \Coeff{1}
\cdot (1-2\cdot a_x) \cdot b_x \cdot d_1) + \Coeff{1} \cdot a_x \cdot (\seqlength - c_1)
\bmod d_1}{d_1 \cdot \Coeff{2}}
\end{equation}

Then we verify that the right-hand side of~(\ref{eq:y-expression}) is of the form
$a_y\cdot\frac{\seqlength-c_2-(\seqlength-c_2)\bmod d_2}{d_2}+(1-2\cdot a_y)\cdot b_y$,
with~$c_2$ and~$d_2$ being integer constants depending on the regular expression associated
with~$\CtrGamma_2$, with~$a_y$ being in~$\Curly{0,1}$, and with~$b_y$ being a natural number.
This is done by solving a system of constraints assuming that~$\seqlength$ belongs to~$\NValues$.
The solutions of such system are the candidate points of the next step.

\subsubsection{Step Three: Proving Feasibility of an Integer Point}
\label{sec:prove-feasibility}

Once we found two distinct integer points laying on the straight line~$\Function=0$,
we show that both points are feasible for any~$\seqlength$ in~$\NValues$.

For a point of coordinates~$(h_x,h_y)$ we construct two constant-size automata~$\Aut_1$
and~$\Aut_2$, where~$\Aut_1$ (resp.\ $\Aut_2$) is an automaton recognising the signatures
of all and only time series yielding~$h_x$ (resp.\ $h_y$) as the value of~$\Result_1$
(resp.\ $\Result_2$).
Let~$\Aut_{\seqlength}$ be a constant-size automaton representing the $\seqlength\in\NValues$
condition, and~$d$ denote the smallest difference between two values in~$\NValues$.
If, in the intersection~$\Aut$ of~$\Aut_1$, $\Aut_2$, \dots, $\Aut_{\seqlength}$ there are cycles
of length~$d$, then the point~$(h_x,h_y)$ is feasible for any sequence whose length is in~$\NValues$.
From this intersection we also compute the smallest value of~$\seqlength$,
for which these two points are feasible.
This is the length of the shortest path from the initial state of~$\Aut$ to an accepting state of~$\Aut$
that goes through a state belonging to a cycle of length~$d$.

If we cannot prove the feasibility of our two current points,
then we try a different combination of~$a_x$ and~$b_x$, and obtain two other distinct points.
Since the set of values of~$b_x$ is, potentially, unbounded we limit ourselves only
to the values of~$b_x$ belonging to the set~$\Curly{0,1,2,3}$.

\begin{example}
Consider the conjunction of the $\constraint{nb\_peak}(X,P)$ and the $\constraint{nb\_valley}(X,V)$
time-series constraints imposed on the same time series~$X=\XSeq$,
and the linear invariant $P+V\leq\seqlength-2$.
Let us now analyse whether this invariant is facet defining or not.
By~\cite{BoundsConstraints}, both $\Upp{P}(\seqlength)$ and $\Upp{V}(\seqlength)$ are equal to
$\frac{\seqlength-1-(\seqlength-1)\bmod 2}{2}$.
\begin{itemize}
\item
When~$P$ is equal to~$\Upp{P}(\seqlength)$, then by~(\ref{eq:y-expression}),
$V$ is equal to $\frac{\seqlength-3+(\seqlength-1)\bmod 2}{2}$;
we consider two cases:
\begin{enumerate}[i.]
\item
If $(\seqlength-1)\bmod 2=0$, then $\frac{\seqlength-3+(\seqlength-1)\bmod 2}{2} =
                                    \frac{(\seqlength-1)-2}{2} = \Upp{V}(\seqlength)-1$.
\item
If $(\seqlength-1)\bmod 2=1$, then $\frac{\seqlength-3+(\seqlength-1)\bmod 2}{2} =
                                    \frac{(\seqlength-2)-2}{2} = \Upp{V}(\seqlength)-1$.
\end{enumerate}
In both cases, we obtain the candidate point $P_1=(\Upp{P}(\seqlength),\Upp{V}(\seqlength)-1)$.
\item
When~$P$ is equal to $\Upp{P}(\seqlength)-1$, then by~(\ref{eq:y-expression}),
     $V$ is $\frac{\seqlength-1+(\seqlength-1)\bmod 2}{2}$;
we consider two cases:
\begin{enumerate}[i.]
\item
If $(\seqlength-1)\bmod 2=0$, then $\frac{\seqlength-1+(\seqlength-1)\bmod 2}{2} =
                                    \frac{\seqlength-1}{2} = \Upp{V}(\seqlength)$
and we obtain the candidate point $P_2=(\Upp{P}(\seqlength)-1,\Upp{V}(\seqlength))$.
\item
If $(\seqlength-1)\bmod 2=1$, then $\frac{\seqlength-1+(\seqlength-1)\bmod 2}{2} =
                                    \frac{(\seqlength-2)+2}{2} = \Upp{V}(\seqlength)+1$
and we obtain the candidate $(\Upp{P}(\seqlength)-1,\Upp{V}(\seqlength)+1)$.
This candidate is not feasible since its second coordinate is strictly greater than
the maximum value of the second coordinate of any feasible point.
\end{enumerate}
\end{itemize}

Hence, for the case~$(\seqlength-1)\bmod 2=0$, we obtain two distinct candidate
points~$P_1$ and~$P_2$ located on the straight line $P+V=\seqlength-2$.
To prove that $P_2=(\Upp{P}(\seqlength)-1,\Upp{V}(\seqlength))$ is feasible,
we construct and intersect the automata for the
$\Result_1=\Upp{P}(\seqlength)$, 
$\Result_2=\Upp{V}(\seqlength)-1$, and
$(\seqlength-1)\bmod 2=0$ conditions,
and observe that the intersection has a cycle of length~$2$,
which implied the feasibility of~$P_2$ for any odd sequence size.
The same procedure is used for proving the feasibility of~$P_1$ for any odd sequence size.

Since both~$P_1$ and~$P_2$ lay on the straight line $\Result_1+\Result_2=\seqlength-2$,
and are feasible for any odd length, then the straight line $\Result_1+\Result_2=\seqlength-2$
is a facet of the convex hull of feasible points, when $\seqlength$ is odd.
\qedexample
\end{example}



\section{Synthesising Parameterised Non-Linear Invariants}
\label{sec:non-linear-invariants}

 The contribution of this section is a methodology for two families of time-series constraints,
namely the~$\constraint{nb}\_\pattern$ and the~$\constraint{sum\_width}\_\pattern$ families,
which both proposes conjectures and proves them automatically by using \emph{constant-size automata},
i.e.~automata whose number of states, and whose input alphabet size are independent both from
an input time-series length and from the values in an input time series.
For a conjunction of two time-series constraints~$\CtrGamma_1(\X,\Result_1)$ and~$\CtrGamma_2(\X,\Result_2)$
imposed on the same time series~$\X=\XSeq$, our method describes sets of infeasible result-value pairs
for~$(\Result_1,\Result_2)$.
We assume that every time-series constraint mentioned in this section belongs either to
the~$\constraint{nb}\_\pattern$ or to the~$\constraint{sum\_width}\_\pattern$ family.
Each set of infeasible pairs is described by a formula $\BoolFun_i(\Result_1,\Result_2,\seqlength)$
expressed as a conjunction
$\Predicate_i^1~\land~\Predicate_i^2~\land~\dots~\land~\Predicate_i^{k_i}$ of elementary
conditions~$\Predicate_i^j$ between~$\Result_1$, $\Result_2$ and~$\seqlength$.
The learned Boolean function $\BoolFun_1\lor\BoolFun_2\lor\dots\lor\BoolFun_m$ represents
the union of sets of infeasible pairs~$(\Result_1,\Result_2)$, while its negation
$\neg\BoolFun_1\land\neg\BoolFun_2\land\dots\land \neg \BoolFun_m$ corresponds to
an implied constraint, which is a \emph{universally true Boolean formula}, namely

\begin{equation}
\label{eq:introduction}
\forall\hspace*{1pt}\X,~\CtrGamma_1(\X,\Result_1)\land\CtrGamma_2(\X,\Result_2)\Rightarrow\bigwedge\limits_{i=1}^m\neg\BoolFun_i(\Result_1,\Result_2,\seqlength)
\end{equation}

In order to prove that~(\ref{eq:introduction}) is universally true we need to show that for every
$\BoolFun_i(\Result_1,\Result_2,\seqlength)$, there does not exist a time series of length~$\seqlength$
yielding~$\Result_1$ (resp.\ $\Result_2$) as the result value of~$\CtrGamma_1$ (resp.\ $\CtrGamma_2$)
and satisfying $\BoolFun_i(\Result_1,\Result_2,\seqlength)$.
The key idea of our proof scheme is to represent the infinite set of time series satisfying each
elementary condition~$\Predicate_i^j$ of $\BoolFun_i(\Result_1,\Result_2,\seqlength)$
as a constant-size automaton~$\Aut_{i,j}$.
Then checking that the intersection of all automata $\Aut_{i,1},\Aut_{i,2},\dots,\Aut_{i,k_i}$
is empty implies that $\BoolFun_i(\Result_1,\Result_2,\seqlength)$ is indeed infeasible.
Note that such proof scheme is independent of the time-series length~$\seqlength$;
moreover, it does not explore any search space.

As for the linear invariants, the generation process of non-linear invariants is offline:
it is done once and for all to build a reusable database of generic invariants.
This section is organised as follows:

\begin{itemize}
\item
\secref{sec:example-mining} motivates this work with a running example, which illustrates
the need for deriving non-linear invariants.
\item
\secref{sec:overview-mining} presents our method for deriving non-linear invariants for a
conjunction of time-series constraints.
It starts with an overview of the three phases of our method, and then details each phase:
  \begin{enumerate}
  \item
  A \emph{generating data phase} is detailed in the introduction of \secref{sec:overview-mining}.
  Its goal is to generate a dataset, from which we will extract non-linear invariants. 
  \item
  A \emph{mining phase} is detailed in \secref{sec:mining}.
  It extracts, from the data generated in the mining phase, a hypothesis~$H$ consisting of Boolean
  functions of the form $f_1\lor f_2\lor\dots\lor f_m$.
  \item
  A \emph{proof phase} is detailed in \secref{sec:proof}.
  For every Boolean function $f_i$ (with~$i\in[1,m]$) in the extracted hypothesis~$H$,
  the proof phase either proves its validity for \emph{every} time-series length, 
  or refute it by generating a counter example. The counter example is used to modify
  the current hypothesis and the process is repeated.
  \end{enumerate}
Note that our generated data is noise-free, and that our goal is not to discover statistical properties
of time-series constraints, but rather to extract non-linear invariants, which are always true.
\end{itemize}


 \subsection{Motivation and Running Example}
\label{sec:example-mining}

Consider a conjunction of time-series constraints $\CtrGamma_1(X,\Result_1)\land\CtrGamma_2(X,\Result_2)$
imposed on the same time series~$X$.
In~\secref{sec:linear-invariants}, using the representation of~$\gamma_1$ and~$\gamma_2$ as
\emph{register automata}, we presented a method for deriving parameterised linear invariants
linking the values of~$\Result_1$, $\Result_2$.
Although, in most cases the derived inequalities were proven to be facet-defining,
we observe that in some cases, even when using these invariants, the solver could
still take a lot of time to find a feasible solution or to prove infeasibility.
This happens because of some infeasible combinations of values of the result variables
that were located \emph{inside} the convex hull of all feasible combinations.
The following example illustrates such a situation.

\begin{example}[running example]
\label{ex:running}
Consider the conjunction of $\constraint{sum\_width\_de\-crea\-sing\_sequen\-ce}($ $X,\Result_1)$
and $\constraint{sum\_width\_zigzag}(X,\Result_2)$ time-series constraints imposed on the same
time series~$X$ of length~$\seqlength$, where a decreasing sequence and a zigzag respectively
correspond to~$\DecreasingSequencePattern$ and\\ $\ZigzagPattern$.
For the values of~$\seqlength$ in the interval~$[9,12]$, \figref{fig:polytope_with_infeasible_values}
represents feasible pairs of~$(\Result_1,\Result_2)$ as blue squares, and infeasible pairs lying inside
the convex hull of feasible (blue) points as red circles.
The convex hull contains a significant number of infeasible (red) points,
which we want to characterise automatically.
\qedexample
\end{example}
\begin{figure}
\begin{center}
\begin{tikzpicture}[scale=0.4]
\begin{scope}[xshift=0cm,yshift=-20cm]
\begin{axis}[xlabel={\Large \bf sum\_width\_decreasing\_sequence},
    ylabel={\Large \bf sum\_width\_zigzag},
    title={\Large \bf  Length: 9},
    minor tick num=1,
    grid=both,grid style={line width=.1pt, draw=gray!10},
    major grid style={line width=.1pt,draw=gray!10},
    xmin=-1,xmax=10,
    ymin=-1,ymax=10,
    scatter/classes={
    c1={mark=square*,blue!100},
    c2={mark=square*,blue!100},
    c3={mark=square*,blue!100},
    c4={mark=square*,blue!100},
    c5={mark=square*,blue!100},
    c6={mark=square*,blue!100},
    c7={mark=square*,blue!100},
    c8={mark=square*,blue!100},
    c9={mark=square*,blue!100},
    c10={mark=square*,green!30},
    c11={mark=square*,green!20},
    c12={mark=square*,green!10}}]

    \addplot[scatter,only marks,
        scatter src=explicit symbolic]
        coordinates {
(0, 0) [c1]
(2, 0) [c1]
(4, 2) [c1]
(4, 0) [c1]
(4, 3) [c1]
(6, 4) [c1]
(6, 2) [c1]
(6, 3) [c1]
(6, 5) [c1]
(2, 2) [c1]
(4, 4) [c1]
(6, 6) [c1]
(8, 7) [c1]
(6, 0) [c1]
(8, 6) [c1]
(8, 4) [c1]
(8, 5) [c1]
(3, 0) [c2]
(5, 2) [c2]
(5, 0) [c2]
(5, 3) [c2]
(7, 5) [c2]
(7, 3) [c2]
(7, 0) [c2]
(7, 2) [c2]
(7, 4) [c2]
(8, 3) [c2]
(7, 6) [c2]
(8, 0) [c2]
(8, 2) [c2]
(9, 6) [c2]
(9, 4) [c2]
(9, 2) [c2]
(9, 0) [c2]
            };
    \addplot[red]
        coordinates {
(0,0)
(9,0)
(9,6)
(8,7)
(6,6)
(0,0)
            };
    \addplot[scatter,only marks,mark=*,red]
        coordinates {
(1,0)
(1,1)
(2,1)
(3,1)
(3,2)
(3,3)
(4,1)
(5,1)
(5,4)
(5,5)
(6,1)
(7,1)
(8,1)
(9,1)
(9,3)
(9,5)
            };
\end{axis}
\end{scope}
\begin{scope}[xshift=9cm,yshift=-20cm]
\begin{axis}[xlabel={\Large \bf  sum\_width\_decreasing\_sequence},
    ylabel={\Large \bf sum\_width\_zigzag},
    title={\Large \bf  Length: 10},
    minor tick num=1,
    grid=both,grid style={line width=.1pt, draw=gray!10},
    major grid style={line width=.1pt,draw=gray!10},
    xmin=-1,xmax=11,
    ymin=-1,ymax=11,
    scatter/classes={
    c1={mark=square*,blue!100},
    c2={mark=square*,blue!100},
    c3={mark=square*,blue!100},
    c4={mark=square*,blue!100},
    c5={mark=square*,blue!100},
    c6={mark=square*,blue!100},
    c7={mark=square*,blue!100},
    c8={mark=square*,blue!100},
    c9={mark=square*,blue!100},
    c10={mark=square*,green!30},
    c11={mark=square*,green!20},
    c12={mark=square*,green!10}}]

    \addplot[scatter,only marks,
        scatter src=explicit symbolic]
        coordinates {
(0, 0) [c1]
(2, 0) [c1]
(4, 2) [c1]
(4, 0) [c1]
(4, 3) [c1]
(6, 4) [c1]
(6, 2) [c1]
(6, 3) [c1]
(6, 5) [c1]
(6, 0) [c1]
(8, 6) [c1]
(2, 2) [c1]
(4, 4) [c1]
(6, 6) [c1]
(8, 8) [c1]
(8, 7) [c1]
(8, 5) [c1]
(8, 4) [c1]
(8, 2) [c1]
(8, 3) [c1]
(10, 8) [c1]
(3, 0) [c2]
(5, 2) [c2]
(5, 0) [c2]
(5, 3) [c2]
(7, 4) [c2]
(7, 2) [c2]
(7, 0) [c2]
(7, 5) [c2]
(7, 3) [c2]
(7, 6) [c2]
(8, 0) [c2]
(9, 7) [c2]
(9, 5) [c2]
(9, 3) [c2]
(9, 0) [c2]
(9, 2) [c2]
(9, 4) [c2]
(9, 6) [c2]
(10, 6) [c2]
(10, 4) [c2]
(10, 2) [c2]
(10, 0) [c2]
            };
    \addplot[red]
        coordinates {
(0,0)
(10,0)
(10,8)
(8,8)
(0,0)
            };
    \addplot[scatter,only marks,mark=*,red]
        coordinates {
(1,0)
(1,1)
(2,1)
(3,1)
(3,2)
(3,3)
(4,1)
(5,1)
(5,4)
(5,5)
(6,1)
(7,1)
(7,7)
(8,1)
(9,1)
(9,8)
(10,1)
(10,3)
(10,5)
(10,7)
            };
\end{axis}
\end{scope}
\begin{scope}[xshift=18cm,yshift=-20cm]
\begin{axis}[xlabel={\Large \bf  sum\_width\_decreasing\_sequence},
    ylabel={\Large \bf  sum\_width\_zigzag},
    title={\Large \bf  Length: 11},
    minor tick num=1,
    grid=both,grid style={line width=.1pt, draw=gray!10},
    major grid style={line width=.1pt,draw=gray!10},
    xmin=-1,xmax=12,
    ymin=-1,ymax=12,
    scatter/classes={
    c1={mark=square*,blue!100},
    c2={mark=square*,blue!100},
    c3={mark=square*,blue!100},
    c4={mark=square*,blue!100},
    c5={mark=square*,blue!100},
    c6={mark=square*,blue!100},
    c7={mark=square*,blue!100},
    c8={mark=square*,blue!100},
    c9={mark=square*,blue!100},
    c10={mark=square*,green!30},
    c11={mark=square*,green!20},
    c12={mark=square*,green!10}}]

    \addplot[scatter,only marks,
        scatter src=explicit symbolic]
        coordinates {
(0, 0) [c1]
(2, 0) [c1]
(4, 2) [c1]
(4, 0) [c1]
(4, 3) [c1]
(6, 4) [c1]
(6, 2) [c1]
(6, 3) [c1]
(6, 5) [c1]
(6, 0) [c1]
(8, 6) [c1]
(8, 4) [c1]
(8, 5) [c1]
(8, 7) [c1]
(2, 2) [c1]
(4, 4) [c1]
(6, 6) [c1]
(8, 8) [c1]
(10, 9) [c1]
(8, 3) [c1]
(8, 2) [c1]
(10, 8) [c1]
(8, 0) [c1]
(10, 6) [c1]
(10, 7) [c1]
(3, 0) [c2]
(5, 2) [c2]
(5, 0) [c2]
(5, 3) [c2]
(7, 4) [c2]
(7, 2) [c2]
(7, 0) [c2]
(7, 3) [c2]
(7, 5) [c2]
(9, 7) [c2]
(9, 5) [c2]
(9, 3) [c2]
(7, 6) [c2]
(9, 0) [c2]
(9, 2) [c2]
(9, 4) [c2]
(9, 6) [c2]
(10, 5) [c2]
(9, 8) [c2]
(10, 3) [c2]
(10, 0) [c2]
(10, 2) [c2]
(10, 4) [c2]
(11, 8) [c2]
(11, 6) [c2]
(11, 4) [c2]
(11, 2) [c2]
(11, 0) [c2]
            };
    \addplot[red]
        coordinates {
(0,0)
(11,0)
(11,8)
(10,9)
(8,8)
(0,0)
            };
    \addplot[scatter,only marks,mark=*,red]
        coordinates {
(1,0)
(1,1)
(2,1)
(3,1)
(3,2)
(3,3)
(4,1)
(5,1)
(5,4)
(5,5)
(6,1)
(7,1)
(7,7)
(8,1)
(9,1)
(10,1)
(11,1)
(11,3)
(11,5)
(11,7)
            };
\end{axis}
\end{scope}
\begin{scope}[xshift=27cm,yshift=-20cm]
\begin{axis}[xlabel={\Large \bf  sum\_width\_decreasing\_sequence},
    ylabel={ \Large \bf sum\_width\_zigzag},
    title={ \Large \bf Length: 12},
    minor tick num=1,
    grid=both,grid style={line width=.1pt, draw=gray!10},
    major grid style={line width=.1pt,draw=gray!10},
    xmin=-1,xmax=13,
    ymin=-1,ymax=13,
    scatter/classes={
    c1={mark=square*,blue!100},
    c2={mark=square*,blue!100},
    c3={mark=square*,blue!100},
    c4={mark=square*,blue!100},
    c5={mark=square*,blue!100},
    c6={mark=square*,blue!100},
    c7={mark=square*,blue!100},
    c8={mark=square*,blue!100},
    c9={mark=square*,blue!100},
    c10={mark=square*,green!30},
    c11={mark=square*,green!20},
    c12={mark=square*,green!10}}]

    \addplot[scatter,only marks,
        scatter src=explicit symbolic]
        coordinates {
(0, 0) [c1]
(2, 0) [c1]
(4, 2) [c1]
(4, 0) [c1]
(4, 3) [c1]
(6, 4) [c1]
(6, 2) [c1]
(6, 3) [c1]
(6, 5) [c1]
(6, 0) [c1]
(8, 6) [c1]
(8, 4) [c1]
(8, 5) [c1]
(8, 7) [c1]
(8, 2) [c1]
(8, 3) [c1]
(10, 8) [c1]
(2, 2) [c1]
(4, 4) [c1]
(6, 6) [c1]
(8, 8) [c1]
(10, 10) [c1]
(10, 9) [c1]
(10, 7) [c1]
(8, 0) [c1]
(10, 6) [c1]
(10, 4) [c1]
(10, 5) [c1]
(12, 10) [c1]
(3, 0) [c2]
(5, 2) [c2]
(5, 0) [c2]
(5, 3) [c2]
(7, 4) [c2]
(7, 2) [c2]
(7, 0) [c2]
(7, 3) [c2]
(7, 5) [c2]
(9, 6) [c2]
(9, 4) [c2]
(9, 2) [c2]
(9, 0) [c2]
(9, 7) [c2]
(9, 5) [c2]
(9, 8) [c2]
(9, 3) [c2]
(10, 3) [c2]
(7, 6) [c2]
(10, 0) [c2]
(10, 2) [c2]
(11, 9) [c2]
(11, 7) [c2]
(11, 5) [c2]
(11, 3) [c2]
(11, 0) [c2]
(11, 2) [c2]
(11, 4) [c2]
(11, 6) [c2]
(11, 8) [c2]
(12, 8) [c2]
(12, 6) [c2]
(12, 4) [c2]
(12, 2) [c2]
(12, 0) [c2]
            };
    \addplot[red]
        coordinates {
(0,0)
(12,0)
(12,10)
(10,10)
(0,0)
            };
    \addplot[scatter,only marks,mark=*,red]
        coordinates {
(1,0)
(1,1)
(2,1)
(3,1)
(3,2)
(3,3)
(4,1)
(5,1)
(5,4)
(5,5)
(6,1)
(7,1)
(7,7)
(8,1)
(9,1)
(9,9)
(10,1)
(11,1)
(11,10)
(12,1)
(12,3)
(12,5)
(12,7)
(12,9)
};
\end{axis}
\end{scope}
\end{tikzpicture}
\end{center}
\caption[Feasible and infeasible combinations of the result values of
two time-series constraints imposed on the same sequence whose length
is in $\Curly{9,10,11,12}$]{\label{fig:polytope_with_infeasible_values} 
Feasible points, shown as blue squares,
for the result variables $\Result_1,\Result_2$
of the conjunction of
$\constraint{sum\_width\_decreasing\_sequence}(X,\Result_1)$ and
$\constraint{sum\_width\_zigzag}(X,\Result_2)$ on the same time series $X = \XSeq$
for the values of $\seqlength$ in $\Curly{9,10,11,12}$;
red circles represent infeasible points inside the convex hull of feasible points,
while red straight lines depict the facets of the convex hull of feasible points.
}
\end{figure}


Next section develops a systematic approach for generating non-linear invariants characterising infeasible
combinations of~$\Result_1$ and~$\Result_2$ located within the convex hull of feasible combinations.


 \subsection{Discovering and Proving Invariants}
\label{sec:overview-mining}

Consider a conjunction of time-series constraints~$\gamma_1(X,\Result_1)$ and~$\gamma_2(X,\Result_2)$
imposed on the same time series $X$.
This work focuses on \emph{automatically extracting and proving} invariants that characterise
some subsets of infeasible combinations of~$\Result_1$ and~$\Result_2$ that are all located
inside the convex hull of feasible combinations of~$\Result_1$ and~$\Result_2$.
Our approach uses three sequential phases.

\noindent\hspace*{2pt}\textbullet\hspace*{2pt}\textsc{[generating data phase]}
The first phase is a preparatory work, namely \emph{generating data}.
For each time-series length~$\seqlength$ in~$[7,12]$, we generate \emph{all} feasible combinations
of the values of~$\Result_1$ and~$\Result_2$.
For each of the~$6$ lengths,
(\emph{i})~we compute the convex hull of feasible points of~$\Result_1$ and~$\Result_2$ using
Graham's scan~\cite{Graham72}, and
(\emph{ii})~we detect the set~$\InfeasibleSet$ of infeasible combinations of~$\Result_1$ and~$\Result_2$
in this convex hull.
  
\noindent\hspace*{2pt}\textbullet\hspace*{2pt}\textsc{[mining phase]}
The second phase, called the \emph{mining phase}, consists of extracting a hypothesis~$H$ describing
the set~$\InfeasibleSet$ of infeasible combinations of~$\Result_1$ and~$\Result_2$ from the generated data.
We represent this hypothesis as a disjunction of Boolean functions $f_i(\Result_1,\Result_2,\seqlength)$.

\noindent\hspace*{2pt}\textbullet\hspace*{2pt}\textsc{[proof phase]}
The third phase, called the \emph{proof phase}, consists in refining the discovered hypothesis~$\Hyp$
by validating some Boolean functions~$f_i$ and by refuting and eliminating others using
\emph{constant-size automata}.
A refined hypothesis, which is \emph{proved} to be correct in the general case,
i.e.~for \emph{any time-series length}, is called a \emph{description} of the set~$\InfeasibleSet$.


 \subsubsection{Mining Phase}
\label{sec:mining}
Consider a conjunction of two time-series constraints~$\gamma_1(X,\Result_1)$ and~$\gamma_2(X,\Result_2)$,
imposed on the same time series~$X$.
This section shows how to extract a hypothesis in the form of a \emph{disjunction of Boolean functions},
describing the infeasible combinations of values of~$\Result_1$ and~$\Result_2$ that are located within
the convex hull of feasible combinations.

There exist a number of works on learning a disjunction of predicates~\cite{bshouty17a}, and some special case,
where disjunction corresponds to a geometric concept~\cite{BshoutyGeomConcepts,CHEN199970}. 
Usually, the learner interacts with an oracle through various types of queries or with the user by receiving
positive and negative examples; the learner tries to minimise the number of such interactions to speed up convergence. 

In our case, the input data consists of the set of positive, called \emph{infeasible}, and negative, called \emph{feasible},
examples, which is finite and which is completely produced by our generating phase.
This allows exploring all possible inputs without any interaction.

We now present the components of our mining phase:
\begin{itemize}
\item
First, we describe our dataset, which consists of \emph{feasible} and \emph{infeasible} pairs
of the result values~$\Result_1,\Result_2$.
\item
Second, we define the space of concepts, \emph{hypotheses}, we can potentially extract from our dataset. 
\item
Third, we outline the \emph{target hypothesis} for time-series constraints, i.e.~what we are searching for.
\item
Finally, we briefly describe the algorithm used for finding the target hypothesis.
\end{itemize}

\paragraph{Input Dataset}
\label{sec:input-dataset}
We represent our generated data as the union of two sets of triples~$\DataPos$ (resp.\ $\DataNeg$) called
the set of feasible (resp.\ infeasible) examples, such that:
\begin{itemize}
\item
For every~$(k,p_1,p_2)$ (with $k\in[7,12]$) in~$\DataPos$, there exists at least one time series of length~$k$
that yields~$p_1$ and~$p_2$ as the values of~$\Result_1$ and~$\Result_2$, respectively.
\item
For every~$(k,p_1,p_2)$ (with $k\in[7,12]$) in~$\DataNeg$,
\begin{enumerate}
\item
there does not exist any time series of length~$k$ that would yield~$p_1$ and~$p_2$ as the values
of~$\Result_1$ and~$\Result_2$, respectively.
\item
$(p_1,p_2)$ is located within the convex hull of feasible combinations of~$\Result_1$ and~$\Result_2$.
\end{enumerate}
\end{itemize}

\paragraph{Space of Hypotheses}
\label{sec:hypothesis-space}
Every element of our hypothesis space is a disjunction of Boolean functions from a finite predefined set~$\HypSet$.
Each element of~$\HypSet$ is a conjunction $C_1\land C_2\land\dots\land C_{\NbPredicates}$ with every~$C_i$
being a predicate, called an \emph{atomic relation}, where the main atomic relations are:
\bigbreak

\begin{enumerate}[(i)]
\begin{minipage}{0.25\linewidth}
\item \hspace*{3pt}$\seqlength\geq c$\hspace*{0.5pt},
\item \hspace*{3pt}$\seqlength\bmod c=d$\hspace*{0.5pt},
\end{minipage}
\begin{minipage}{0.25\linewidth}
\item \hspace*{3pt}$\Result_j\bmod c=d$\hspace*{0.5pt},
\item \hspace*{3pt}$\Result_j\geq d$\hspace*{0.5pt},
\end{minipage}
\begin{minipage}{0.25\linewidth}
\item \hspace*{3pt}$\Result_j\leq d $\hspace*{0.5pt},
\item \hspace*{3pt}$\Result_j=c$\hspace*{0.5pt},
\end{minipage}
\begin{minipage}{0.25\linewidth}
\item \hspace*{3pt}$\Result_j=\Upp{\Result_j}(\seqlength)-c$\hspace*{0.5pt},
\item \hspace*{3pt}$\Result_j=c\cdot\Result_k+d$\hspace*{0.5pt},
\end{minipage}
\end{enumerate}
with~$c$ and~$d$ being natural numbers, and~$\Upp{\Result_k}(\seqlength)$ being the maximum possible value
of~$\Result_k$ given the constraint~$\gamma_k(\XSeq,\Result_k)$.
The intuition of these atomic relations is now explained:
\bigbreak
\begin{itemize} 
\item (i)~stands from the fact that many invariants are only valid for long enough time series.
\item (ii)~is motivated by the fact that the parity of the length of a time series is sometimes relevant.
\item (iii)~is justified by the fact that the parity of~$\Result_1$ or~$\Result_2$ can come into play.
\item (iv)~and~(v) are related to the fact that infeasible combinations of~$\Result_1$ and~$\Result_2$ can be located on a ray or an interval.
\item (vi)~and~(vii) are respectively linked to the fact that quite often infeasible combinations of~$\Result_1$ and~$\Result_2$ within
      the convex hull are very close to the minimum or the maximum values~\cite{BoundsConstraints} of~$\Result_k$ (with $k\in[1,2]$),
      i.e.~$c$ is a very small constant, typically~$0$ or~$1$.
\item (viii)~denotes the fact that some invariants correspond to a linear combination of~$\Result_1$ and~$\Result_2$.
\end{itemize}

\paragraph{Target Hypothesis}
\label{sec:target-hypothesis}
\begin{definition}[Boolean function consistent wrt a dataset]
A Boolean function of~$\HypSet$ is \emph{consistent} wrt a dataset~$\DataSet$ iff it is true for at least one infeasible example of~$\DataSet$,
and false for every feasible example of~$\DataSet$.
\end{definition}
For example, $\Result_1=\Result_2~\land~\Result_1\bmod 2=1$ is consistent with the dataset of \figref{fig:polytope_with_infeasible_values},
but the two Boolean functions~$\Result_1=13$ and~$\Result_1=\Result_2$ are not.
\begin{definition}[universally true Boolean function]
A Boolean function of~$\HypSet$ is \emph{universally true} if it is true for any time series of any length.
\end{definition}

\begin{definition}[target hypothesis]
The \emph{target hypothesis}~$H$ is the disjunction of all Boolean functions of~$\HypSet$ consistent with~$\DataSet$.
\end{definition}
Note that in the target hypothesis some Boolean functions can be subsumed by other Boolean functions.
We cannot do the subsumption analysis at this point since we do not yet know which Boolean functions are true or not.

\paragraph{Mining Algorithm}
\label{sec:mining-algorithm}
Our mining algorithm filters out all the Boolean functions not consistent with our dataset and
returns the disjunction of the remaining Boolean functions.
Note that the mining algorithm ignores Boolean functions involving the atomic relation
(i)~$\seqlength>c$, which is handled in the proof phase.
Remember that we run the algorithm only on the limited dataset $\DataSet_{[7,12]}$,
i.e.~the dataset generated from time series of length in~$[7,12]$. This is because sizes
that are too small lead to degenerate polytopes, while sizes that are too large are too expensive in terms of computation.


 \subsubsection{Proof Phase}
\label{sec:proof}
After extracting from~$\DataSet_{[7,12]}$ the target hypothesis
$\Hyp=\BoolFun_1\lor\BoolFun_2\lor\dots\lor\BoolFun_{\NbFunctions}$
characterising subsets of infeasible combinations of~$\Result_1$ and~$\Result_2$
that are all located within the convex hull of feasible combinations
of~$\Result_1$ and~$\Result_2$, we refine this hypothesis,
by keeping only universally true Boolean functions~$\BoolFun_i$.

Before presenting our proof technique, we look at the structure of the hypothesis~$\Hyp$.
Every Boolean function~$\BoolFun$ in~$\Hyp$ is of the form
$\BoolFun=\Predicate_1\land\Predicate_2\land\dots\land\Predicate_{\NbPredicates}$
and can be classified into one of the two following categories:

\begin{itemize}
\item
\textsc{Independent Boolean Function} means that every~$\Predicate_i$ is
an \emph{independent atomic relation}, i.e.~depends either on~$\Result_1$ or~$\Result_2$, but not on both.
For instance, $\Result_1=\Upp{\Result_1}(\seqlength)\land\Result_2\bmod 2=1$ is an independent Boolean function.

\bigbreak
\item
\textsc{Dependent Boolean Function} means that there exists at least one~$\Predicate_i$ that is
a \emph{dependent atomic relation}, i.e.~mentions both~$\Result_1$ and~$\Result_2$.
For instance, $\Result_1\bmod 2=1\land\Result_1=\Result_2+1$ is a dependent Boolean function.
\end{itemize}

\bigbreak
The proof of an invariant depends on its category.
We now show how to prove that an independent (resp.\ dependent) Boolean function is universally true.
\bigbreak

\paragraph{Proof of Independent Boolean Functions}
\label{sec:proof-independent-functions}

Since most atomic relations are independent, i.e.~cases~(i) to~(vii), we first focus on a necessary and sufficient
condition for proving that an independent Boolean function is universally true.
Such necessary and sufficient condition is given in the main result of this section,
namely \thref{th:non-linear-invariants-main}, provided that there exists constant-size automata
associated with the atomic relations in~$\BoolFun$.

\begin{definition}[set of supporting signatures for an atomic relation]
\label{def:atomic-supporting-signatures}
For an atomic relation~$\Predicate$, the \emph{set of supporting signatures}~$\HypTimeSeries{\Predicate}$
is the set of words in~$\Sigma^*$ such that, for every word in~$\HypTimeSeries{\Predicate}$ there exists
a time series satisfying~$\Predicate$, whose signature is this word.
\end{definition}
\begin{definition}[set of supporting signatures for a Boolean function]
\label{def:sup_sign_bool_func}
For an independent Boolean function $\BoolFun=\Predicate_1\land\Predicate_2\land\dots\land\Predicate_{\NbPredicates}$,
we define the \emph{set of supporting signatures}~$\HypTimeSeries{\BoolFun}$ as
$\bigcap\limits_{i=1}^{\NbPredicates}\HypTimeSeries{C_i}$.
\end{definition}

A Boolean function~$\BoolFun$ is universally true iff it describes infeasible combinations of~$\Result_1$ and~$\Result_2$
for any time-series length, and thus the set~$\HypTimeSeries{\BoolFun}$ is empty.

For any atomic relation~$\Predicate$ from~(i) to~(vii), i.e.~an independent atomic relation,
the corresponding set of supporting signatures is represented as the language of a constant-size automaton~$\Aut_{\Predicate}$.
Constant size means that the number of states of this automaton does not depend on the length of the input time series.
For a Boolean function $\BoolFun=\Predicate_1\land\Predicate_2\land\dots\land\Predicate_{\NbPredicates}$,
$\HypTimeSeries{\BoolFun}$ is simply the set of signatures recognised by the automaton obtained after intersecting
all~$\Aut_{\Predicate_i}$ (with $i\in[1,\NbPredicates]$).
This provides a necessary and sufficient condition for proving that a Boolean function~$\BoolFun$ is universally true.

\begin{theorem}[necessary and sufficient condition for an independent Boolean function to be universally true]
\label{th:non-linear-invariants-main}
Consider two time-series constraints~$\gamma_1(X,\Result_1)$ and~$\gamma_2(X,\Result_2)$ on the same time series~$X$,
and a Boolean function $\BoolFun(\Result_1,\Result_2,\seqlength)=\Predicate_1\land\Predicate_2\land\dots\land\Predicate_{\NbPredicates}$
such that, for every~$\Predicate_i$ there exists a constant-size automaton~$\Aut_{\Predicate_i}$.
The function~$\BoolFun$ is universally true iff the intersection of all automata
for~$\Aut_{\Predicate_i}$ (with $i\in[1,\NbPredicates]$) is empty.
\end{theorem}
The proof of \thref{th:non-linear-invariants-main} follows from
Definitions~\ref{def:atomic-supporting-signatures} and~\ref{def:sup_sign_bool_func}.

\vspace{5pt}
For some Boolean function $\BoolFun=C_1\land C_2\land\dots\land C_{\NbPredicates}$,
the set $\HypTimeSeries{\BoolFun}=\bigcap\limits_{i=1}^{\NbPredicates}\HypTimeSeries{C_i}$ may not be empty, but finite.
In this case, we compute the length~$c$ of the longest signature in~$\HypTimeSeries{\BoolFun}$,
and obtain a new Boolean function $\BoolFun'=\BoolFun\land\seqlength\geq c+1$.
By construction, the set~$\HypTimeSeries{\BoolFun'}$ is empty, thus~$\BoolFun'$ is universally true.

\bigbreak

\secref{sec:conditional-automata} will further show how to generate automata for independent atomic relations.
Every such automaton is called a \emph{conditional automaton}.

\paragraph{Proof of Dependent Boolean Functions}
\label{sec:proof-dependent-functions}

Some dependent Boolean functions, i.e.~case~(viii), can be handled by adapting the technique for generating linear
invariants described in~\secref{sec:linear-invariants}.

Consider two time-series constraints~$\gamma_1(\X,\Result_1)$ and~$\gamma_2(\X,\Result_2)$ on the same time series~$\X$.
We present here a method for verifying that the dependent Boolean function
$\Result_1-\Const\cdot\Result_2=1$, with~$\Const$ being either~$1$ or~$2$, is universally true.
Note that such Boolean function was extracted during the mining phase for~$17$ pairs of time-series constraints.

We prove by contradiction that the corresponding Boolean function is universally true.
Our proof consists of the following steps:
\begin{enumerate}

\item 
{\bf Assumption.}~~Assume that there exists a time series~$X$ such that $\Result_1-\Const\cdot\Result_2=1$.

\item 
{\bf Implication for the parity of~$\Result_1$ and $\Const\cdot\Result_2$.}~~When
$\Result_1-\Const\cdot\Result_2=1$, then~$\Result_1$ and~$\Const\cdot\Result_2$ have different parity.

\item 
{\bf Obtaining  a contradiction.}~~Since~$\Result_1 $ and $\Const\cdot\Result_2$ must have different parity,
there exists a value of~$b$ that is either~$0$ or~$1$ such that the conjunction
$\Result_1-\Const\cdot\Result_2= 1~\land~\Result_1\bmod 2=b~\land~\Const\cdot\Result_2\bmod 2=1-b$ holds.
In order to prove that $\Result_1-\Const\cdot\Result_2=1$ is infeasible, for either value of parameter~$b$,
we need to show that, either the obtained conjunction is infeasible, e.g.~when $\Const=2$ and~$b$ is~$0$,
or the method of \secref{sec:linear-invariants} produces a linear invariant
$\Result_1-\Const\cdot\Result_2\geq c$, with~$c$ being strictly greater than~$1$.
\end{enumerate}

If at this third step of our proof method the considered conjunction is feasible, and the desired invariant
$\Result_1-\Const\cdot\Result_2\geq c$~was not obtained, then we cannot draw any conclusion about the infeasibility
of $\Result_1-\Const\cdot\Result_2=1$.

In practice, for the $17$ pairs of time-series constraints, for which we extracted the Boolean function
$\Result_1-\Const\cdot\Result_2=1$, the method of~\secref{sec:linear-invariants} did indeed generate
the desired linear invariant, which proved that the considered Boolean function is universally true.

\input{figures/groups.tex}

\begin{example}[mining, proving and filtering non-linear invariants for the running example]
Consider the conjunction of the $\constraint{sum\_width\_decreasing\_sequence}(\X,\Result_1)$ and
the $\constraint{sum\_width\_zigzag}(\X,\Result_2)$ time-series constraints on the same time series~$\X$,
introduced in \exref{ex:running}.
For this conjunction, we now describe the result of the mining and the proving phases of our method,
as well as the dominance filtering, i.e.~discarding Boolean functions subsumed by some other Boolean function.
\begin{itemize}
\item
During the mining phase we extracted a disjunction of~$156$ Boolean functions.
Most Boolean functions, even if they are true, are redundant.
For example, the Boolean function $\Result_1=1~\land~\Result_2=1$ is
subsumed by~$\Result_1=1$, and thus can be discarded.
However, at this point we cannot do the dominance filtering since we do not yet know
which Boolean functions are universally true.
\item
During the proof phase we proved that~$95$ out of the extracted~$156$ Boolean functions are universally true.
\item
Finally, after the dominance filtering of the~$95$ proved Boolean functions we obtain the disjunction of
the following seven Boolean functions:

\begin{enumerate}
\begin{minipage}{0.26\linewidth}
\item[\ding{172}] $\Result_1=1$,
\item[\ding{175}] $\Result_1=3~\land~\Result_2\geq 1$,
\end{minipage}
\begin{minipage}{0.415\linewidth}
\item[\ding{173}] $\Result_2=1$,
\item[\ding{176}] $\Result_1=\Upp{\Result_1}(\seqlength)~\land~\Result_2\bmod 2=1$,
\end{minipage}
\begin{minipage}{0.4\linewidth}
\item[\ding{174}] $\Result_1=5~\land~\Result_2\geq 4$,
\item[\ding{177}] $\Result_1\bmod 2=1~\land~\Result_1=\Result_2$,
\end{minipage}
\end{enumerate}
\begin{enumerate}
\begin{minipage}{\linewidth}
\item[\ding{178}] $\seqlength\bmod 2=0~\land~~\Result_1=\Upp{\Result_1}(\seqlength)-1~\land~\Result_2=\Upp{\Result_2}(\seqlength)$.
\end{minipage}
\end{enumerate}

\end{itemize}

All four upper plots and the two lower plots on the left of \figref{fig:groups} contain the groups of infeasible combinations
of~$\Result_1$ and~$\Result_2$ corresponding to the Boolean functions from~\ding{172} to~\ding{177} for~$\seqlength$ being~$9$.
The two lower plots on the right of \figref{fig:groups} contain the infeasible combinations of~$\Result_1$ and~$\Result_2$
corresponding to the~\ding{178} Boolean function for~$\seqlength$ being~$10$ and~$12$, respectively.

The Boolean functions from~\ding{172} to~\ding{176} and~\ding{178} were proved by intersecting the automata for the atomic
relations in these Boolean functions, and check that it was empty.

In order to prove the dependent Boolean function~\ding{177}, we consider the conjunction of three constraints,
namely $\Result_1\bmod 2=1$, $\constraint{sum\_width\_decreasing\_sequence}$, and $\constraint{sum\_width\_zigzag}$.
Each of the three constraints can be represented by an automaton or by a register automaton satisfying the required properties
of the method of \secref{sec:linear-invariants}, which generates for this conjunction the invariant $\Result_1\geq\Result_2+2$.
This proves that~\ding{177} is a universally true Boolean function.

We now give an interpretation of five of those Boolean functions:

\begin{itemize}
\item
\ding{172} and~\ding{173} means that, in the languages of $\constraint{decreasing\_sequence}$ and $\constraint{zigzag}$,
respectively, there is no word consisting of one letter.

\item
\ding{176} means that, when a time series yields $\Upp{\Result_1}(\seqlength)$ as the value of~$\Result_1$,
every occurrence of~$\constraint{zigzag}$ in its signature must start and end  with~$\reg{>}$, and the length
of every word in the language of~$\constraint{zigzag}$ starting and ending with the same letter is even.

\item
\ding{177} is related to the fact that every word in the language of~$\constraint{zigzag}$ contains at least
one word of the language of $\constraint{decreasing\_sequence}$ as a factor, and every such factor is of even length.

\item
\ding{178} means that, when a time series yields $\Upp{\Result_2}(\seqlength)$ as the value of~$\Result_2$,
then its signature is a word in the language of~$\constraint{zigzag}$, and every occurrence of $\constraint{decreasing\_sequence}$
is of even length, and thus~$\Result_1$ must be even.
At the same time, $\Upp{\Result_1}(\seqlength)-1=\seqlength-1$ is odd, when~$\seqlength$ is even. 
\qedexample
\end{itemize}

\end{example}



\section{Synthesising Conditional Automata}
\label{sec:conditional-automata}
For the time-series constraints considered in this work we need to generate constant-size finite
automata representing a certain condition, e.g.~an automaton recognising the signatures of all and only all
time series with the maximum number of peaks.
Such automata are required for proving non-linear invariants parameterised by the time-series length,
described in \secref{sec:non-linear-invariants}, and also for the facet analysis of linear invariants,
described in \secref{sec:sharpness}.
This section shows how to synthesise a constant\nobreakdash-size automaton, i.e.~an automaton whose number
of states is independent, both from the input time-series length and from the values in an input time series,
accepting the signatures of all, and only all, time series satisfying atomic relations of \secref{sec:hypothesis-space}.
For brevity, we only consider the atomic relation (vii)~$\Result=\Upp{\Result}(\seqlength)-\Const$,
where~$\Result$ is constrained by some time-series constraint~$\CtrGamma(\XSeq, \Result)$,
with~$\CtrGamma$ being $\constraint{nb}\_\pattern$ or $\constraint{sum\_width}\_\pattern$,
and where~$\Upp{\Result}(\seqlength)$ is the maximum possible value of~$\Result$ yielded
by a time series of length~$\seqlength$.
This atomic relation is indeed the most difficult case for generating a constant-size automaton.
The construction associated with other atomic relations are described in~\cite{ekaterina_thesis}.
We start with an illustrative example.

\begin{example}[automaton for a gap atomic relation]
Consider the $\constraint{nb\_peak}(\XSeq,\Result)$ time-series constraint and a gap atomic relation~$\Predicate$ defined
by~$\Result=\Upp{\Result}(\seqlength)$.
We showed in~\cite{BoundsConstraints} that the maximum value of~$\Result$ for
a given time-series length~$\seqlength$ is $\Max{0,\Frac{\seqlength-1}{2}}$.
Hence, the automaton for~$\Predicate$ must recognise the signatures of all and only time series
yielding $\Max{0,\Frac{\seqlength-1}{2}}$ as the value of~$\Result$.

Part~(A) of \figref{fig:question_illustration} gives the minimal automaton accepting the set of signatures
reaching this upper bound, while Part~(B) lists all words of length~$4$ and~$5$ over the alphabet $\{\reg{<},\reg{=},\reg{>}\}$
having the maximum number of peaks, $2$ in this case, that can be obtained from the corresponding automaton.
\qedexample
\end{example}

\input{figures/question_illustration}

\noindent The rest of this section is organised as follows:
\begin{itemize}

\item\hspace*{0.1pt}[{\bf Gap Automaton}]
In the context of time-series constraints of the form $\constraint{nb}\_\pattern$ or $\constraint{sum\_width}\_\pattern$,
\secref{sec:gap-automaton} first introduces the notion of \emph{gap of a time series~$\x$}, which indicates how far apart
the result value of a time-series constraint yielded by~$\x$ is from the given upper bound; it then presents
the \emph{main contribution} of this section, namely, the notion of \emph{$\Shift$-gap automaton} for a time-series constraint,
i.e.~a constant-size automaton that only accepts integer sequences whose gap is~$\Shift$. 
Second, it gives a sufficient condition on the time-series constraint for the existence of such an automaton.
Third, it describes how to synthesise such~$\Shift$-gap automaton.
\begin{enumerate}

\item
\secref{sec:crucial-notions} introduces an intermediate notion, the \emph{loss of a time series} wrt a time-series constraint,
which is the maximum difference between the length of this time series and the length of the shortest time series yielding the
same result value of a time-series constraint.
For example, all words of length~$4$ (resp.\ $5$) in Part~(B) of \figref{fig:question_illustration} are the signatures
of time series whose gap is~$0$ and whose loss is~$0$ (resp.\ $1$). Part~(C) of \figref{fig:question_illustration} gives
two signatures of time series with gap (resp.\ loss)~$1$ and~$2$ (resp.\ $3$ and~$5$).

Finally, it introduces the notion of \emph{loss automaton}, i.e.~a register automaton used to compute the loss.
How to synthesise a loss automaton will be explained in \secref{sec:principal-conditions-nb}.

\item
\secref{sec:principal-conditions} introduces a sufficient condition in the form of a conjunction of four conditions
on a time-series constraint, called \emph{principal conditions} that, when satisfied, guarantee the existence of
the~$\Shift$-gap automaton.

  \begin{itemize}
  \item
  When the first three principal conditions hold, describing the set of time series whose gap is~$\Shift$ is equivalent
  to describing the set of time series whose loss belongs to a certain interval, depending on~$\Shift$.    
  \item
  When the fourth principal condition holds, there exists a loss automaton whose registers can either be monotonously
  increased or reset to a natural number.
  \end{itemize}

\item
For a given time-series constraint satisfying the four principal conditions and for any non-negative integer~$\Shift$,
\secref{sec:deriving-scheme} constructively proves the existence of the $\Shift$-gap automaton, i.e.~assuming the loss
automaton is known it shows how to construct the~$\Shift$-gap automaton.

\end{enumerate}
\item\hspace*{0.1pt}[{\bf Loss Automaton}]
For space reason \secref{sec:principal-conditions-nb} focuses only on the construction of the loss automaton for
the $\constraint{nb}\_\pattern$ family, the construction for the $\constraint{sum\_width}\_\pattern$ family being
described in~\cite{ekaterina_thesis}.

It introduces a sufficient condition on a regular expression~$\pattern$ such that, when~$\pattern$ satisfies this condition,
the $\constraint{nb}\_\pattern$ family satisfies the principal conditions of \secref{sec:principal-conditions}.
It also shows how to obtain a loss automaton for a $\constraint{nb}\_\pattern$ time-series constraint from
the seed transducer~\cite{Beldiceanu:synthesis} for~$\pattern$.
The main idea is to compute the \emph{regret} of every transition of the seed transducer as a special case of
\emph{minimax regret}~\cite{FrenchDecisionTheory,Savage} from decision theory, which gives the minimum additional cost
to pay when one action is chosen instead of another.
In~CP, the minimax regret has been used for assessing an extra cost when a  variable is assigned to a given value~\cite{CraigReluPascalDale}.
\end{itemize}

\subsection{Synthesising a $\boldsymbol{\Shift}$-gap Automaton for a Time-Series Constraint \label{sec:gap-automaton}}

We present the main contribution of this section namely a systematic method for deriving a~$\Shift$-gap automaton
for a time-series constraint, see \defref{def:gap-automaton}, satisfying certain conditions that will be given
in \defref{def:principal-conditions}.
We first introduce the \emph{gap of a ground time series} in \defref{def:gap}, and
the \emph{$\Shift$-gap automaton for a time-series constraint} in \defref{def:gap-automaton}.
Let~$\QreClass$ denote the set of time-series constraints of the $\constraint{nb}\_\pattern$
and $\constraint{sum\_width}\_\pattern$ families.

\begin{definition}[gap of a ground time series]
\label{def:gap}
Consider a time-series constraint~$\CtrGamma$ and a ground time series~$\x$ of length~$\seqlength$.
\Definition{gap}{$\x$ wrt $\CtrGamma$}{\CharArg{\Gap}{\CtrGamma}{\x}}{\QreClass\times \IntegersStar}{\Naturals}{}
It is the difference between the maximum value of~$\Result$ that could be yielded by a time series
of length~$\seqlength$, and the value of~$\Result$ yielded by~$\x$.
\end{definition}

\exref{ex:gap-loss} will illustrate the notion of gap for different time series.

\begin{definition}[$\Shift$-gap automaton]
\label{def:gap-automaton}
Consider a time-series constraint~$\CtrGamma$ and a natural number~$\Shift$.
The \emph{$\Shift$-gap automaton for~$\CtrGamma$} is a minimal automaton that accepts the signatures of all,
and only all, ground time series whose gap wrt~$\CtrGamma$ is~$\Shift$.
\end{definition}

\defref{def:principal-conditions} will further give a sufficient condition on a time-series
constraint~$\CtrGamma$ for the existence of a~$\Shift$-gap automaton for~$\CtrGamma$.
\begin{example}[\hspace*{1pt}$0$-gap automaton]
\label{ex:gap-automaton-definition}
The~$0$-gap automaton for $\constraint{nb}\_\PeakPatternName$ was given in Part~(A) of \figref{fig:question_illustration}.
It only recognises the signatures of ground time series containing the maximum number of peaks.
\qedexample
\end{example}
To construct the~$\Shift$-gap automaton for a time-series constraint~$\CtrGamma$
we introduce the notion of \emph{loss of a time series}.
For a time series of length~$\seqlength$, its loss is the difference between~$\seqlength$ and the length
of a shortest time series yielding the same result value of~$\CtrGamma$.
The main idea of our method for generating~$\Shift$-gap automata is that by knowing the loss of a time series,
and whether it contains at least one~$\pattern$-pattern or not, we can determine its gap.

We now describe how to derive the~$\Shift$-gap automaton for a time-series constraint~$\CtrGamma$.

\subsubsection{Defining the Loss and the Loss Automaton}
\label{sec:crucial-notions}

Consider a time-series constraint~$\CtrGamma$ and a natural number~$\Shift$.
\defref{def:loss} introduces the \emph{loss of a time series} wrt~$\CtrGamma$,
and \defref{def:loss-automaton} presents the notion of \emph{loss automaton} for~$\CtrGamma$.

\begin{definition}[loss of a time series]
\label{def:loss}
Consider a time-series constraint~$\CtrGamma$ and a ground time series~$\x$ of length~$\seqlength$.
\Definition{loss}{$\x$ wrt $\CtrGamma$}{\CharArg{\Loss}{\CtrGamma}{\x}}{\QreClass\times\IntegersStar}{\Naturals}{}
It is the difference between~$\seqlength$ and the length of a shortest time series that yields
the same result value of~$\CtrGamma$ as~$\x$.
\end{definition}

\begin{example}[gap and loss of a time series]
\label{ex:gap-loss}
Now we illustrate the computation of the gap and the loss.
Consider the $\constraint{nb}\_\PeakPatternName$ time-series constraint.
From~\cite{BoundsConstraints}, the maximum number of peaks in a time series of length~$\seqlength$
is~$\Max{0,\Frac{\seqlength-1}{2}}$.
\begin{itemize}
\item
The time series $\x^1=\Tuple{1,2,1,2,1,2,1}$ has a gap of~$0$ since it contains three peaks,
which is maximum, and a loss of~$0$ since any shorter time series has a smaller number of peaks.
\item
The time series $\x^2=\Tuple{1,2,1,2,1,1,1,1}$ has a gap of~$1$ since it has only two peaks,
when three is the maximum, and a loss of~$3$ since a shortest time series with~$2$ peaks is
of length~$5$.
\item
The time series $\x^3=\Tuple{1,1,1,0,0,1,1,1,1}$ has a gap of~$4$ since it has no peaks,
when the maximum is~$4$, and a loss of~$8$ since a shortest time series without any peaks is
of length~$1$.\qedexample
\end{itemize}
\end{example}

\begin{definition}[loss automaton for a time-series constraint]
\label{def:loss-automaton}
Consider a time-series constraint~$\CtrGamma$.
A \emph{loss automaton} for~$\CtrGamma$ is a register automaton over the alphabet~$\Curly{<,=,>}$
with a constant number of registers such that, for any ground time series~$\x$, it
returns~$\CharArg{\Loss}{\CtrGamma}{\x}$ after having consumed the signature of~$\x$.
\end{definition}

For the $\constraint{nb}\_\pattern$ and $\constraint{sum\_width}\_\pattern$ families,
a loss automaton can be synthesised from the seed transducer of the regular expression~$\pattern$.
For the $\constraint{nb}\_\pattern$ family, this  will be explained in \secref{sec:principal-conditions-nb}.

\subsubsection{Principal Conditions for Deriving a $\boldsymbol{\Shift}$-Gap Automaton}
\label{sec:principal-conditions}

Consider a~$g\_f\_\pattern$ time-series constraint, denoted by~$\CtrGamma$,
and a natural number~$\Shift$.
\defref{def:principal-conditions} formulates a sufficient condition, consisting of a conjunction
of four conditions, named \emph{principal conditions}, for the existence of
the~$\Shift$-gap automaton for~$\CtrGamma$.
The first three principal conditions express the idea that, knowing the loss of a time series and,
whether it has at least one~$\pattern$-pattern or not, fully determines the gap of this time series.
The fourth condition requires the existence of a loss automaton~$\Aut$ for~$\CtrGamma$,
whose registers may either monotonously increase, or be reset to a natural number,
and each accepting state of~$\Aut$ either accepts only signatures with at least one occurrence
of~$\pattern$, or accepts only signatures without any occurrence of~$\pattern$.

Before formulating the principal conditions, \defref{def:before-after-found} introduces the
notions of before-found and after-found state of a loss automaton.

\begin{definition}[before-found and after-found states]
\label{def:before-after-found}
Consider a loss automaton~$\Aut$ for the~$g\_f\_\pattern$ time-series constraint. 
An accepting state~$q$ of~$\Aut$ is a \emph{before-found} (resp.\ \emph{after-found}) state,
if there exists a time series~$\x$ without any $\pattern$-patterns
(resp.\ with at least one $\pattern$-pattern) such that,
after having consumed the signature of~$\x$, $q$~is the final state of~$\Aut$.
\end{definition}

Note that an accepting state of a loss automaton can have both statuses.

\begin{definition}[principal conditions]
\label{def:principal-conditions}
Consider a $\CtrGamma(\x, \Result)$ time-series constraint.
The \emph{four principal conditions on~$\CtrGamma$} are defined as follows:
\begin{enumerate}

\item
\label{cond:principal-1}
{\bf Gap-to-loss condition.}
There exists a function $\GapToLossFun\colon\QreClass\times\Naturals\times\Curly{0,1}\times\Naturals\rightarrow\Naturals$,
called the \emph{gap-to-loss function}, such that for any ground time series~$\x=\xseq$,
we have $\CharArg{\Loss}{\CtrGamma}{\x}$ being equal to
$\GapToLossFun(\CharArg{\Gap}{\CtrGamma}{\x},\Sgn{\Result},\seqlength)$,
where~$\Signum$ is the signum function.
Hence, in order to compute the loss of a ground time series it is enough to know
(i)~its gap,
(ii)~whether it has at least one~$\pattern$-pattern or not, and
(iii)~the length of this time series.

\item
\label{cond:principal-2}
{\bf Boundedness condition.}
For given values of $\CharArg{\Gap}{\CtrGamma}{\x}$ and~$\Sgn{\Result}$, and for any~$\seqlength$ in~$\Naturals$,
the value of the gap-to-loss function $\GapToLossFun(\CharArg{\Gap}{\CtrGamma}{\x},\Sgn{\Result},\seqlength)$
belongs to a bounded integer interval, called the \emph{loss interval wrt $\Tuple{\CharArg{\Gap}{\CtrGamma}{\x},\Sgn{\Result}}$}.

\item
\label{cond:principal-3}
{\bf Disjointedness condition.}
For a given value of~$\Sgn{\Result}$, and two different values of gap, $\Shift_1$ and~$\Shift_2$,
the loss intervals wrt~$\Tuple{\Shift_1,\Sgn{\Result}}$ and wrt~$\Tuple{\Shift_2,\Sgn{\Result}}$ are disjoint.

\item
\label{cond:principal-4}
{\bf Loss-automaton condition.}
There exists a loss automaton~$\Aut$ for~$\CtrGamma$ satisfying all the following conditions:
\begin{enumerate}

\item
\label{cond:principal-4-1}
Every register update of~$\Aut$ has one of the following forms:

  \begin{enumerate}
  \item
  The register is incremented by a natural number, or by the value of another register.
  \item
  The value of the register is reset to a natural number.
  \end{enumerate}

\item
\label{cond:principal-4-1b}
The initial values of the registers of~$\Aut$ are natural numbers.

\item
\label{cond:principal-4-1c}
The acceptance function of~$\Aut$ is a weighted sum with natural number coefficients of the last values
of the registers of~$\Aut$ after having consumed an input signature.

\item
\label{cond:principal-4-2}
The sets of before-found states and after-found states of~$\Aut$ are disjoint.
It means that, by knowing the final state of~$\Aut$ after having consumed the signature of any
ground time series~$\x$, we also know the value of~$\Sgn{\Result}$ yielded by~$\x$.
\end{enumerate}
\end{enumerate}
Conditions~\ref{cond:principal-1}.,~\ref{cond:principal-2}.,~\ref{cond:principal-3}.
are called the \emph{gap-loss-relation conditions},
Conditions~\ref{cond:principal-4-1},~\ref{cond:principal-4-1b},~\ref{cond:principal-4-1c}
are called the \emph{non-negativity conditions},
while Condition~\ref{cond:principal-4-2} is called the \emph{separation condition} on~$\Aut$.

\end{definition}

\begin{example}[principal conditions]
\label{ex:peak-gap-loss}
Consider a~$\CtrGamma(\x,\Result)$ time-series constraint.
For the time series~$\x^1$, $\x^2$, and~$\x^3$ of \exref{ex:gap-loss},
\figref{fig:gap-loss} shows the relation between the gap, the loss, the time-series lengths,
and~$\Result$ when~$\CtrGamma$ is $\constraint{nb}\_\PeakPatternName$.
For any time series~$\x^i$ (with $i\in[1,3]$) of length~$\seqlength_i$ yielding~$\Result_i$
as the value of~$\Result$, its gap (resp.\ loss) is equal to the length of the violet (resp.\ blue)
dotted line segment starting from the point~$\x^i$ of coordinates $(\seqlength_i,\Result_i)$.
Note that the boundedness and the disjointedness conditions are satisfied for $\constraint{nb}\_\PeakPatternName$.
\qedexample
\end{example}

\begin{figure}[!h]
\floatbox[{\capbeside\thisfloatsetup{capbesideposition={right,top}, 
capbesidewidth=0.65\textwidth}}]{figure}[\FBwidth]
{\caption{The horizontal (resp.\ vertical) axis represents the length of the sequence~$\seqlength$
(resp.\ the result value~$\Result$ of~$\CtrGamma=\constraint{nb}\_\PeakPatternName$).
The red curve shows the maximum value of~$\Result$ for a given~$\seqlength$;
any point~$\x^i$ with coordinates $(\seqlength_i,\Result_i)$ denotes all time series
of length~$\seqlength_i$ yielding~$\Result_i$ as the value of~$\Result$.
The length of the blue (resp.\ violet) dotted line-segments starting from~$\x^i$ equals
the loss (resp.\ gap) of~$\x^i$.
\label{fig:gap-loss}}}
{\scalebox{0.8}{\begin{tikzpicture}[x=0.4cm,y=0.4cm]
\begin{scope}[xshift=1cm]
\def\xmin{0}
\def\xmax{10}
\def\ymin{0}
\def\ymax{5}

\draw[style=help lines, ystep=1, xstep=1] (\xmin,\ymin) grid (\xmax,\ymax); 

\node[anchor=west,fill=white] at (0.01,4.2) {\color{red} \footnotesize
  $\Result=\max(0,\lfloor\frac{n-1}{2}\rfloor)$};

\draw (1,0) -- (2,0) [line width=0.4mm, color=red] node {};
\draw (2,0) -- (3,1) [line width=0.4mm, color=red] node {};
\draw (3,1) -- (4,1) [line width=0.4mm, color=red] node {};
\draw (4,1) -- (5,2) [line width=0.4mm, color=red] node {};
\draw (5,2) -- (6,2) [line width=0.4mm, color=red] node {};
\draw (6,2) -- (7,3) [line width=0.4mm, color=red] node {};
\draw (7,3) -- (8,3) [line width=0.4mm, color=red] node {};
\draw (8,3) -- (9,4) [line width=0.4mm, color=red] node {};
\draw (9,4) -- (10,4)[line width=0.4mm, color=red] node {};

\node at (7.5,3.6) { $\x^1$};
\node at (8.5,2.35) {$\x^2$};
\node at (9.5,0.6) {$\x^3$};

\draw (8,2) -- (5,2) [->, line width=0.35mm, dotted, color=blue] node{};
\draw[black,fill=blue] (5,2) circle [radius=1pt]; 

\draw (8,2) -- (8,3) [->, line width=0.35mm, dotted, color=violet] node{};
\draw[black,fill=violet] (8,3) circle [radius=1pt];

\draw (9,0) -- (1,0) [->, line width=0.35mm, dotted, color=blue] node{};
\draw[black,fill=blue] (1,0) circle [radius=1pt]; 

\draw (9,0) -- (9,4) [->, line width=0.35mm, dotted, color=violet] node{};
\draw[black,fill=violet] (9,4) circle [radius=1pt]; 

\draw (1,0) -- (1,-0.2) [line width=0.1mm, color=blue] node{};
\draw (1,-0.5) -- (1,-1) [line width=0.1mm, color=blue] node{};

\draw (9,0) -- (9,-0.2) [line width=0.1mm, color=blue] node{};
\draw (9,-0.5) -- (9,-1) [line width=0.1mm, color=blue] node{};

\draw (1,-0.8) -- (9,-0.8) [<->, line width=0.15mm, color=blue]
node[midway, below]{\textcolor{blue}{ $\CharArg{\Loss}{\CtrGamma}{\x^3}$}}; 

\draw (9,0) -- (11,0) [line width=0.1mm, color=violet] node{};

\draw (9,4) -- (11,4) [line width=0.1mm, color=violet] node{};

\draw (10.8,0) -- (10.8,4) [<->, line width=0.15mm, color=violet]
node[midway, below, sloped]{\textcolor{violet}{ $\CharArg{\Gap}{\CtrGamma}{\x^3}$}}; 

\draw[->] (\xmin,\ymin) -- (\xmax+0.3,\ymin) node[below] { $\seqlength$};
\draw[->] (\xmin,\ymin) -- (\xmin,\ymax+0.3) node[left]  {$\Result$};

\foreach \x in {1,...,9} \node at (\x, \ymin) [below] {\scriptsize $\x$};
\foreach \y in {0,...,4} \node at (\xmin,\y)  [left]  {\scriptsize $\y$};

\draw[black,fill=black] (7,3) circle [radius=1.5pt];
\draw[black,fill=black] (8,2) circle [radius=1.5pt];
\draw[black,fill=black] (9,0) circle [radius=1.5pt];
\end{scope}
\end{tikzpicture}}
}
\end{figure}


\subsubsection{Synthesising the $\boldsymbol{\Shift}$-Gap Automaton\label{sec:deriving-scheme}}

Consider a~$\CtrGamma$ time-series constraint satisfying all four principal conditions
of \secref{sec:principal-conditions}, and a natural number~$\Shift$.
We prove that the~$\Shift$-gap automaton for~$\CtrGamma$ exists.
First, \lemref{lem:boudedness-disjointedness} states a necessary and sufficient condition in terms
of loss for a ground time series to have its gap being a given constant when the gap-loss-relation
condition is satisfied.
This lemma allows one to describe in terms of loss the set of ground time series whose gap is~$\Shift$.
Then using the result of \lemref{lem:boudedness-disjointedness}, \thref{th:existence-gap-automaton}
constructively proves that the~$\Shift$-gap automaton for~$\CtrGamma$ exists.

\begin{lemma}[relation between gap and loss]
\label{lem:boudedness-disjointedness}
Consider a~$\CtrGamma(\x,\Result)$ time-series constraint such that the gap-loss-relation conditions,
see \defref{def:principal-conditions}, are all satisfied, and a natural number~$\Shift$.
Then, for a time series~$\x$, $\CharArg{\Gap}{\CtrGamma}{\x}$ is~$\Shift$
iff $\CharArg{\Loss}{\CtrGamma}{\x}$ belongs to the loss interval wrt $\Tuple{\Shift,\Sgn{\Result}}$.
\end{lemma}

\begin{proof}
The necessity follows from the boundedness condition, see \condref{cond:principal-2}, and the
sufficiency follows from the disjointedness condition, see \condref{cond:principal-3}
of \defref{def:principal-conditions}.
\end{proof}

\begin{theorem}[existence of the $\Shift$-gap automaton]
\label{th:existence-gap-automaton}
Consider a $g\_f\_\pattern(\x,\Result)$ time-series constraint, denoted by~$\CtrGamma$,
such that all four principal conditions, described in \defref{def:principal-conditions}, are satisfied.
Then the~$\Shift$-gap automaton for~$\CtrGamma$ exists.
\end{theorem}

\begin{proof}
Let us denote by~$\Aut$ the loss automaton for~$\CtrGamma$,
satisfying the non-negativity and the separation conditions.
Note that such automaton necessarily exists since the loss-automaton
condition, see \condref{cond:principal-4} of \defref{def:principal-conditions}, is satisfied.
We prove the theorem by explicitly constructing a constant-size automaton~$\FixedAut$ using~$\Aut$;
after minimising~$\FixedAut$ we obtain the sought~$\Shift$-gap automaton.

\bigbreak

\noindent{\bf [Construction of $\boldsymbol{\FixedAut}$]}
By \lemref{lem:boudedness-disjointedness}, there exist a loss interval~$\LossI{\Shift,0}$
wrt~$\Tuple{\Shift,0}$ and a loss interval~$\LossI{\Shift,1}$ wrt~$\Tuple{\Shift, 1}$
such that any ground time series~$\x$, whose gap is~$\Shift$, belongs to one of the following types:
\begin{itemize}
\item
\label{type-1}
{\bf Type~1.} The time series~$\x$ has no $\pattern$-patterns and the value of
$\CharArg{\Loss}{\CtrGamma}{\x}$ is in~$\LossI{\Shift,0}$.

\item
\label{type-2}
{\bf Type~2.} The time series~$\x$  has at least one $\pattern$-pattern and the value of
$\CharArg{\Loss}{\CtrGamma}{\x}$ is in~$\LossI{\Shift,1}$.
\end{itemize}
Hence, our goal is to construct a constant-size automaton~$\FixedAut$ that recognises
the signatures of all, and only all, ground time series that belongs either to Type~1 or to Type~2.

Let $\Tuple{\Acc_1,\Acc_2,\dots,\Acc_{\nacc}}$ denote the~$\nacc$ registers of the loss automaton~$\Aut$,
whose initial values are $\Tuple{v_1,v_2,\dots,v_{\nacc}}$,
let $\alpha(\Acc_1,\Acc_2,\dots,\Acc_{\nacc})$ denote
the acceptance function of~$\Aut$,
let~$\hat{\delta}$ be the transition function of~$\Aut$,
and let~$\MaxElement$ be the maximum element in $\LossI{\Shift,0}\cup\LossI{\Shift,1}$.
Then, the states, the initial state, the accepting states, and the transitions
of~$\FixedAut$ are defined as follows:
\begin{itemize}
\item {\bf States.}
For every state~$q$ of~$\Aut$, there are $(\MaxElement+2)^\nacc$ states in~$\FixedAut$,
each of which is labelled with $q_{i_1,i_2,\dots,i_{\nacc}}$,
with every~$i_j$ (with $j\in[1,\nacc]$) being in $[0,\MaxElement+1]$.

\item {\bf Initial state.}
If~$q^0$ is the initial state of~$\Aut$, then $q^0_{v_1,v_2,\dots,v_{\nacc}}$
is the initial state of~$\FixedAut$.

\item {\bf Accepting states.}
A state $q_{i_1,i_2,\dots, i_{\nacc}}$ of~$\FixedAut$ is accepting iff either
  \begin{enumerate}
  \item
  $q$ is a before-found state of~$\Aut$ and the value of $\alpha(i_1,i_2,\dots,i_{\nacc})$
  is within~$\LossI{\Shift,0}$, or 
  \item
  $q$ is an after-found state of~$\Aut$ and the value of $\alpha(i_1,i_2,\dots,i_{\nacc})$
  is within $\LossI{\Shift,1}$.
  \end{enumerate}
\item {\bf Transitions.}
There is a transition from state $q_{i_1,i_2,\dots, i_{\nacc}}$
(with $i_1,i_2,\dots,i_{\nacc}\in[0,\MaxElement+1]$) to state
$q^*_{k_1,k_2, \dots,k_{\nacc}}$ labelled with~$s$
in $\Curly{\reg{<},\reg{=},\reg{>}}$, if the value of the transition function
$\hat{\delta}(q,\Tuple{i_1,i_2,\dots,i_{\nacc}},$ $s)$ is equal to
$(q^*,\Tuple{i^*_1,i^*_2,\dots,i^*_\nacc})$, where every~$k_j$ is equal
to $\min(\MaxElement+1,i^*_j)$, with $j$ in $[1,\nacc]$.
\end{itemize}

\bigbreak

\noindent{\bf [Interpretation of the states of $\boldsymbol{\FixedAut}$]}
If after consuming the signature of some ground time series, the automaton~$\FixedAut$
arrives in a state $q_{i_1,i_2,\dots,i_{\nacc}}$, then after consuming the same signature,
the loss automaton~$\Aut$ arrives in state~$q$;
for every $j\in[1,\nacc]$, when $i_j\leq\MaxElement$ (resp.\ $i_j=\MaxElement+1$),
the register~$\Acc_j$ has value~$i_j$ (resp.\ is strictly greater than~$\MaxElement$).
Hence, the states of~$\FixedAut$ encode the register values of~$\Aut$ when consuming the same input signature.

\noindent{\bf [Size of $\boldsymbol{\FixedAut}$]}
By construction, the automaton~$\FixedAut$ has a constant size, i.e.~its number of states
is $m\cdot(\MaxElement+2)^{\nacc}$, where~$m$, $\nacc$ and~$\MaxElement$ are parameters,
i.e.~independent from the time-series length, respectively defined as:
\begin{itemize}
\item 
the number of states of~$\Aut$,
\item
the number of registers of~$\Aut$,
\item
the maximum value of $\LossI{\Shift,0}\cup\LossI{\Shift,1}$,
where~$\LossI{\Shift,0}$ and~$\LossI{\Shift,1}$ are bounded intervals
depending only on the constraint~$\CtrGamma$ and the gap~$\Shift$.
\end{itemize}

We explain why~$\FixedAut$ needs only $m\cdot(\MaxElement+2)^{\nacc}$ states to recognise the signatures of all,
and only all, ground time series of either Type~1 or Type~2.
By the boundedness condition (Condition~\ref{cond:principal-2} of \defref{def:principal-conditions})
and by definition of~$\MaxElement$, for any ground time series whose gap is~$\Shift$, its loss cannot exceed~$\MaxElement$.
We show that if, when consuming the signature of some ground time series, the value of some register of~$\Aut$
becomes greater than~$\MaxElement$, then we no longer need to know its exact value.

Recall that the acceptance function~$\alpha$ of~$\Aut$ is a weighted sum with natural coefficients of
the last values of the registers of~$\Aut$. 
If, for a register~$\Acc_j$, the corresponding coefficient in~$\alpha$ is zero, then it does not affect
the value of~$\alpha$, and the exact value of~$\Acc_j$ is irrelevant.
Otherwise, once the value of~$\Acc_j$ exceeds~$\MaxElement$, the value of~$\alpha$ also exceeds~$\MaxElement$,
and the loss of such a time series is greater than~$\MaxElement$.
By the non-negativity conditions, if the value of~$\Acc_j$ exceeds~$\MaxElement$ it can either increase even more,
or it can be reset to a natural constant.
In either case, the exact value of~$\Acc_j$ is irrelevant,
and it is enough to know a lower bound, $\MaxElement+1$ of its value.

\bigbreak

\noindent{\bf [Correctness of $\boldsymbol{\FixedAut}$]}
We now prove that the constructed automaton~$\FixedAut$ is sound,
i.e.~it recognises the signatures of \emph{only} ground time series of either Type~1 or Type~2,
and complete i.e.~it recognises the signatures of \emph{all} ground time series of either Type~1 or Type~2.
\begin{itemize}

\item
{\bf Soundness of $\boldsymbol{\FixedAut}$.}
We prove the soundness of~$\FixedAut$ by contradiction.
Assume there exists a ground time series~$\x$ recognised by~$\FixedAut$ and whose gap is not~$\Shift$.
Let $q_{i_1,i_2,\dots,i_\nacc}$ be the final state of~$\FixedAut$ after consuming the signature~$\SigSeq$ of~$\x$.
Due to the non-negativity conditions, by construction of~$\FixedAut$ this means that, after consuming~$\SigSeq$,
the register automaton~$\Aut$ finishes in the state~$q$ of~$\Aut$, and for every $j\in[1,\nacc]$,
if $i_j\leq\MaxElement$ (resp.\ $i_j =\MaxElement+1$), then the register~$\Acc_j$ has value~$i_j$
(resp.\ is strictly greater than~$\MaxElement$).
By the separation condition on~$\Aut$, the state~$q$ of~$\Aut$ is either a before-found or an after-found state. 
Since $q_{i_1,i_2,\dots,i_\nacc}$ is an accepting state of~$\FixedAut$,
then either~$q$ is a before-found state and $\alpha(i_1,i_2,\dots,i_\nacc)\in\LossI{\Shift,0}$,
or~$q$ is an after-found state and $\alpha(i_1,i_2,\dots,i_\nacc)\in\LossI{\Shift,1}$.
In the former (resp.\ latter) case, $\x$~belongs to Type~1 (resp.\ Type~2), and
by \lemref{lem:boudedness-disjointedness}, the gap of~$\x$ is~$\Shift$, a contradiction.

\item
{\bf Completeness of $\boldsymbol{\FixedAut}$.}
We prove the completeness of~$\FixedAut$ also by contradiction.
Assume there exists a ground time series~$\x$ whose gap is~$\Shift$, 
i.e.\ it belongs either to Type~1 or to Type~2, but its signature~$S$ is not recognised by~$\FixedAut$.
Then, 
 \begin{enumerate}
 \item
 \label{situation-1}
 either the final state $q_{i_1,i_2,\dots,i_\nacc}$ of~$\FixedAut$ after consuming~$S$ is not accepting, 
 \item
 \label{situation-2}
 or the automaton~$\FixedAut$ cannot consume the full signature~$S$.
 \end{enumerate}
We show that both situations are impossible.
\begin{itemize}
\item {\bf Impossibility of Situation \ref{situation-1}.}
Due to the non-negativity conditions, and by construction of~$\FixedAut$,
after having consumed the signature of~$\x$, the automaton~$\Aut$ ends in state~$q$ of~$\Aut$,
and the value of the acceptance function is equal to $\alpha(i_1,i_2,\dots,i_\nacc)$.
Since the gap of~$\x$ is~$\Shift$, by \lemref{lem:boudedness-disjointedness} and by the separation condition,
either~$q$ is a before-found state of~$\Aut$ and $\alpha(i_1,i_2,\dots,i_\nacc)$ belongs
to~$\LossI{\Shift,0}$ or $q$ is an after-found state of $\Aut$ and $\alpha(i_1,i_2,\dots,i_\nacc)$ belongs
to~$\LossI{\Shift,1}$.
In either case, the state $q_{i_1,i_2,\dots,i_\nacc}$ of~$\FixedAut$ must be accepting by construction,
thus Situation~\ref{situation-1} is impossible.

\item {\bf Impossibility of Situation \ref{situation-2}.}
Assume that (1)~at a state $q_{i_1,i_2,\dots,i_\nacc}$ of~$\FixedAut$, there does not exist a transition
labelled with some input symbol~$s$, and that (2)~$\FixedAut$ needs to trigger this transition when
consuming the signature of~$\x$.
Then, at state~$q$ of~$\Aut$, there does not exist a transition labelled with~$s$.
This contradicts the nature of the loss automaton~$\Aut$ since it must compute the loss of any ground time series,
and thus accept any time series.
Hence, Situation~\ref{situation-2} is also impossible.
\end{itemize}

Therefore, both situations are impossible, which implies that the time series~$\x$ does not exist,
and thus the automaton~$\FixedAut$ is complete.
\end{itemize}

Since~$\FixedAut$ is sound and complete, the minimisation of~$\FixedAut$ gives the sought~$\Shift$-gap automaton.
\hspace{\fill} \qed
\end{proof}


\subsection{Synthesising the Loss Automaton for the $\constraint{nb}\_\boldsymbol{\pattern}$ Family \label{sec:principal-conditions-nb}}
First, for the $\constraint{nb}\_\pattern$ family, we show that, when~$\pattern$ has a property,
named the \emph{\HomogeneityProp{,}} the first three principal conditions of \defref{def:principal-conditions} are satisfied.
Second, based on the~\HomogeneityProp{} we show how to satisfy the fourth principal condition by constructing from the seed
transducer for~$\pattern$ a loss automaton satisfying the loss-automaton condition.
Consequently, the constructive proof of \thref{th:existence-gap-automaton} can be used to derive the~$\Shift$-gap automaton.

\begin{sloppypar}
\begin{enumerate}
\item
\secref{sec:assumption-nb} introduces the~\HomogeneityProp{}.
Sections~\ref{sec:lemmas-nb} and~\ref{sec:loss-automata-nb} both assume the~\HomogeneityProp{}.
\item
\secref{sec:lemmas-nb} proves three theorems stating that, the gap-to-loss, the boundedness,
and the disjointedness conditions are satisfied for $\constraint{nb}\_\pattern$.
\item
\secref{sec:loss-automata-nb} gives a systematic method for constructing a loss automaton~$\Aut$
satisfying the non\nobreakdash-negativity and the separation conditions.
\end{enumerate}
\end{sloppypar}

\subsubsection{The \textsc{HOMOGENEITY} Property}
\label{sec:assumption-nb}

\begin{property}[\HomogeneityProp{}]
\label{prop:homogeneity}
A regular expression~$\pattern$ has the~\HomogeneityProp{} if the following conditions are both satisfied:

\begin{enumerate}
\item
The pair $\Tuple{\pattern,\Char{\Before}{\pattern}}$ is a \emph{recognisable pattern}~\cite{ASTRA:ICTAI17:generation}.
This implies that the seed transducer~$\Transducer{\pattern}$ for~$\pattern$ exists and can be constructed by
the method of~\cite{ASTRA:ICTAI17:generation}.
\item
For any state~$q$ of~$\Transducer{\pattern}$ that is the destination state of a~$\Found$-transition,
the number of transitions in the shortest~$\Found$-path starting from~$q$ is a constant that does not depend on~$q$.
\end{enumerate}
\end{property}

For a regular expression~$\pattern$ with the~\HomogeneityProp{}, the following lemma gives the maximum number
of~$\pattern$-patterns in a time series of length~$\seqlength$. 

\begin{lemma}[maximum of the result value]
\label{lem:max-value-nb}
Consider a time-series constraint $\constraint{nb}\_\pattern$ such that~$\pattern$ has the~\HomogeneityProp{},
and~$\Transducer{\pattern}$ denotes the seed transducer for~$\pattern$.
Let~$\D{\pattern}$ denote the length of shortest~$\Found$-path in~$\Transducer{\pattern}$
starting from any state that is the destination of a~$\Found$-transition,
and let~$\C{\pattern}$ denote the difference between~$\D{\pattern}$ and
the length of shortest~$\Found$-path in~$\Transducer{\pattern}$ starting
from the initial state of~$\Transducer{\pattern}$.
Then, the maximum number of~$\pattern$-patterns in a time series of length~$\seqlength$ is computed as 

\begin{equation}
\label{eq:nb-up}
\Frac{\seqlength-\C{\pattern}}{\D{\pattern}}.
\end{equation}

\end{lemma}
\begin{proof}
For any time series~$X=\XSeq$, there is a bijection between its set of~$\pattern$-patterns and the~$\Found$ symbols
in the output sequence of~$\Transducer{\pattern}$ after consuming the signature of~$X$.
Hence, we need to show that $\Frac{\seqlength-\C{\pattern}}{\D{\pattern}}$ is the maximum number of the~$\Found$ symbols
in the output sequence~$\OutputSeq$ of~$\Transducer{\pattern}$ after having consumed the signature of any time series of
length~$\seqlength$.
The first~$\Found$ symbol in~$\OutputSeq$ cannot occur before the position~$\ell$,
where~$\ell$ is the length of the shortest~$\Found$-path starting from the initial state.
Since~$\Transducer{\pattern}$ has the~\HomogeneityProp{} then every other~$\Found$ symbol
can occur in~$\OutputSeq$ with the interval of~$\D{\pattern}$.
Such an~$\OutputSeq$ output sequence has the number of ~$\Found$ symbols
being equal to $\Frac{\seqlength-(\ell-\D{\pattern})}{\D{\pattern}}$.
We replace $\ell-\D{\pattern}$ with~$\C{\pattern}$ and obtain Formula~(\ref{eq:nb-up}).
\hspace{\fill} \qed
\end{proof}

\subsubsection{Verifying the Gap-Loss-Relation Conditions}
\label{sec:lemmas-nb}

This section shows that the gap-loss-relation conditions, see~\defref{def:principal-conditions},
for a $\constraint{nb}\_\pattern$ time-series constraint are satisfied, assuming~$\pattern$ has the~\HomogeneityProp{.}
\thref{th:gap-to-loss-nb} proves the gap-to-loss condition and derives the formula for the gap-to-loss function;
\thref{th:boundedness-nb} proves the boundedness condition and derives the formula of loss interval for a given gap and
sign of the result value, and, finally, \thref{th:disjointedness-nb} proves the disjointedness condition.
\begin{theorem}[gap-to-loss condition]
\label{th:gap-to-loss-nb}
Consider a~$\CtrGamma(\x,\Result)$ time-series constraint that belongs to the $\constraint{nb}\_\pattern$ family
with~$\pattern$ having the~\HomogeneityProp{.}
First, the gap-to-loss condition is satisfied for~$\CtrGamma$.
Second, for any ground time series~$\x$ of length~$\seqlength$, the gap-to-loss function is defined by:
\begin{equation}
\label{eq:lemma-gap-loss}
\CharArg{\Loss}{\CtrGamma}{\x}=\CharArg{\Gap}{\CtrGamma}{\x}\cdot\D{\pattern}+
(1-\Sgn{\Result})\cdot(\min(\seqlength,\C{\pattern})-1)+\max(0,\seqlength-\C{\pattern})\bmod\D{\pattern}, 
\end{equation}
where~$\Signum$ is the signum function, and~$\C{\pattern}$ and~$\D{\pattern}$ are the constants from
the maximum value of~$\Result$ given in \lemref{lem:max-value-nb}.
\end{theorem}

\begin{proof}
We successively consider two disjoint cases wrt~$\Sgn{\Result}$.

\bigbreak

\noindent {\bf [$\boldsymbol{\Sgn{\Result}}$ is zero]}
We need to prove that $\CharArg{\Loss}{\CtrGamma}{\x}$ is equal to
$\CharArg{\Gap}{\CtrGamma}{\x}\cdot\D{\pattern}+\min(\seqlength,\C{\pattern})-1+\max(0,\seqlength-\C{\pattern})\bmod\D{\pattern}$.
When~$\Result$ is zero, the loss of~$\x$ is~$\seqlength-1$ since a shortest time series without any~$\pattern$-patterns is of length~$1$.
Thus, we need to show that
$\CharArg{\Gap}{\CtrGamma}{\x}\cdot\D{\pattern}+\min(\seqlength,\C{\pattern})-1+\max(0,\seqlength-\C{\pattern})\bmod\D{\pattern}$
is equal to~$\seqlength-1$.
From the maximum value of~$\Result$, given by the~\HomogeneityProp{,} we have the following equality:
\begin{equation}
\label{eq:lemma-gap-loss-1}
\CharArg{\Gap}{\CtrGamma}{\x}=\Max{0,\Frac{\seqlength-\C{\pattern}}{\D{\pattern}}}-\Result=\Max{0,\Frac{\seqlength-\C{\pattern}}{\D{\pattern}}}. 
\end{equation}
Let us consider two cases wrt the value of $\CharArg{\Gap}{\CtrGamma}{\x}$, namely:

\begin{itemize}
\item
$\CharArg{\Gap}{\CtrGamma}{\x}$ is zero.
By~\eqref{eq:lemma-gap-loss-1}, $\seqlength<\C{\pattern}+\D{\pattern}$, and the value of the right-hand side of~\eqref{eq:lemma-gap-loss}
is equal to $\min(\seqlength,\C{\pattern})-1+\max(0,\seqlength-\C{\pattern})$, which is~$\seqlength-1$.

\item
$\CharArg{\Gap}{\CtrGamma}{\x}$ is positive.
Then, by~\eqref{eq:lemma-gap-loss-1}, $\seqlength\geq\C{\pattern}+\D{\pattern}$,
and we have the following equality: 
\begin{equation}
\label{eq:lemma-gap-loss-2}
\CharArg{\Gap}{\CtrGamma}{\x}=\Frac{\seqlength-\C{\pattern}}{\D{\pattern}}=
\frac{\seqlength-\C{\pattern}-(\seqlength-\C{\pattern})\bmod\D{\pattern}}{\D{\pattern}}
\end{equation}
From~\eqref{eq:lemma-gap-loss-2} we obtain the expression for $\seqlength-1$, which is
$\CharArg{\Gap}{\CtrGamma}{\x}\cdot\D{\pattern}+\C{\pattern}-1+(\seqlength-\C{\pattern})\bmod\D{\pattern}$.
\end{itemize}

\bigbreak

\noindent{\bf [$\boldsymbol{\Sgn{\Result}}$ is one]}
We need to prove that $\CharArg{\Loss}{\CtrGamma}{\x}$ is equal to
$\CharArg{\Gap}{\CtrGamma}{\x}\cdot\D{\pattern}+\max(0,\seqlength-\C{\pattern})\bmod\D{\pattern}$.
Since~$\Result$ is positive, $\seqlength$~is strictly greater than~$\C{\pattern}$, and thus
$\max(0,\seqlength-\C{\pattern})$ is equal to~$\seqlength-\C{\pattern}$.
Further, by definitions of gap and loss, we have:
\begin{equation}
\label{eq:lemma-gap-loss-3}
\CharArg{\Gap}{\CtrGamma}{\x}=\Frac{\seqlength-\C{\pattern}}{\D{\pattern}}-\Result=
\frac{\seqlength-\C{\pattern}-(\seqlength-\C{\pattern})\bmod\D{\pattern}}{\D{\pattern}}-\frac{(\seqlength-\CharArg{\Loss}{\CtrGamma}{\x})-\C{\pattern}}{\D{\pattern}} 
\end{equation}
Since on the right-hand side of~\eqref{eq:lemma-gap-loss-3}, both divisions are integer divisions we obtain:
\begin{equation}
\label{eq:lemma-gap-loss-4}
\CharArg{\Gap}{\CtrGamma}{\x}=\frac{\CharArg{\Loss}{\CtrGamma}{\x}-(\seqlength-\C{\pattern})\bmod\D{\pattern}}{\D{\pattern}}.
\end{equation}
By isolating $\CharArg{\Loss}{\CtrGamma}{\x}$ from~\eqref{eq:lemma-gap-loss-4} we obtain the formula of the theorem.
\hspace{\fill} \qed
\end{proof}

\begin{example}[gap-to-loss condition]
\label{ex:loss-to-gap-function-peak}
Consider a $\constraint{nb}\_\pattern(\XSeq,\Result)$ time-series constraint with~$\pattern$ being the~$\PeakPatternName$
regular expression, which has the~\HomogeneityProp{}.
Hence, we can apply~\thref{th:gap-to-loss-nb} for computing the gap-to-loss function for $\constraint{nb}\_\pattern$.
By~\lemref{lem:max-value-nb}, the maximum value of~$\Result$ is $\Max{0,\Frac{\seqlength-1}2}$, and thus~$\C{\pattern}$
and~$\D{\pattern}$, are~$1$ and~$2$, respectively.
Then the gap-to-loss function for $\constraint{nb}\_\pattern$ is

$\CharArg{\Loss}{\CtrGamma}{\x}=2\cdot\CharArg{\Gap}{\CtrGamma}{\x}+\max(0,\seqlength-1)\bmod 2$.
\qedexample
\end{example}

\begin{theorem}[boundedness condition]
\label{th:boundedness-nb}
Consider a~$\CtrGamma(\x,\Result)$ time-series constraint that belongs to the $\constraint{nb}\_\pattern$ family
with~$\pattern$ having the~\HomogeneityProp{.}
First, the boundedness condition is satisfied for~$\CtrGamma$;
second, for any given gap~$\Shift$ and any value of~$\Sgn{\Result}$, 
the loss interval~$\LossInt$ wrt~$\Tuple{\Shift,\Sgn{\Result}}$ is defined by:
\begin{enumerate}[(i)]
\item
$\LossMin\hspace*{1pt}=\Shift\cdot\D{\pattern}+(1-\Sgn{\Result})\cdot\Sgn{\Shift}\cdot(\C{\pattern}-1)$,

\item
$\LossMax=\D{\pattern}\cdot(\Shift+1)-1+(1-\Sgn{\Result})\cdot(\C{\pattern}-1)$.
\end{enumerate}

\end{theorem}

\begin{proof}
Let~$\x$ be a ground time series of length~$\seqlength$ whose gap is~$\Shift$.
From~\thref{th:gap-to-loss-nb}, we have that $\CharArg{\Loss}{\CtrGamma}{\x}$ is
$\Shift\cdot\D{\pattern}+(1-\Sign)\cdot(\min(\seqlength,\C{\pattern})-1)+\max(0,\seqlength-\C{\pattern})\bmod\D{\pattern}$. 
By case analysis wrt the value of~$\Sign$, i.e.~either~$0$ or~$1$,
we now show that $\LossMin\leq\CharArg{\Loss}{\CtrGamma}{\x}\leq\LossMax$.

\bigbreak

\noindent {\bf[$\boldsymbol{\Sign}$ is zero]} 
In this case, $\CharArg{\Loss}{\CtrGamma}{\x}$ simplifies to
$\Shift\cdot\D{\pattern}+\min(\seqlength,\C{\pattern})-1+\max(0,\seqlength-\C{\pattern})\bmod\D{\pattern}$. 
Since $\Shift\cdot\D{\pattern}-1$ is a constant, in order to prove that $\LossMin$ (resp.\ $\LossMax$)
is a lower (resp.\ upper) bound on $\CharArg{\Loss}{\CtrGamma}{\x}$, we need to find the minimum (resp.\ maximum)
of the function~$z(\seqlength)=\min(\seqlength,\C{\pattern})+\max(0,\seqlength-\C{\pattern})\bmod\D{\pattern}$.
\begin{enumerate}[(i)]
\item
{\bf $\boldsymbol{\LossMin}\leq\boldsymbol{\CharArg{\Loss}{\CtrGamma}{\x}}$}.
We prove that $\CharArg{\Loss}{\CtrGamma}{\x}=\Shift\cdot\D{\pattern}+z(\seqlength)\geq\LossMin$
by case analysis on~$\Shift$:
\begin{enumerate}
\item
{\bf [$\boldsymbol{\Sgn{\Shift}}$ is zero]}  As shown in the proof of~\thref{th:gap-to-loss-nb},
$\seqlength<\C{\pattern}+\D{\pattern}$ and the minimum value of the function~$z(\seqlength)$ is~$1$,
and is reached for~$\seqlength$ being~$1$.
\item
{\bf [$\boldsymbol{\Sgn{\Shift}}$ is one]} We have~$\seqlength\geq\C{\pattern}+\D{\pattern}$,
and thus~$\min(\seqlength,\C{\pattern})$ is equal to~$\C{\pattern}$, and the minimum value of
the function~$z(\seqlength)$ is~$\C{\pattern}$.
\end{enumerate}

Hence, $\Shift\cdot\D{\pattern}+\Sgn{\Shift}\cdot(\C{\pattern} -1)$ is indeed
a lower bound on~$\CharArg{\Loss}{\CtrGamma}{\x}$ when~$\Sign$ is zero.

\item
{\bf $\boldsymbol{\LossMax}\geq\boldsymbol{\CharArg{\Loss}{\CtrGamma}{\x}}$}.
We prove that $\CharArg{\Loss}{\CtrGamma}{\x}\leq\LossMax$.
The maximum value of~$z(\seqlength)$ is $\C{\pattern}+\D{\pattern}-1$.
Hence, 
$\D{\pattern}\cdot(\Shift+1)-1+\C{\pattern}-1$ is indeed an upper bound on $\CharArg{\Loss}{\CtrGamma}{\x}$.

\end{enumerate}

\bigbreak

\noindent {\bf [$\boldsymbol{\Sign}$ is one]}
In this case,~$\CharArg{\Loss}{\CtrGamma}{\x}$ simplifies to
$\Shift\cdot\D{\pattern}+\max(0,\seqlength-\C{\pattern})\bmod\D{\pattern}$. 
A lower (resp.~upper) bound on $(\seqlength-\C{\pattern})\bmod\D{\pattern}$ is zero (resp.\ $\D{\pattern}-1$).
Hence, $\LossMin$ and~$\LossMax$ are, respectively, a lower and an upper bound on~$\CharArg{\Loss}{\CtrGamma}{\x}$.
\hspace{\fill} \qed
\end{proof}

\begin{example}[boundedness condition]
\label{ex:boundedness-peak}
Consider a $\constraint{nb}\_\pattern(\x,\Result)$ time-series constraint with~$\pattern$ being the $\PeakPatternName$ regular expression.
Since~$\pattern$ has the~\HomogeneityProp{} we can apply \thref{th:boundedness-nb} for computing the loss interval
for $\constraint{nb}\_\pattern$.
Recall that the values of~$\C{\pattern}$ and~$\D{\pattern}$, are respectively, $1$~and~$2$.
Then, for any value~$\Shift$ of gap and any value of~$\Sgn{\Result}$,
the loss interval wrt~$\Tuple{\Shift,\Sign}$ is $[2\cdot\Shift,2\cdot\Shift+1]$.
\qedexample
\end{example}

\begin{theorem}[disjointedness condition]
\label{th:disjointedness-nb}
Consider a $\constraint{nb}\_\pattern(\XSeq,\Result)$ time-series constraint such that~$\pattern$ has the~\HomogeneityProp{.}
Then the disjointedness condition is satisfied for~$\constraint{nb}\_\pattern$.
\end{theorem}

\begin{proof}
The disjointedness condition can be proved using the formula of the loss interval of~\thref{th:boundedness-nb}.
For each value of~$\Sgn{\Result}$, i.e.~either $0$~or~$1$, we take two different values of gap,
w.l.o.g.~$\Shift$ and~$\Shift+t$ with a non-negative integer~$t$, and show that the upper limit of the loss interval
wrt~$\Tuple{\Shift,\Sgn{\Result}}$ is strictly less than the lower limit of the loss interval wrt $\Tuple{\Shift+t,\Sgn{\Result}}$.
This implies the disjointedness condition.

~\hspace{\fill} \qed
\end{proof}

\subsubsection{Verifying the  Loss-Automaton Condition}
\label{sec:loss-automata-nb}

We focus on the loss-automaton condition for the $\constraint{nb}\_\pattern$ time-series constraints,
i.e.~we construct a loss automaton~$\Aut$ for $\constraint{nb}\_\pattern$ satisfying the non-negativity and the separation conditions.
This is done by deriving~$\Aut$ from a seed transducer for~$\pattern$, which exists assuming~$\pattern$ has
the~\HomogeneityProp{~\cite{ASTRA:ICTAI17:generation}.}
In order to satisfy the separation condition for the loss automaton for $\constraint{nb}\_\pattern$,
we require the seed transducer for~$\pattern$ to have a specific form that we now introduce in~\defref{def:separated-transducer}.

\begin{definition}[separated seed transducer]
\label{def:separated-transducer}
Given a regular expression~$\pattern$, a seed transducer~$\Transducer{\pattern}$ for~$\pattern$ is \emph{separated}
iff for any state~$q$ of~$\Transducer{\pattern}$, one of the two following conditions holds:
\begin{enumerate}
\item
Any path from the initial state of~$\Transducer{\pattern}$ to~$q$ is a~$\Found$-path.
\item
There are no~$\Found$-paths from the initial state of~$\Transducer{\pattern}$ to~$q$.
\end{enumerate}
\end{definition}

\begin{figure}[!h]
\scalebox{0.8}{
\begin{tikzpicture}
\begin{scope}[xshift=-2.3cm,yshift=-0.6cm,->,>=stealth',shorten >=1pt,auto,node distance=32mm,semithick,
              information text/.style={rounded corners=1pt,inner sep=1ex}]
\node[initial,initial where=above,initial text=,initial distance=3mm,accepting,state,draw=black,fill=white] (s) {$s$};
\node[accepting,state,draw=black,fill=white] (r)[below left of=s] {$r$};
\node[accepting,state,draw=black,fill=white] (t)[below right of=s] {$t$};
\path
      (s)    edge [loop right,draw=black]    node{\scriptsize $\begin{array}{c} > 
                                    \hspace*{2pt}:{\color{black}\mathbf{not\_found}}
                                    \end{array}$} (s)
       (s)    edge [draw=black,out=65,in=30,looseness=7] node[right]{\scriptsize $\begin{array}{c} = 
                                    \hspace*{2pt}:{\color{black}\mathbf{not\_found}}
                                    \end{array}$} (s)
      (s)    edge [bend right,draw=black]    node[left] {\scriptsize $\begin{array}{c} < 
                                    \hspace*{2pt}:{\color{black}\mathbf{not\_found}}
                                    \end{array}$} (r)
      (r)    edge [draw=black]    node{\scriptsize $\begin{array}{c} > 
                                    \hspace*{2pt}:{\color{black}\mathbf{found}}
                                    \end{array}$} (t)
      (r)    edge [loop left,draw=black]    node[left]{\scriptsize $\begin{array}{c} < 
                                    \hspace*{2pt}:{\color{black}\mathbf{not\_found}}
                                    \end{array}$} (r)
 (r)    edge [loop below,draw=black]    node[below]{\scriptsize $\begin{array}{c} = 
                                    \hspace*{2pt}:{\color{black}\mathbf{not\_found}}
                                    \end{array}$} (r)
      (t)    edge [loop above,draw=black]    node[above]{\scriptsize $\begin{array}{c} > 
                                    \hspace*{2pt}:{\color{black}\mathbf{not\_found}}
                                    \end{array}$} (t)
      (t)    edge [loop below,draw=black]    node[below]{\scriptsize $\begin{array}{c} = 
                                    \hspace*{2pt}:{\color{black}\mathbf{not\_found}}
                                    \end{array}$} (t)
      (t)    edge [bend left,draw=black]    node {\scriptsize $\begin{array}{c} < 
                                    \hspace*{2pt}:{\color{black}\mathbf{not\_found}}
                                    \end{array}$} (r);
\node at (-3.8,-4.5) {\small (A)};
\end{scope}
\begin{scope}[xshift=4.5cm,yshift=-0.6cm,->,>=stealth',shorten >=1pt,auto,node distance=24mm,semithick,
              information text/.style={rounded corners=1pt,inner sep=1ex}]
 \node[initial,initial where=above,initial text=,initial distance=3mm,accepting,state,draw=black, fill=white] (t) {$s$};
 \node[accepting,state,fill=white]                                                                (s) at (3.33,0) {$r$};
 \node[accepting,state, fill=white]                                                               (s1)[below of=t]{$r'$};
      \node[accepting,state, fill=white]                                                          (t1)[below of=s]{$t$};
      \path
        (t) edge [loop left,draw=black]               node [left] {\scriptsize $\begin{array}{c} 
                                    =\hspace*{2pt}:{\color{black}\mathbf{not\_found}}
                                    \end{array}$} (t)
 (t) edge [loop below,draw=black]               node [below] {\scriptsize $\begin{array}{c} 
                                    >\hspace*{2pt}:{\color{black}\mathbf{not\_found}}
                                    \end{array}$} (t)
        (t) edge [draw=black] node [above]{\scriptsize $\begin{array}{c}  
                                    <\hspace*{2pt}:{\color{black}\mathbf{not\_found}}
                                    \end{array}$}  (s)
	 (s) edge [loop right,draw=black]              node [right]{\scriptsize $\begin{array}{c} 
                                    =\hspace*{2pt}:{\color{black}\mathbf{not\_found}}
                                    \end{array}$}  (s)
(s) edge [loop above,draw=black]              node [above]{\scriptsize $\begin{array}{c} 
                                    <\hspace*{2pt}:{\color{black}\mathbf{not\_found}}
                                    \end{array}$}  (s)
        (s) edge [draw=black] node [right]{\scriptsize $\begin{array}{c} 
                                    >\hspace*{2pt}:{\color{black}\mathbf{found}}
                                    \end{array}$}  (t1)
        (t1)edge [loop right,draw=black]              node [right]{\scriptsize $\begin{array}{c} 
                                    =\hspace*{2pt}:{\color{black}\mathbf{not\_found}}
                                    \end{array}$}  (t1)
	(t1)edge [loop below,draw=black]              node [below]{\scriptsize $\begin{array}{c} 
                                    >\hspace*{2pt}:{\color{black}\mathbf{not\_found}}
                                    \end{array}$} (t1)
	(t1)edge [bend angle=20,bend left] node [below]{\scriptsize $\begin{array}{c}  
                                    <\hspace*{2pt}:{\color{black}\mathbf{not\_found}}
                                    \end{array}$}(s1)
        (s1)edge [loop left,draw=black]               node [left]{\scriptsize $\begin{array}{c} 
                                    =\hspace*{2pt}:{\color{black}\mathbf{not\_found}}
                                    \end{array}$}  (s1)
(s1)edge [loop below,draw=black]               node [below]{\scriptsize $\begin{array}{c} 
                                    <\hspace*{2pt}:{\color{black}\mathbf{not\_found}}
                                    \end{array}$}  (s1)
        
        (s1)edge [draw=black] node [above]{\scriptsize $\begin{array}{c} > 
                                    \hspace*{2pt}:{\color{black}\mathbf{not\_found}}
                                    \end{array}$}  (t1);
\node at (-3,-4.5){\small (B)};
\end{scope}
\end{tikzpicture}}
\caption{\label{fig:peak-transducer-separated} (A)~Seed transducer and (B)~separated seed transducer for the $\PeakPatternName$ regular expression.}
\end{figure}
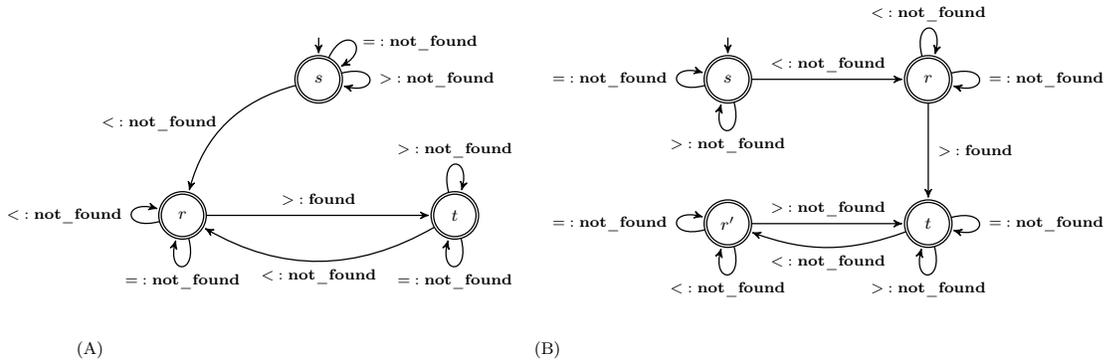


\begin{example}[separated seed transducer]
\label{ex:peak-transducer-to-loss-automaton}
Part~(B) of \figref{fig:peak-transducer-separated} gives the separated seed transducer
for $\PeakPatternName$ obtained from the seed transducer in Part~(A).
\qedexample
\end{example}

Note that, even if the seed transducer for~$\pattern$ constructed by the method of~\cite{ASTRA:ICTAI17:generation} is not separated,
it can be easily made so by duplicating some of its states.
Subsequently we assume that the seed transducer for~$\pattern$ is separated,
and we derive the loss automaton~$\Aut$ in the same way as we generate register
automata for time-series constraints~\cite{Beldiceanu:synthesis}, namely:
\begin{enumerate}
\item
First, we identify the required registers of~$\Aut$ and their role.
\item
Second, to each phase letter of the output alphabet of the seed transducer for~$\pattern$,
we associate a set of instructions, i.e.~register updates.
The loss automaton~$\Aut$ is obtained by replacing every phase letter of the seed transducer for~$\pattern$
by the corresponding set of instructions.
\end{enumerate}

\paragraph{Identifying the Required Registers of the Loss Automaton}
\label{sec:registers-nb}
Consider a $\constraint{nb}\_\pattern$ time-series constraint.
Intuitively, when consuming the signature of a ground time series, every transition triggered by the seed
transducer~$\Transducer{\pattern}$ for~$\pattern$ has a certain impact on the loss of this time series.
To quantify this impact for the case of $\constraint{nb}\_\pattern$ time-series constraints, 
\defref{def:regret} introduces the notion of \emph{regret of a transition} of a seed transducer for~$\pattern$.
The regret of a transition $t$ gives how many additional transitions $\Transducer{\pattern}$ has to trigger,
before it can trigger the next~$\Found$-transition, if it triggers~$t$ rather than the transition on a shortest~$\Found$-path.

\begin{definition}[regret of a transition]
\label{def:regret}
Consider a regular expression~$\pattern$ and its seed transducer~$\Transducer{\pattern}$.
For any transition~$\Tr$ of~$\Transducer{\pattern}$ from state~$q_1$ to state~$q_2$,
the \emph{regret} of~$\Tr$ equals one plus the difference between the lengths of
the shortest~$\Found$-paths from~$q_2$, respectively~$q_1$.
\end{definition}

\begin{example}[regret of a transition]
\label{ex:regret}
Consider the~$\PeakPatternName$ regular expression, whose separated seed transducer is given in Part~(B)
of~\figref{fig:peak-transducer-separated}. We denote by~$\trans{q_1}{q_2}{a}$ a transition of the seed transducer
from state~$q_1$ to state~$q_2$ whose input symbol is~$a$.
All transitions in~$\Curly{\trans{s}{r}{<},\trans{r}{t}{>},\trans{t}{r'}{>},\trans{r'}{t}{<}}$ between
two distinct states have a regret of~$0$, while all transitions
in~$\Curly{\trans{s}{s}{>},\trans{s}{s}{=},\trans{r}{r}{<},\trans{r}{r}{=},\trans{t}{t}{>},\trans{t}{t}{=},\trans{r'}{r'}{<},\trans{r'}{r'}{=}}$
have a regret of~$1$.
\qedexample
\end{example}

\lemref{lem:regret-loss} shows the connection between the loss of a ground time series~$\x$ and the regret of the transitions
triggered by the seed transducer for~$\pattern$ when consuming the signature of~$\x$.

\begin{lemma}[regret-loss relation]
\label{lem:regret-loss}
Consider a~$\CtrGamma(\x,\Result)$ time-series constraint with~$\CtrGamma$ being $\constraint{nb}\_\pattern$
such that~$\pattern$ has the \HomogeneityProp{.}
Let~$t=\Tuple{t_1,t_2,\dots,t_{\seqlength-1}}$ denote the sequence of transitions triggered
by the seed transducer~$\Transducer{\pattern}$ for~$\pattern$ upon consuming the signature of~$\x=\xseq$,
and let~$\LastFoundIndex$ denote the index of the last~$\Found$-transition in~$t$, if no such transition exists,
$\LastFoundIndex$~is zero.
The following equality holds:

$\CharArg{\Loss}{\CtrGamma}{\x}=\seqlength-1-\LastFoundIndex+\sum_{i=1}\limits^{\LastFoundIndex}\Regret{t_i}$,
where~$\Regret{t_i}$ denotes the regret of transition~$t_i$.
\end{lemma}

\begin{proof}
Since $\Tuple{t_{\LastFoundIndex+1},t_{\LastFoundIndex+2},\dots,t_{\seqlength-1}}$
does not contain any~$\Found$-transition, it implies that the loss of~$\x$ is at least $\seqlength-1-\LastFoundIndex$.
Then, the sum $\sum\limits_{i=1}^{\LastFoundIndex}\Regret{t_i}$ shows how many additional transitions were triggered
to achieve the same number of~$\Found$-transitions in the output sequence.
Hence, the loss of~$\x$ is the sum of $\seqlength-1-\LastFoundIndex$ and $\sum\limits_{i=1}^{\LastFoundIndex}\Regret{t_i}$.
\hspace{\fill} \qed
\end{proof}

\begin{example}[regret-loss transition]
\label{ex:regret-peak}
Consider the $\PeakPatternName$ regular expression, whose
separated seed transducer $\Transducer{\PeakPatternName}$ is given
in Part~(B) of \figref{fig:peak-transducer-separated}. 
Upon consuming the signature of the time series~$\x=\langle 1,1,2,1,2,$ $1,1,2,1,2\rangle$,
the seed transducer~$\Transducer{\PeakPatternName}$ triggers the following sequence of transitions
$\langle\trans{s}{s}{=},\trans{s}{r}{<},\trans{r}{t}{>},\trans{t}{r'}{<},$ $\trans{r'}{t}{>},\trans{t}{t}{=},\trans{t}{r'}{<},\trans{r'}{t}{>},\trans{t}{r'}{<}\rangle$.
The index of the last triggered~$\Found$-transition is~$8$.
From \lemref{lem:regret-loss}, we obtain
$\CharArg{\Loss}{\CtrGamma}{\x}=10-1-8+(1+0+0+0+0+1+0+0+0)=3$.
\qedexample 
\end{example}

From \lemref{lem:regret-loss}, three registers are needed for the loss automaton.
Given a prefix of a signature consumed by the seed transducer, let~$\LastFoundIndex$ denote the last triggered~$\Found$-transition:
\begin{itemize}
\item
Register~$\ThirdAcc$ gives the sum of the regrets of the transitions triggered before~$\LastFoundIndex$.
Note that the regret of~$\LastFoundIndex$ is zero.
\item
Register~$\SecondAcc$ gives the sum of the regrets of the transitions triggered after~$\LastFoundIndex$.
\item
Register~$\FirstAcc$ gives the number of transitions triggered after~$\LastFoundIndex$.
\end{itemize}

The initial value of these three registers is zero.
The decoration table, given in the next section, follows from~\lemref{lem:regret-loss}.

\paragraph{Decoration Table of a Loss Automaton}
\label{sec:decoration-table-nb}

As stated before, a loss automaton for $\constraint{nb}\_\pattern$ has three registers~$\FirstAcc$, $\SecondAcc$ and~$\ThirdAcc$.
Given a prefix of some signature consumed by the seed transducer~$\Transducer{\pattern}$,
let~$\LastFoundIndex$ denote the last triggered~$\Found$-transition.
When~$\Transducer{\pattern}$ triggers the transition~$t$, we have one of the two following cases:
\begin{enumerate}

\item

[$t$ is not a~$\Found$-transition]
Then~$\LastFoundIndex$ is still the last triggered~$\Found$-transition.
There is one more transition triggered after~$\LastFoundIndex$, and the register~$\FirstAcc$ must be increased by~$1$.
Further, the value of~$\SecondAcc$ should be increased by the regret of~$t$.
Finally, register~$\ThirdAcc$ remains unchanged.

\item

[$t$ is a~$\Found$-transition]
Then~$t$ becomes the last triggered~$\Found$-transition.
Since there is no transition triggered after~$t$, registers~$\FirstAcc$ and~$\SecondAcc$ must both be reset to~$0$.
Register~$\ThirdAcc$ must be increased by the sum of the regrets of all the transitions triggered
after~$\LastFoundIndex$ and before~$t$, i.e.~the value of~$\SecondAcc$.
\end{enumerate}

By~\lemref{lem:regret-loss}, the loss of a time series is the sum between the sum of the regrets
of all the triggered transitions before the last~$\Found$-transition and the number of transitions
triggered after the last~$\Found$\nobreakdash-transition.
This is the sum of the last values of~$\FirstAcc$ and~$\ThirdAcc$.
Part~(A) of~Figure~\ref{tab:decoration-table-nb} summarises how registers are updated.

\begin{figure}[!h]
{\small
\begin{tikzpicture}[scale=0.88, every node/.style={scale=0.88}]
\begin{scope}[information text/.style={rounded corners,inner sep=1ex}]
\draw node[right,text width=8.6cm,information text,fill=white] {
\begin{tabular}{llll}
\toprule
{\bf initial values} & $\FirstAcc\gets 0$ & $\SecondAcc\gets 0$ & $\ThirdAcc\gets 0$ \\
{$\begin{array}{l}\textbf{acceptance}\\ \textbf{function}\end{array}$} & $\ThirdAcc+\FirstAcc$ \\ \midrule
\textbf{phase letters} & \hspace*{3pt}\emph{update of $C$} & \hspace*{2pt}\emph{update of $D$} &
                                                             \hspace*{2pt}\emph{update of $R$} \\ \midrule
$\Found$ & $\FirstAcc\gets 0$ & $\SecondAcc\gets 0$ & $\ThirdAcc\gets\ThirdAcc+\SecondAcc$ \\
$\NotFound$ & $\FirstAcc\gets\FirstAcc+1$ & $\SecondAcc\gets\SecondAcc+\Regret{t}$ \\
\bottomrule
\end{tabular}
};
\draw node at (5,-1.5) {\normalsize(A)};
\end{scope}
\begin{scope}[xshift=11cm,yshift=1.5cm,->,>=stealth',shorten >=1pt,auto,semithick,
              information text/.style={rounded corners=1pt,inner sep=1ex},font=\scriptsize,
              node distance=28mm,->,>=stealth',shorten >=1pt,auto,semithick]
      \node[initial,initial where=above,initial text=,initial distance=3mm,accepting,state] (t)             {{$s$}};
      \node[accepting,state]                                                                (s) [right of=t]{{$r$}};
      \node[accepting,state]                                                                (s1) at (0,-3)  {{$r'$}};
      \node[accepting,state]                                                                (t1) at (2.8,-3){{$t$}};
      \node[rectangle,draw] at (1.4,-1) (ctr) { return $R+C$};
      \draw[dotted,-] (t)  -- (ctr.west);
      \draw[dotted,-] (s)  -- (ctr.east);
      \draw[dotted,-] (s1) -- (ctr.west);
      \draw[dotted,-] (t1) -- (ctr.east);
      \path
        (s) edge [loop right]              node [right]{ $\begin{array}{c} <,= \\
                                                                \left \{\begin{array}{l}
                                                                         C \leftarrow\hspace*{1pt}C + 1 \\
                                                                         D \leftarrow D + 1
                                                                        \end{array}
                                                                \right\} \\ ~  \\ ~ \\ ~
                                                                \end{array}$} (s)
        (t) edge [loop left]               node [left] { $\begin{array}{c} >,= \\
                                                                \left \{\begin{array}{l}
                                                                         C \leftarrow\hspace*{1pt}C + 1 \\
                                                                         D \leftarrow D + 1
                                                                        \end{array}
                                                                \right\} \\ ~  \\ ~ \\ ~
                                                                \end{array}$} (t)
        (t) edge node [above]{ $\begin{array}{c} < \\
                                                                \left \{\begin{array}{l}
                                                                         C \leftarrow C + 1 \\
                                                                        \end{array}
                                                                \right\}
                                                                \end{array}$} (s)
        (s) edge node [sloped,above]{ $\begin{array}{c} > \\
                                                                \left \{\begin{array}{l}
                                                                         C \leftarrow 0     \\
                                                                         D \leftarrow 0     \\
                                                                         R \leftarrow R + D
                                                                        \end{array}
                                                                \right\}
                                                                \end{array}$} (t1)
        (t1)edge [loop right]              node [right]{ $\begin{array}{c} \\ ~  \\ ~ \\ >,= \\
                                                                \left \{\begin{array}{l}
                                                                         C \leftarrow\hspace*{1pt}C + 1 \\
                                                                         D \leftarrow D + 1
                                                                        \end{array}
                                                                \right\}
                                                                \end{array}$} (t1)
        (s1)edge [loop left]               node [left] { $\begin{array}{c}  \\ ~  \\ ~ \\ <,= \\
                                                                \left \{\begin{array}{l}
                                                                         C \leftarrow\hspace*{1pt}C + 1 \\
                                                                         D \leftarrow D + 1
                                                                        \end{array}
                                                                \right\}
                                                                \end{array}$} (s1)
        (t1)edge node [below]{ $\begin{array}{c} < \\
                                                                \left \{\begin{array}{l}
                                                                         C \leftarrow C + 1 \\
                                                                        \end{array}
                                                                \right\}
                                                                \end{array}$} (s1)
        (s1)edge [bend angle=20,bend left] node [above]{ $\begin{array}{c} > \\
                                                                \left \{\begin{array}{l}
                                                                         C \leftarrow 0     \\
                                                                         D \leftarrow 0     \\
                                                                         R \leftarrow R + D
                                                                        \end{array}
                                                                \right\}
                                                                \end{array}$} (t1);
\draw node at (-0.7,-1.05) {\normalsize(B)};
\end{scope}
\end{tikzpicture}
}
\caption{\label{tab:decoration-table-nb}
(A)~Decoration table for the loss automaton for $\constraint{nb}\_\pattern$ time-series constraints,
    where~$\Regret{t}$ denotes the regret of a transition~$t$ of the seed transducer for~$\pattern$;
(B)~Loss automaton for $\PeakConstraint$; the initial value of the registers~$C$, $D$, and~$R$ is zero;
    as the regret of the~$\NotFound$ transitions~$\trans{s}{r}{<}$ and~$\trans{t}{r'}{<}$ of the seed
    transducer for~$\pattern$ is zero, the register~$D$ remains unchanged while triggering these two
    transitions.}
\end{figure}

To obtain the loss automaton for a $\constraint{nb}\_\pattern$ time-series constraint, we replace every
output letter in the separated seed transducer for~$\pattern$ with the corresponding set of register
updates according to the decoration table shown in Part~(A) of Figure~\ref{tab:decoration-table-nb}.
The initial value of all three registers is zero, and the acceptance function is $\FirstAcc+\ThirdAcc$.
\begin{example}[loss automaton]
\label{ex:loss-automaton}
The loss automaton for $\constraint{nb}\_\PeakPatternName$, obtained from the seed transducer in Part~(B)
of~\figref{fig:peak-transducer-separated} and from the decoration table in Part~(A) of Figure~\ref{tab:decoration-table-nb},
is given in Part~(B) of Figure~\ref{tab:decoration-table-nb}.
\qedexample
\end{example}


\subsection{Summary}
We presented a systematic approach for generating $\Shift$-gap automata for time\nobreakdash-series constraints,
and demonstrated its applicability for the $\constraint{nb}\_\pattern$ family.
We used the obtained automata both (i)~for proving that 70\% of our synthesised linear invariants were facet defining,
and (ii)~for proving the correctness of all non-linear invariants of a database of invariants on conjunctions of time-series constraints.

Although, we did this work in the context of time series, the same method can be used
for generating~$\Shift$\nobreakdash-gap automata
for any constraint satisfying the four principal conditions.
As an example, consider the $\constraint{nb\_group}(X,\Result,P)$ constraint~\cite{Cosytec97,BeldiceanuCarlssonRamponTruchet05},
where~$X$ is a sequence of $n$ integer variables, $\Result$ is an integer variable, and~$P$ is a non-empty finite set of integer numbers.
This constraint restricts $\Result$ to be the number of maximal subsequences of~$X$ whose elements are in~$P$.
For example, $\constraint{nb\_group}(\Tuple{1,3,4,1,0,9,0},$ $3,\{0,1\})$ holds. 
Then a sharp upper bound on~$\Result$ is~$\Frac{\seqlength}{2}$, and it can be shown that all the four principal conditions
are satisfied for $\constraint{nb\_group}$.
Hence by \thref{th:existence-gap-automaton} for any natural~$\Shift$, the~$\Shift$-gap automaton for $\constraint{nb\_group}$
exists and can be constructed by the method given in the proof of \thref{th:existence-gap-automaton}.



\section{Evaluation}\label{sec:evaluation}

To test the generated invariants, we use real-world electricity demand data from an industrial partner.
The dataset contains time series of length~$96$ ($2$ days in half-hour resolution) for multiple years.
We use fixed size prefixes of the data to show scaleability of our methods. 

In a first experiment we consider prefixes of length~$25$ and test all binary combinations of the considered constraints
both with our baseline implementation of the individual constraints (version \emph{pure}) and with the added,
generated invariants applied to each suffix (version \emph{incremental}).
From the dataset, we extract as features the observed values for a pair of constraints for a time-series instance,
and then try to find an assignment that achieves these values.
Each problem is feasible, as it is based on an existing assignment.
Any improvement of the propagation is due to detecting failures in partial assignments more quickly by applying the invariants
to suffixes of the complete series.
Our default search strategy labels the signature variables first, followed by the decision variables,
always starting with the smallest values. As all constraints used here operate on the signature variables only,
we can always find an assignment of the decision variables once a feasible assignment of the signatures is found.  

Figure~\ref{fig:matrix} shows the results, with the \emph{pure} baseline above the main diagonal,
and the results with the added invariants (\emph{incremental}) below the main diagonal.
Each box represents the results for~$100$ time series.
The number in the box, if present, shows how many of the~$100$ experiments timed out (limit $2$~seconds)
with the default search strategy. The colour of the cell indicates the average number of backtracks
required for the solved instances, based on the legend below the matrix. All experiments were run
using SICStus Prolog~4.3.5 on a Windows~10 laptop with~64~GB of memory,
using a single core of the Intel~i7 processor running at~2.9~GHz base speed.

\begin{figure}[h]
\caption{\label{fig:matrix}Comparing baseline (top left) and added invariants (bottom right) models
on all binary combinations of considered constraints;
Length~$25$ variables; $100$~feasible samples,
Number of timeouts as numbers, average number of backtracks of solved problems as cell colour.}
\input{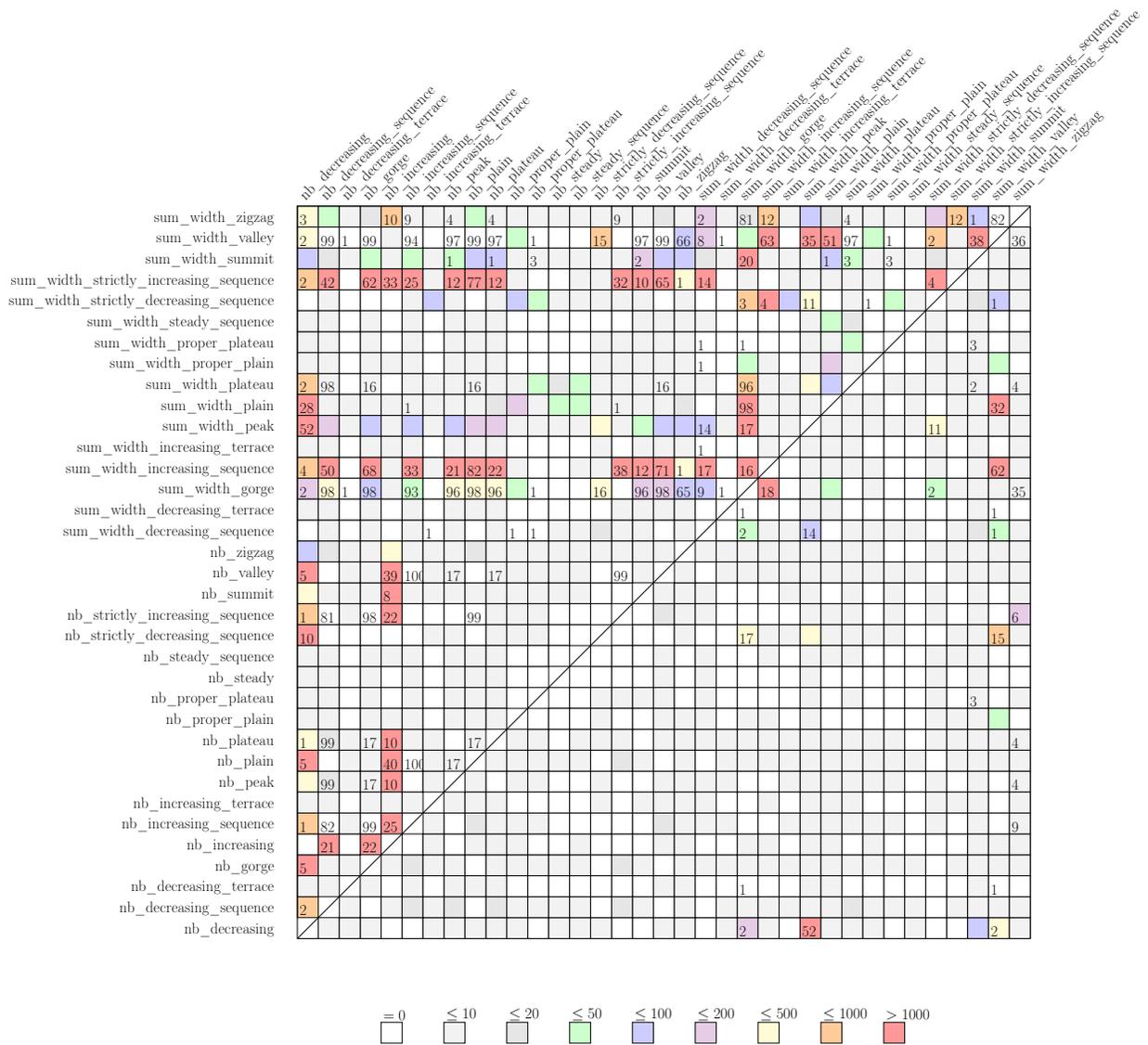}
\end{figure}

Adding the invariants decreases both the number of timeouts and the number of backtracks for most, but not all,
constraint combinations. While some constraint combinations are easily solved even without the invariants,
there are many cases where the baseline constraints are not able to find a solution quickly,
but the added invariants reduce the backtrack count close to zero.
It is interesting to note that all combinations of the \texttt{nb\_} constraints are solved
with less than~$20$~backtracks when the invariants are added, while the baseline constraint
do not find any solutions for several combinations of such constraints.

We repeat the experiments, but now for time-series length increasing from~$20$ to~$90$,
to investigate scaleability of the approach. Figure~\ref{fig:scalability} shows the baseline results on the left,
the results with added invariants on the right. We plot the percentage of instances solved as a function of execution time.
For the baseline, we see that with increasing problem size the percentage of problems solved steadily drops from~93.9\%
for size~$20$ to~75.9\% for size~90 with a timeout of~$2$~seconds. Adding the invariants improves the percentage to~99.3\%
for size~$20$, while still achieving~97.9\% for size $90$.

\begin{figure}[htbp]
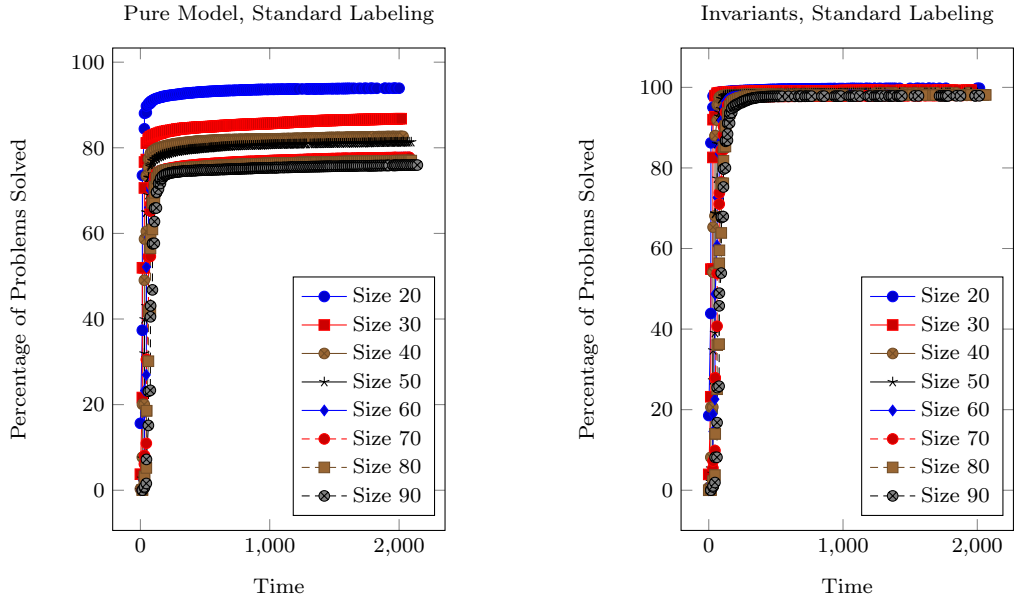

  \caption{\label{fig:scalability}Comparing baseline (left) and added invariants (right) models on time series of sizes~$20$-$90$;
  $100$~feasible sample; Showing cumulative percentage of problems solved as a function of execution time, timeout~$2$~seconds.}
  \begin{subfigure}[b]{.5\linewidth}
      \input{figures/tuples_graph_time_2_1_labelinga}
    \end{subfigure}%
  \begin{subfigure}[b]{.5\linewidth}
      \input{figures/tuples_graph_time_2_invariant_1_labelinga}
  \end{subfigure}%
  \end{figure}

\begin{figure}[htbp]
  \caption{\label{fig:overlap}Percentage of Problems Solved for~$3$ Overlapping Segments of Lengths~$22$, $24$, and~$25$;
  Execution time in top row, backtracks required in bottom row.}
  \begin{subfigure}[b]{.333\linewidth}
    \includegraphics[width=\linewidth]{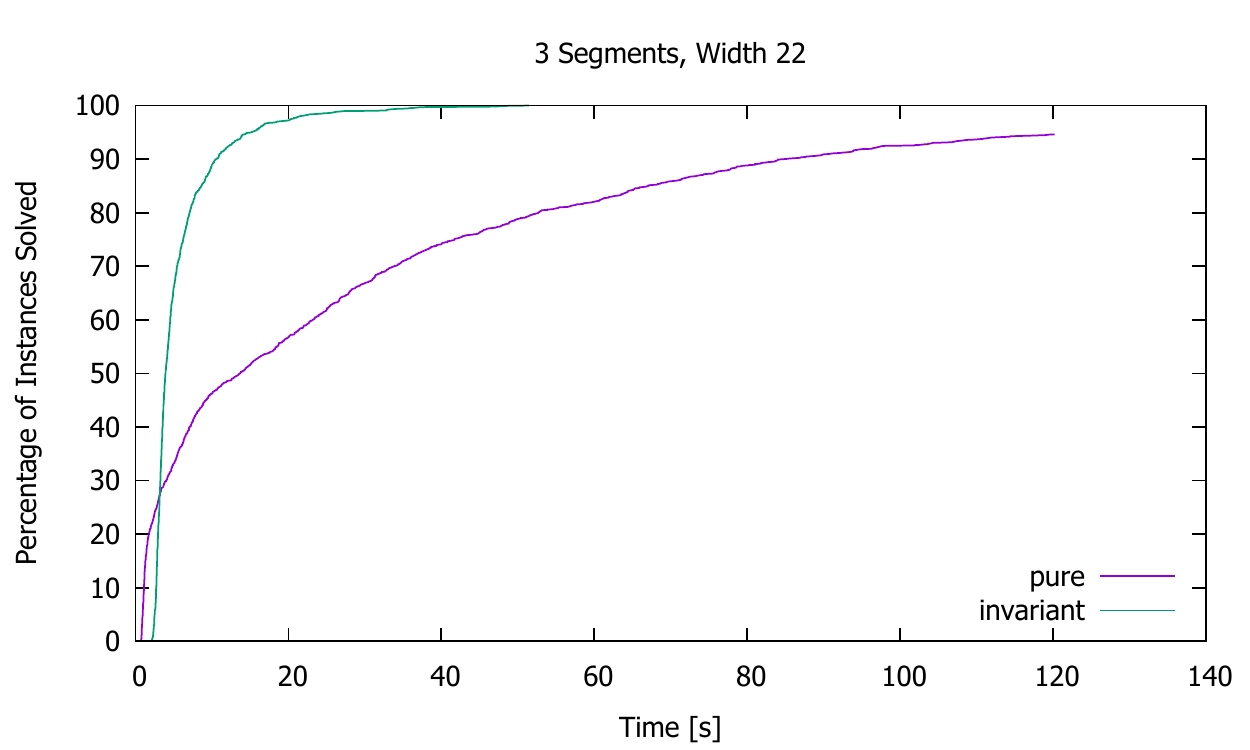}
    \subcaption{Time, Size 22}
  \end{subfigure}%
  \begin{subfigure}[b]{.333\linewidth}
  \includegraphics[width=\linewidth]{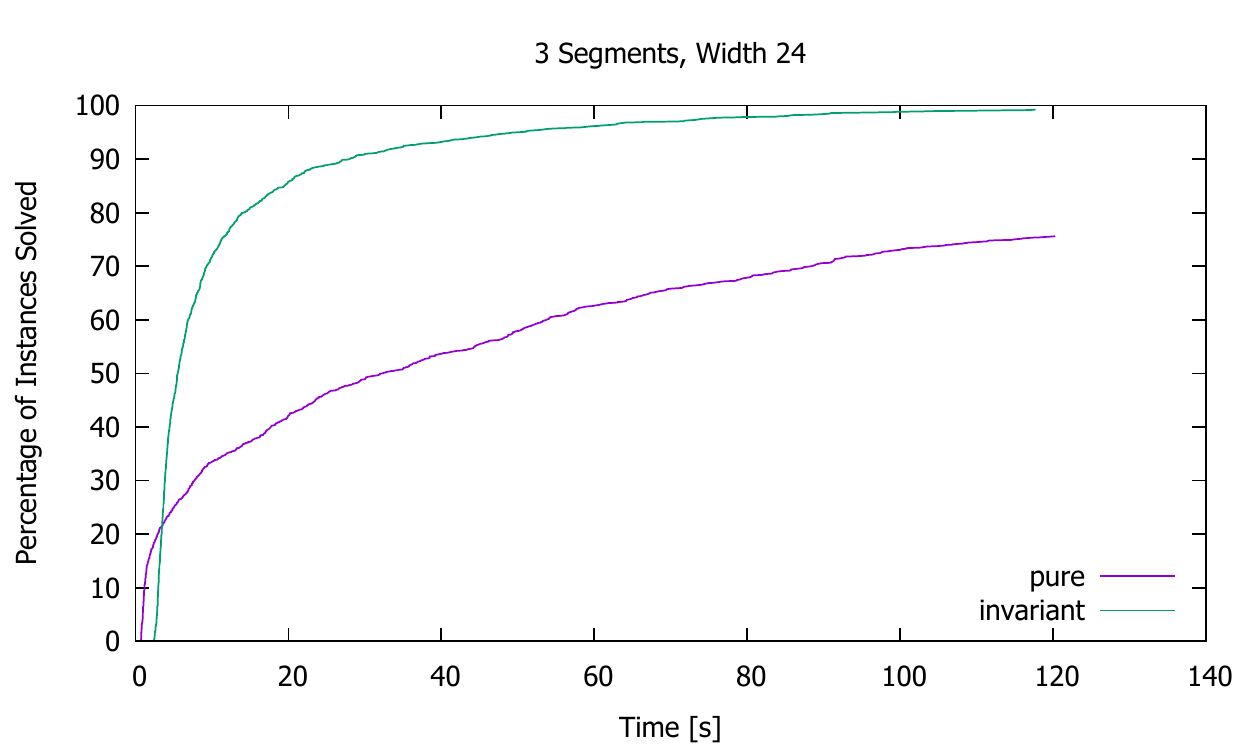}
    \subcaption{Time, Size 24}
  \end{subfigure}%
  \begin{subfigure}[b]{.333\linewidth}
  \includegraphics[width=\linewidth]{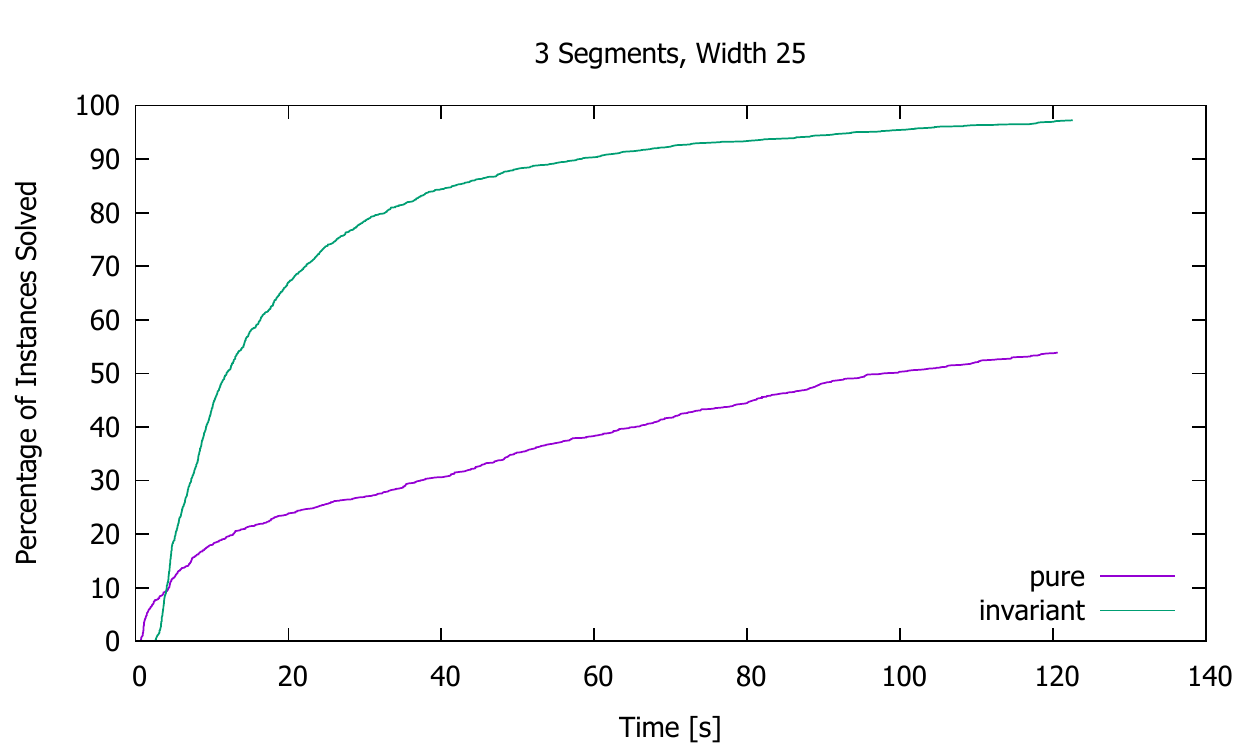}
    \subcaption{Time, Size 25}
  \end{subfigure}\\
  \begin{subfigure}[b]{.333\linewidth}
  \includegraphics[width=\linewidth]{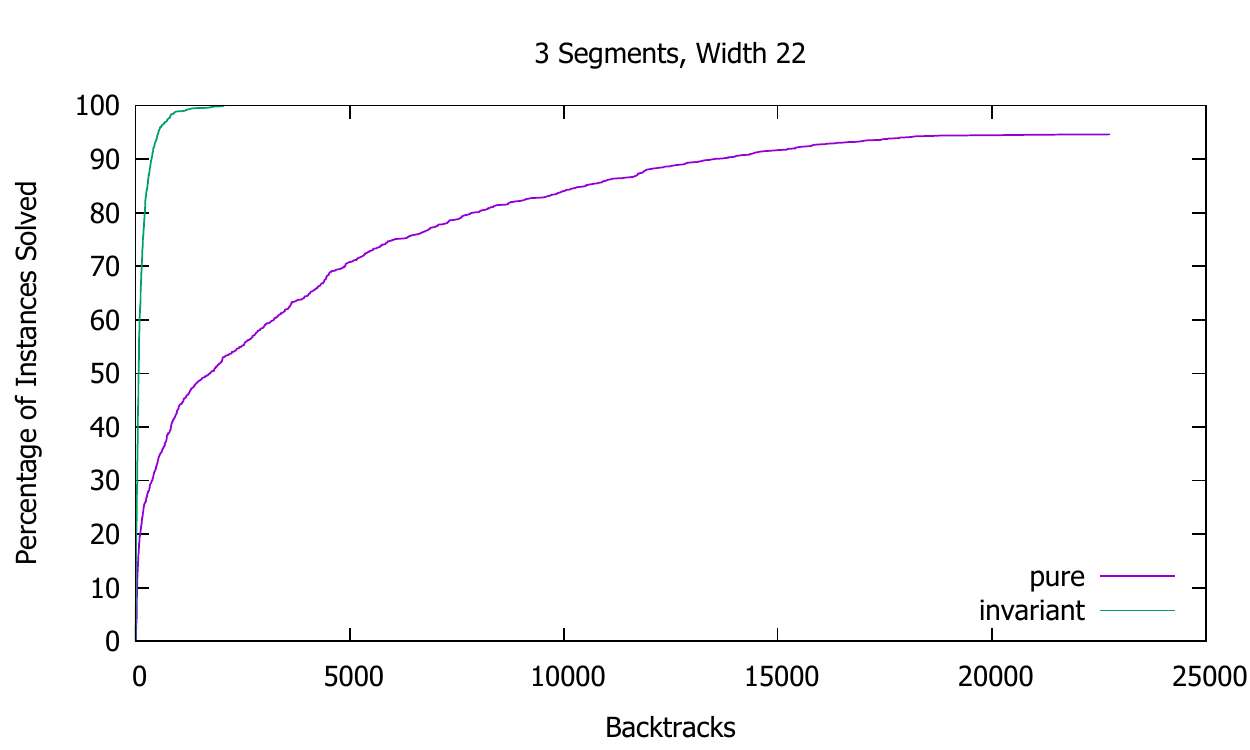}
    \subcaption{Backtracks, Size 22}
  \end{subfigure}%
  \begin{subfigure}[b]{.333\linewidth}
  \includegraphics[width=\linewidth]{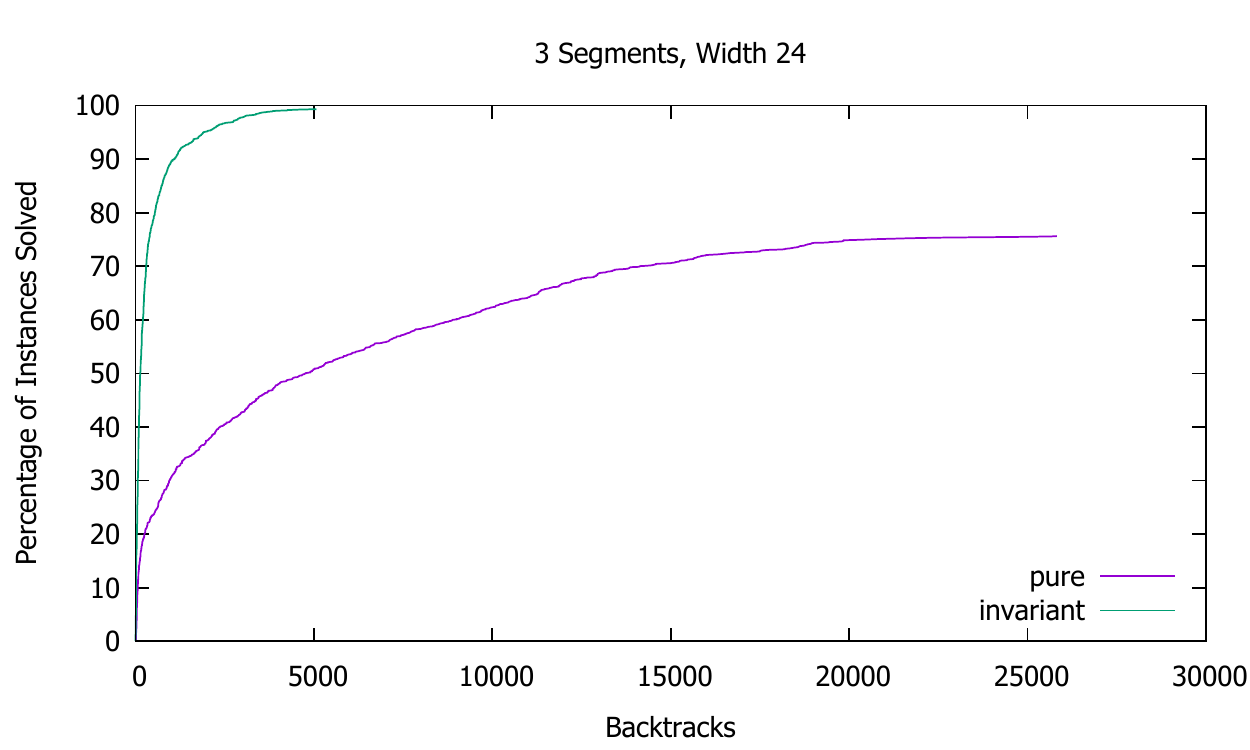}
    \subcaption{Backtracks, Size 24}
  \end{subfigure}%
  \begin{subfigure}[b]{.333\linewidth}
  \includegraphics[width=\linewidth]{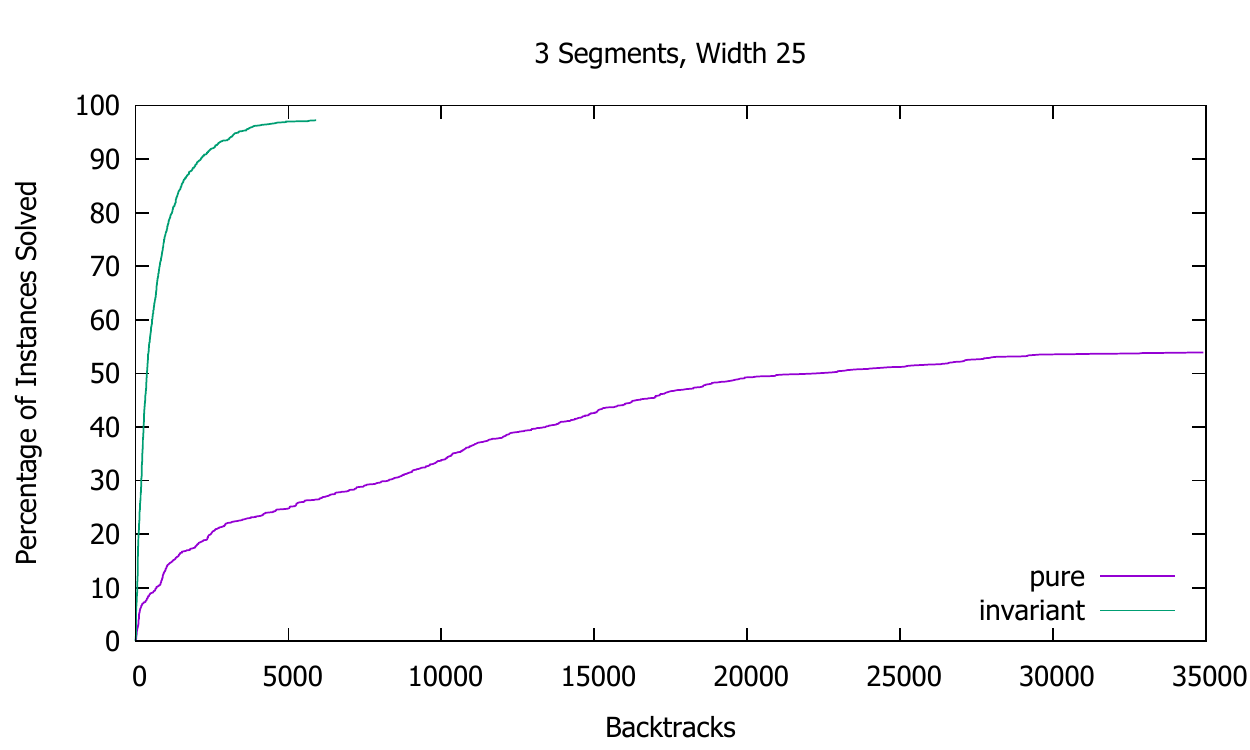}
    \subcaption{Backtracks, Size 25}
  \end{subfigure}%
\end{figure}
To test the method in a realistic setting, we consider the conjunction of all~$35$ considered time-series constraints on the dataset.
To capture the shape of the time series more accurately, we split the series into overlapping segments from~$00$-$12$, $06$-$18$,
and~$12$-$24$ hours, each segment containing~$24$ data points, overlapping in~$12$ data points with the previous segment.
We then set up the conjunction of the~$35$ time-series constraints for each segment, using the \emph{pure} and \emph{incremental}
variants described above. This leads to~$3\times 35\times 2=210$ automaton constraints with decision variables.
The invariants are created for every pair of constraints, and every suffix, leading to a large number of inequalities.
The search routine assigns all signature variables from left to right, and then assigns the decision variables, with a timeout of~$120$ seconds.

In order to understand the scaleability of the method, we also consider time series of~$44$ resp.\ $50$ data points
(three segments of length~$22$ and~$25$), extracted from the daily data stream covering a four-year period ($1448$~samples).
In Figure~\ref{fig:overlap} we show the time and backtrack profiles for finding a first solution.
The top row shows the percentage of instances solved within a given time budget, the bottom row shows the percentage of problems solved
within a backtrack budget. For easy problems, the \emph{pure} variant finds solutions more quickly, but the \emph{incremental} version
pays off for more complex problems, as it reduces the number of backtracks required sufficiently to account for the large overhead
of stating and pruning all invariants. The problems for segment length~$20$ (not shown) can be solved without timeout for both variants,
as the segment length increases, the number of timeouts increases much more rapidly for the \emph{pure} variant.

The results show that adding the generated invariants drastically improves the propagation, even for feasible problems.
The improvement is due to detecting infeasibility of a generated sub-problem for the remaining suffix of the unassigned variables more rapidly,
and therefore avoiding having to explore this infeasible subtree in the overall search.

\section{Conclusion}

Using the operational view of time-series constraints,
i.e.~the seed transducers for each regular expression and register automata,
we presented systematic methods for synthesising
1)~linear and 2)~non-linear invariants
linking the result values of several time-series constraints
and parameterised by a function of the time-series length,
and 3)~conditional automata representing a condition on the result value of a time-series constraint.
Since all these conditional automata have a number of states and an input alphabet
that do not depend on the length of an input sequence, these automata allow us to prove
both the fact that linear invariants are facet defining or not, and the validity of non-linear invariants,
for \emph{any long enough sequence length}.
All the~$2000$ synthesised parametrised invariants
were put in a publicly available database of invariants~\cite{Catalog18}
linked to the time-series catalogue that was used to automatically enhance
short-term electricity production models that were acquired from real production data.

\bibliographystyle{plain}
\bibliography{paper.bib}

\end{document}